\documentclass[twoside,11pt]{article}

\usepackage{amsmath}
\usepackage{amssymb}
\usepackage{amsthm}  
\usepackage[colorlinks=false,allbordercolors={1 1 1}]{hyperref}  
\usepackage[capitalise,nameinlink]{cleveref}   

\usepackage[nohyperref, preprint]{jmlr2e}

\usepackage{subcaption}  
\usepackage{caption}
\usepackage[utf8]{inputenc}
\usepackage[T1]{fontenc}
\usepackage[USenglish]{babel}

\usepackage{mathtools}
\usepackage{comment}
\usepackage[shortlabels]{enumitem}
\usepackage{csquotes}
\usepackage{tikz}
\usepackage{tikz-qtree}
\usetikzlibrary{trees,positioning,shapes}
\usetikzlibrary{arrows.meta}
\usepackage{pgfplots}
\usepackage{xcolor,colortbl}
\usepackage{algpseudocode}
\usepackage{algorithm}
\usepackage{thmtools}
\usepackage{moresize}
\usepackage{longtable}
\usepackage[normalem]{ulem}  

\usepackage[section]{placeins}

\usepackage{chngcntr}
\usepackage{stmaryrd}
\usepackage{dsfont}  

\usepackage{marginnote}

\usepackage[retainorgcmds]{IEEEtrantools}

\pgfplotsset{compat=1.11}

\newlength\Origarrayrulewidth

\newenvironment{appendixenv}{%
\clearpage
\appendix
\numberwithin{theorem}{section}
\numberwithin{example}{section}
\numberwithin{assumption}{section}
\counterwithin{figure}{section}
\counterwithin{table}{section}
\counterwithin{algorithm}{section}}{}

\newcommand{\ifnotinthesis}[1]{#1}
\newcommand{\ifinthesis}[1]{}

\providecommand{\BlackBox}{\rule{1.5ex}{1.5ex}}  

\renewenvironment{proof}[1][\proofname]{\par
  \pushQED{\qed}%
  \normalfont\noindent{\bf #1\ }
}{%
  \popQED
}
\providecommand{\proofname}{Proof}

\newcommand{\bbE}{\mathbb{E}}

\newcommand{\bbH}{\mathbb{H}}

\newcommand{\bbN}{\mathbb{N}}

\newcommand{\bbR}{\mathbb{R}}

\newcommand{\calD}{\mathcal{D}}

\newcommand{\calL}{\mathcal{L}}

\newcommand{\calN}{\mathcal{N}}
\newcommand{\calO}{\mathcal{O}}

\newcommand{\calU}{\mathcal{U}}

\newcommand{\calX}{\mathcal{X}}
\newcommand{\calY}{\mathcal{Y}}

\providecommand{\bfa}{\boldsymbol{a}}
\providecommand{\bfb}{\boldsymbol{b}}
\providecommand{\bfc}{\boldsymbol{c}}
\providecommand{\bfd}{\boldsymbol{d}}

\providecommand{\bfs}{\boldsymbol{s}}

\providecommand{\bfu}{\boldsymbol{u}}
\providecommand{\bfv}{\boldsymbol{v}}
\providecommand{\bfw}{\boldsymbol{w}}
\providecommand{\bfx}{\boldsymbol{x}}

\providecommand{\bfz}{\boldsymbol{z}}

\providecommand{\bfA}{\boldsymbol{A}}
\providecommand{\bfB}{\boldsymbol{B}}
\providecommand{\bfC}{\boldsymbol{C}}
\providecommand{\bfD}{\boldsymbol{D}}

\providecommand{\bfH}{\boldsymbol{H}}
\providecommand{\bfI}{\boldsymbol{I}}

\providecommand{\bfK}{\boldsymbol{K}}
\providecommand{\bfL}{\boldsymbol{L}}
\providecommand{\bfM}{\mathbf{M}}

\providecommand{\bfU}{\boldsymbol{U}}

\providecommand{\bfW}{\boldsymbol{W}}

\providecommand{\bftheta}{\boldsymbol{\theta}}

\providecommand{\bfmu}{\boldsymbol{\mu}}

\providecommand{\bfSigma}{\boldsymbol{\Sigma}}

\providecommand{\bfPhi}{\boldsymbol{\Phi}}

\providecommand{\bfzero}{\boldsymbol{0}}

\newcommand{\diff}{\,\mathrm{d}}
\newcommand{\quot}[1]{\enquote{#1}}

\newcommand{\equalDef}{\coloneqq}
\newcommand{\defEqual}{\eqqcolon}

\DeclareMathOperator*{\argmax}{argmax}

\DeclareMathOperator{\tr}{tr}

\DeclareMathOperator{\Cov}{Cov}

\allowdisplaybreaks

\newcommand{\Random}{\textsc{Random}}
\newcommand{\MaxDiag}{\textsc{MaxDiag}}
\newcommand{\FrankWolfe}{\textsc{FrankWolfe}}
\newcommand{\MaxDist}{\textsc{MaxDist}}
\newcommand{\MaxDet}{\textsc{MaxDet}}
\newcommand{\Bait}{\textsc{Bait}}
\newcommand{\KMeansPP}{\textsc{KMeansPP}}
\newcommand{\LCMD}{\textsc{LCMD}}

\newcommand{\acs}{\operatorname{acs}}
\newcommand{\acsgrad}{\operatorname{acs-grad}}
\newcommand{\rp}[1]{\operatorname{sketch}(#1)}
\newcommand{\acsrf}[1]{\operatorname{acs-rf}(#1)}
\newcommand{\acsrfhyper}[1]{\operatorname{acs-rf-hyper}(#1)}
\newcommand{\ens}[1]{\operatorname{ens}(#1)}
\newcommand{\scale}[1]{\operatorname{scale}(#1)}
\newcommand{\post}[1]{\operatorname{post}(#1)}

\newcolumntype{C}[1]{>{\centering\arraybackslash}p{#1}}

\setlist[enumerate]{nosep}
\setlist[itemize]{nosep}

\renewcommand{\eqref}[1]{Eq.~(\ref{#1})}

\newcommand{\ntrain}{N_{\mathrm{train}}}
\newcommand{\Ntrain}{N_{\mathrm{train}}}
\newcommand{\nvalid}{N_{\mathrm{valid}}}

\DeclareMathOperator*{\argmin}{argmin}

\newcommand{\Dtrain}{\calD_{\mathrm{train}}}

\newcommand{\Xtrain}{\calX_{\mathrm{train}}}
\newcommand{\Ytrain}{\calY_{\mathrm{train}}}
\newcommand{\Xpool}{\calX_{\mathrm{pool}}}
\newcommand{\Ypool}{\calY_{\mathrm{pool}}}

\newcommand{\tXbatch}{\tilde{\calX}_{\mathrm{batch}}}
\newcommand{\Xbatch}{\calX_{\mathrm{batch}}}
\newcommand{\Ybatch}{\calY_{\mathrm{batch}}}
\newcommand{\Dbatch}{\calD_{\mathrm{batch}}}
\newcommand{\Xsel}{\calX_{\mathrm{sel}}}
\newcommand{\Xrem}{\calX_{\mathrm{rem}}}
\newcommand{\Xmode}{\calX_{\mathrm{mode}}}
\newcommand{\Xcand}{\calX_{\mathrm{cand}}}
\newcommand{\NextSample}{\textsc{NextSample}}

\newcommand{\dfeatin}{d_{\mathrm{pre}}}
\newcommand{\dfeatout}{d_{\mathrm{post}}}
\newcommand{\dfeat}{d_{\mathrm{feat}}}
\newcommand{\Nbatch}{N_{\mathrm{batch}}}
\newcommand{\Nextra}{N_{\mathrm{extra}}}
\newcommand{\Lbatch}{L_{\mathrm{batch}}}
\newcommand{\Npool}{N_{\mathrm{pool}}}
\newcommand{\Xtp}{\calX_{\mathrm{tp}}}

\newcommand{\Nsel}{N_{\mathrm{sel}}}
\newcommand{\Ncand}{N_{\mathrm{cand}}}

\newcommand{\Nens}{N_{\mathrm{ens}}}
\newcommand{\facs}{f_{\mathrm{acs}}}

\newcommand{\tbfx}{\tilde{\bfx}}

\newcommand{\BigO}{\calO}

\newcommand{\assign}{\leftarrow}

\newcommand{\corrrem}[1]{}

\makeatletter
\newcommand\ackname{Acknowledgements}
\if@titlepage
   \newenvironment{acknowledgements}{%
       \titlepage
       \null\vfil
       \@beginparpenalty\@lowpenalty
       \begin{center}%
         \bfseries \ackname
         \@endparpenalty\@M
       \end{center}}%
      {\par\vfil\null\endtitlepage}
\else
   
\fi
\makeatother

\hypersetup{ hidelinks }

\usepackage{lastpage}

\usepackage{lastpage}
\jmlrheading{24}{2023}{1-\pageref{LastPage}}{8/22; Revised
1/23}{5/23}{22-0937}{David Holzmüller, Viktor Zaverkin, Johannes Kästner, and Ingo Steinwart}

\ShortHeadings{Deep Batch Active Learning for Regression}{Holzmüller, Zaverkin, Kästner, and Steinwart}
\firstpageno{1}

\begin{document}

\title{A Framework and Benchmark for Deep Batch Active Learning for Regression}

\author{\name David Holzmüller \email david.holzmueller@mathematik.uni-stuttgart.de \\
       \addr University of Stuttgart\\
       Faculty of Mathematics and Physics\\
       Institute for Stochastics and Applications
       \AND
       \name Viktor Zaverkin\thanks{Present address: NEC Laboratories Europe GmbH, Kurfürsten-Anlage 36, 69115 Heidelberg, Germany} \email zaverkin@theochem.uni-stuttgart.de \\
       \addr University of Stuttgart\\
       Faculty of Chemistry \\
  	   Institute for Theoretical Chemistry
       \AND
       \name Johannes Kästner \email kaestner@theochem.uni-stuttgart.de \\
       \addr University of Stuttgart\\
       Faculty of Chemistry \\
  	   Institute for Theoretical Chemistry
       \AND
       \name Ingo Steinwart \email ingo.steinwart@mathematik.uni-stuttgart.de \\
       \addr University of Stuttgart\\
       Faculty of Mathematics and Physics\\
       Institute for Stochastics and Applications}

\editor{Aarti Singh}

\maketitle

\begin{abstract}%
The acquisition of labels for supervised learning can be expensive. To improve the sample efficiency of neural network regression, we study active learning methods that adaptively select batches of unlabeled data for labeling. We present a framework for constructing such methods out of (network-dependent) base kernels, kernel transformations, and selection methods. Our framework encompasses many existing Bayesian methods based on Gaussian process approximations of neural networks as well as non-Bayesian methods. 
Additionally, we propose to replace the commonly used last-layer features with sketched finite-width neural tangent kernels and to combine them with a novel clustering method. 
To evaluate different methods, we introduce an open-source benchmark consisting of 15 large tabular regression data sets. 
Our proposed method outperforms the state-of-the-art on our benchmark, scales to large data sets, and works out-of-the-box without adjusting the network architecture or training code. We provide open-source code that includes efficient implementations of all kernels, kernel transformations, and selection methods, and can be used for reproducing our results.
\end{abstract}

\ifnotinthesis{%
}

\section{Introduction}

While supervised machine learning (ML) has been successfully applied to many different problems, these successes often rely on the availability of large data sets for the problem at hand. In cases where labeling data is expensive, it is important to reduce the required number of labels. Such a reduction could be achieved through various means: First, finding more sample-efficient supervised ML methods; second, applying data augmentation; third, leveraging information in unlabeled data via semi-supervised learning; fourth, leveraging information from related problems through transfer learning, meta-learning, or multi-task learning; and finally, appropriately selecting which data to label. Active learning (AL) takes the latter approach by using a trained model to choose the next data point to label \citep{settles_active_2009}. The need to retrain after every new label prohibits parallelized labeling methods and can be far too expensive, especially for neural networks (NNs), which are often slow to train. This problem can be resolved by batch mode active learning (BMAL) methods, which select multiple data points for labeling at once. When the supervised ML method is a deep NN, this is known as batch mode deep active learning (BMDAL) \citep{ren_survey_2021}. 
Pool-based BMDAL refers to the setting where data points for labeling need to be chosen from a given finite set of points.

Supervised and unsupervised ML algorithms choose a model for given data. Multiple models can be compared on the same data using model selection techniques such as cross-validation. Such a comparison increases the training cost, but not the (potentially much larger) cost of labeling data.
In contrast to supervised learning, AL is about choosing the data itself, with the goal to reduce labeling cost. 
However, different AL algorithms may choose different samples, and hence a comparison of $N$ AL algorithms might increase labeling cost by a factor of up to $N$. 
Consequently, such a comparison is not sensible for applications where labeling is expensive.
Instead, it is even more important to properly benchmark AL methods on tasks where labels are cheap to generate or a large number of labels is already available.

In the classification setting, NNs typically output uncertainties in the form of a vector of probabilities obtained through a softmax layer, while regression NNs typically output a scalar target without uncertainties. Therefore, many BMDAL algorithms only apply to one of the two settings.
For classification, many BMDAL approaches have been proposed \citep{ren_survey_2021}, and there exist at least some standard benchmark data sets like CIFAR-10 \citep{krizhevsky_learning_2009} on which methods are usually evaluated. On the other hand, the regression setting has been studied less frequently, and no common benchmark has been established to the best of our knowledge, except for a specialized benchmark in drug discovery \citep{mehrjou_genedisco_2021}. We expect that the regression setting will gain popularity, not least due to the increasing interest in NNs for surrogate modeling \citep{behler_perspective_2016, kutz_deep_2017, raissi_physics-informed_2019, mehrjou_genedisco_2021, lavin_simulation_2021}.

\subsection{Contributions} In this paper, we investigate pool-based BMDAL methods for regression.
Our experiments use fully connected NNs on tabular data sets, but the considered methods can be generalized to different types of data and NN architectures. We limit our study to methods that do not require to modify the network architecture and training, as these are particularly easy to use and a fair comparison to other methods is difficult. We also focus on methods that scale to large amounts of (unlabeled) data and large acquisition batch sizes. Our contributions can be summarized as follows:
\begin{enumerate}[(1)]
\item We propose a framework for decomposing typical BM(D)AL algorithms into the choice of a kernel and a selection method. Here, the kernel can be constructed from a base kernel through a series of kernel transformations. %
The use of kernels as basic building blocks allows for an efficient yet flexible and composable implementation of our framework, which we include in our open-source code.
We also discuss how (regression variants of) many popular BM(D)AL algorithms can be represented in this framework and how they can efficiently be implemented. This gives us a variety of options for base kernels, kernel transformations, and selection methods to combine. Our framework encompasses both Bayesian methods based on Gaussian Processes and Laplace approximations as well as geometric methods.
\item We discuss some alternative options to the ones arising from popular BM(D)AL algorithms: We introduce a novel selection method called \LCMD{}; and we propose to combine the finite-width neural tangent kernel \citep[NTK,][]{jacot_neural_2018} as a base kernel with sketching for efficient computation.
\item We introduce an open-source benchmark for BMDAL involving 15 large tabular regression data sets. Using this benchmark, we compare different selection methods and evaluate the influence of the kernel, the acquisition batch size, and the target metric.
\end{enumerate}

Our newly proposed selection method, \LCMD{}, improves the state-of-the-art in our benchmark in terms of RMSE and MAE, while still exhibiting good performance for the maximum error. The NTK base kernel improves the benchmark accuracy for all selection methods, and the proposed sketching method can preserve this accuracy while leading to significant time gains. \Cref{fig:existing_algs} shows a comparison of our novel BMDAL algorithm against popular BMDAL algorithms from the literature, which are all implemented in our framework. The code for our framework and benchmark is based on PyTorch \citep{paszke_pytorch_2019} and is publicly available at
\begin{IEEEeqnarray*}{+rCl+x*}
\text{\url{https://github.com/dholzmueller/bmdal_reg}}
\end{IEEEeqnarray*}
and will be archived together with the generated data at \url{https://doi.org/10.18419/darus-3394}.

\begin{figure}[tb]
\centering
\includegraphics[scale=1.0]{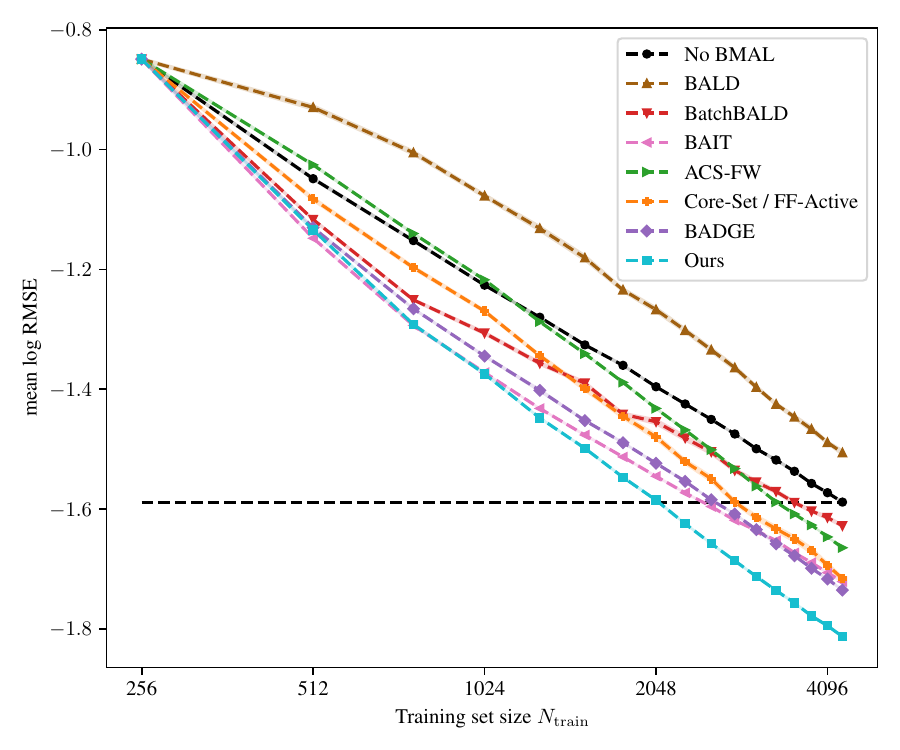}
\caption{%
This figure shows how fast the averaged errors on our benchmark data sets decrease during BMAL for
random selection (no BMAL), BALD \citep{houlsby_bayesian_2011}, BatchBALD \citep{kirsch_batchbald_2019}, BAIT \citep{ash_gone_2021}, ACS-FW \citep{pinsler_bayesian_2019}, Core-Set \citep{sener_active_2018}, FF-Active \citep{geifman_deep_2017}, BADGE \citep{ash_deep_2019}, and our method. 
In \Cref{table:existing_algs_results} and \Cref{sec:bmdal:experiments}, we specify how the compared methods are built from components explained in \Cref{sec:kernels} and \Cref{sec:sel_methods}, and discuss further details such as modifications to apply them to regression.
For the plot, we start with 256 random training samples and select 256 samples in each of 16 BMAL steps. The lines show the average of the logarithmic RMSE over all 15 benchmark data sets and 20 random splits between the BMAL steps.
The shaded area, which is barely visible, corresponds to one estimated standard deviation of the mean estimator, cf.\ \Cref{sec:appendix:results}.
} \label{fig:existing_algs}
\end{figure}

The rest of this paper is structured as follows: In \Cref{sec:problem_setting}, we introduce the basic problem setting of BMDAL for tabular regression with fully-connected NNs and introduce our framework for the construction of BMDAL algorithms. We discuss related work in \Cref{sec:related_work}. We then introduce options to build kernels from base kernels and kernel transformations in \Cref{sec:kernels}. \Cref{sec:sel_methods} discusses various iterative kernel-based selection methods. Our experiments in \Cref{sec:bmdal:experiments} provide insights into the performance of different combinations of kernels and selection methods. Finally, we discuss limitations and open questions in \Cref{sec:conclusion}. More details on the presented methods and experimental results are provided in the Appendix, whose structure is outlined in \Cref{sec:appendix:overview}.

\section{Problem Setting} \label{sec:problem_setting} %

In this section, we outline the problem of BMDAL for regression with fully-connected NNs. We first introduce the regression objective and fully-connected NNs. Subsequently, we introduce the basic setup of pool-based BMDAL as well as our proposed framework. %

\subsection{Regression with Fully-Connected Neural Networks} \label{sec:reg_fcnn}

We consider multivariate regression, where the goal is to learn a function $f: \bbR^d \to \bbR$ from data $(\bfx, y) \in \Dtrain \subseteq \bbR^d \times \bbR$. In the case of NNs, we consider a parameterized function family $(f_{\bftheta})_{\bftheta \in \bbR^m}$ and try to minimize the mean squared loss on training data $\Dtrain$ with $\Ntrain$ samples:
\begin{IEEEeqnarray*}{+rCl+x*}
\calL(\bftheta) & = & \frac{1}{\Ntrain} \sum_{(\bfx, y) \in \Dtrain} (y - f_{\bftheta}(\bfx))^2~.
\end{IEEEeqnarray*}
We refer to the inputs and labels in $\Dtrain$ as $\Xtrain$ and $\Ytrain$, respectively. Corresponding data sets are often referred to as \emph{tabular data} or \emph{structured data}. This is in contrast to data with a known spatiotemporal relationship between the input features, such as image or time-series data, where specialized NN architectures such as CNNs are more successful.

For our derivations and experiments, we consider an $L$-layer fully-connected NN $f_{\bftheta}: \bbR^d \to \bbR$ with parameter vector $\bftheta = (\bfW^{(1)}, \bfb^{(1)}, \hdots, \bfW^{(L)}, \bfb^{(L)})$ and input size $d_0 = d$, hidden layer sizes $d_1, \hdots, d_{L-1}$, and output size $d_L = 1$. The value $\bfz_i^{(L)} = f_{\bftheta}(\bfx_i^{(0)})$ of the NN on the $i$-th input $\bfx_i^{(0)} \in \bbR^{d_0}$ is defined recursively by
\begin{equation}
        \bfx_i^{(l+1)} = \varphi(\bfz_i^{(l+1)}) \in \bbR^{d_{l+1}}, \qquad \bfz_i^{(l+1)} = \frac{\sigma_w}{\sqrt{d_l}} \bfW^{(l+1)} \bfx_i^{(l)} + \sigma_b \bfb^{(l+1)} \in \bbR^{d_{l+1}}~. \label{eq:fcnn}
\end{equation}
Here, the activation function $\varphi: \bbR \to \bbR$ is applied element-wise and $\sigma_w, \sigma_b > 0$ are constant factors. In our experiments, the weight matrices are initialized with independent standard normal entries and the biases are initialized to zero. The factors $\sigma_w/\sqrt{d_l}$ and $\sigma_b$ stem from the neural tangent parametrization (NTP) \citep{jacot_neural_2018, lee_wide_2019}, which is theoretically motivated to define infinite-width limits of NNs and is also used in our applications. However, our derivations apply analogously to NNs without these factors. When considering different NN types such as CNNs, it is possible to apply our derivations only to the fully-connected part of the NN or to extend them to other layers as well.

\subsection{Batch Mode Active Learning} \label{sec:batch_active_learning}

In a single BMAL step, a BMAL algorithm selects a batch $\Xbatch \subseteq \bbR^d$ with a given size $\Nbatch \in \bbN$. Subsequently, this batch is labeled and added to the training set. Here, we consider \emph{pool-based} BMAL, where $\Xbatch$ is to be selected from a given finite \emph{pool set} $\Xpool$ of candidates. Other AL paradigms include membership query AL, where data points for labeling can be chosen freely, or stream-based AL, where data points arrive sequentially and must be immediately labeled or discarded. The pool set can potentially contain information about which regions of the input space are more important than others, especially if it is drawn from the same distribution as the test set. Moreover, pool-based BMAL allows for efficient benchmarking of BMAL methods on labeled data sets by reserving a large portion of the data set for the pool set, rendering the labeling part trivial. 

When comparing and evaluating BMDAL methods, we are mainly interested in the following desirable properties:
\begin{enumerate}[(P1)]
\item The method should improve the sample efficiency of the underlying NN, even for large acquisition batch sizes $\Nbatch$ and large pool set sizes $\Npool$, with respect to the downstream application, which may or may not involve the same input distribution as training and pool data. 
\item The method should scale to large pool sets, training sets, and batch sizes, in terms of both computation time and memory consumption.
\item The method should be applicable to a wide variety of NN architectures and training methods, such that it can be applied to different use cases.
\item The method should not require modifying the NN architecture and training method, for example by requiring to introduce Dropout, such that practitioners do not have to worry whether employing the method diminishes the accuracy of their trained NN. %
\item The method should not require training multiple NNs for a single batch selection since this would deteriorate its runtime efficiency.\footnote{Technically, requiring multiple trained NNs would not be detrimental if it facilitated reaching the same accuracy with correspondingly larger $\Nbatch$.}
\item The method should not require tuning hyperparameters on the downstream application since this would require labeling samples selected with suboptimal hyperparameters.
\end{enumerate}

Property (P1) is central to motivate the use of BMDAL over random sampling of the data and is evaluated for our framework in detail in \Cref{sec:bmdal:experiments} and \Cref{sec:appendix:experiments}. We only evaluate methods with property (P2) since our benchmark involves large data sets. All methods considered here satisfy (P3) to a large extent. Indeed, although efficient computations are only studied for fully-connected layers here, the considered methods can be simply applied to the fully-connected part of a larger NN. All considered methods satisfy (P4), which also facilitates fair comparison in a benchmark. 
All our methods satisfy (P5), although our methods can incorporate ensembles of NNs. Although some of the considered methods have hyperparameters, we fix them to reasonable values independent of the data set in our experiments, such that (P6) is satisfied. 

\Cref{alg:bmdal_loop} shows how BMDAL algorithms satisfying (P4) and (P5) can be used in a loop with training and labeling.

\begin{algorithm}[tb]
\caption{Basic pool-based BMDAL loop with initial labeled training set $\Dtrain$, unlabeled pool set $\Xpool$, BMDAL algorithm \textsc{NextBatch} (see \Cref{alg:kernel_next_batch}) and a list $\Lbatch$ of batch sizes.} \label{alg:bmdal_loop}
\begin{algorithmic}
\For{AL batch size $\Nbatch$ in $\Lbatch$}
	\State Train model $f_{\bftheta}$ on $\Dtrain$
	\State Select batch $\Xbatch \assign \Call{NextBatch}{f_{\bftheta}, \Dtrain, \Xpool, \Nbatch}$ with $|\Xbatch| = \Nbatch$ and $\Xbatch \subseteq \Xpool$
	\State Move $\Xbatch$ from $\Xpool$ to $\Dtrain$ and acquire labels $\Ybatch$ for $\Xbatch$
\EndFor
\State Train final model $f_{\bftheta}$ on $\Dtrain$
\end{algorithmic}
\end{algorithm}

\cite{wu_pool-based_2018} formulates three criteria by which BMAL algorithms may select batch samples in order to improve the sample efficiency of a learning method:
\begin{enumerate}[(A)]
\item[(INF)] The algorithm should favor inputs that are \emph{informative} to the model. These could, for example, be those inputs where the model is most uncertain about the label.
\item[(DIV)] The algorithm should ensure that the batch contains \emph{diverse} samples, i.e., samples in the batch should be sufficiently different from each other.
\item[(REP)] The algorithm should ensure \emph{representativity} of the resulting training set, i.e., it should focus more strongly on regions where the pool data distribution has high density.
\end{enumerate}
Note that (REP) might not be desirable if one expects a significant distribution shift between pool and test data. 
A challenge in trying to adapt non-batch AL methods to the batch setting is that some non-batch AL methods expect to immediately receive a label for every selected sample.
It is usually possible to circumvent this by selecting the $\Nbatch$ samples with the largest acquisition function scores at once, but this does not enforce (DIV) or (REP). 

\begin{algorithm}[tb]
\caption{Kernel-based batch construction framework} \label{alg:kernel_next_batch}
\begin{algorithmic}
\Function{KernelNextBatch}{$f_{\bftheta}, \Dtrain, \Xpool, \Nbatch$}
	\State $k \assign \Call{BaseKernel}{f_{\bftheta}}$
	\State $k \assign \Call{TransformKernel}{k, \Dtrain}$
	\State \Return $\Call{Select}{k, \Xtrain, \Xpool, \Nbatch}$
\EndFunction
\end{algorithmic}
\end{algorithm}

We propose a framework for assembling BMDAL algorithms that is shown in \Cref{alg:kernel_next_batch} and consists of three components: First, a base kernel $k$ needs to be chosen that should serve as a proxy for the trained network $f_{\bftheta}$. Second, the kernel can be transformed using various transformations. These transformations can, for example, make the kernel represent posteriors or improve its evaluation efficiency. Third, a selection method is invoked that uses the transformed kernel as a measure of similarity between inputs. When using Gaussian Process regression with a given kernel $k$ as a supervised learning method instead of an NN, the base kernel could simply be chosen as $k$. 
Note that \textsc{Select} does not observe the training labels directly, however, in the NN setting, these can be implicitly incorporated through kernels that depend on the trained NN.

\begin{example} \label{ex:framework}
In \Cref{alg:kernel_next_batch}, the base kernel $k$ could be of the form $k(\bfx, \tbfx) = \langle \phi(\bfx), \phi(\tbfx)\rangle$, where $\phi$ represents the trained NN without the last layer. When interpreting $k$ as the kernel of a Gaussian process, $\textsc{TransformKernel}$ could then compute a transformed kernel $\tilde k$ that represents the posterior predictive uncertainty after observing the training data. Finally, \textsc{Select} could then choose the $\Nbatch$ points $\bfx \in \Xpool$ with the largest uncertainty $\tilde k(\bfx, \bfx)$.
\end{example}

From a Bayesian perspective, our choice of kernel and kernel transformations can correspond to inference in a Bayesian approximation, as we discuss in \Cref{sec:appendix:posterior}, while the selection method can correspond to the optimization of an acquisition function. However, in our framework, the same \quot{Bayesian} kernels can be used together with non-Bayesian selection methods and vice versa.

\section{Related Work} \label{sec:related_work}

The field of active learning, also known as query learning or sequential (optimal) experimental design \citep{fedorov_theory_1972, chaloner_bayesian_1995}, has a long history dating back at least to the beginning of the 20th century \citep{smith_standard_1918}. For an overview of the AL and BMDAL literature, we refer to \cite{settles_active_2009, kumar_active_2020, ren_survey_2021, weng_learning_2022}.

We first review work relevant to the kernels in our framework, before discussing work more relevant to selection methods, and finally, data sets. More literature related to specific methods is also discussed in \Cref{sec:kernels} and \Cref{sec:sel_methods}.

\subsection{Uncertainty Measures and Kernel Approximations}

A popular class of BMDAL methods is given by Bayesian methods since the Bayesian framework naturally provides uncertainties that can be used to assess informativeness. These methods require to use Bayesian NNs, or in other words, the calculation of an approximate posterior distribution over NN parameters. 
A simple option is to perform Bayesian inference only over the last layer of the NN \citep{lazaro-gredilla_marginalized_2010, snoek_scalable_2015, ober_benchmarking_2019, kristiadi_being_2020}. The Laplace approximation \citep{laplace_memoire_1774, mackay_bayesian_1992} can provide a local posterior distribution around a local optimum of the loss landscape via a second-order Taylor approximation. An alternative local approach based on SGD iterates is called SWAG \citep{maddox_simple_2019}. Ensembles of NNs \citep{hansen_neural_1990, lakshminarayanan_simple_2017} can be interpreted as a simple multi-modal posterior approximation and can be combined with local approximations to yield mixtures of Laplace approximations \citep{eschenhagen_mixtures_2021} or MultiSWAG \citep{wilson_bayesian_2020}. Monte Carlo (MC) Dropout \citep{gal_dropout_2016} is an option to obtain ensemble predictions from a single NN, although it requires training with Dropout \citep{srivastava_dropout_2014}. Regarding uncertainty approximations, our considered algorithms are mainly related to exact last-layer methods and the Laplace approximation, as these do not require to modify the training process. \cite{daxberger_laplace_2021} give an overview of various methods to compute (approximate) Laplace approximations. 

Some recent approaches also build on the neural tangent kernel (NTK) introduced by \cite{jacot_neural_2018}. \cite{khan_approximate_2019} show that certain Laplace approximations are related to the finite-width NTK. \cite{wang_deep_2022} and \cite{mohamadi_making_2022} propose the use of finite-width NTKs for DAL for classification. \cite{wang_neural_2021} use the finite-width NTK at initialization for the streaming setting of DAL for classification and theoretically analyze the resulting method. \cite{aljundi_identifying_2022} use a kernel related to the finite-width NTK for DAL. \cite{shoham_experimental_2023}, \cite{borsos_coresets_2020} and \cite{borsos_semi-supervised_2021} use infinite-width NTKs for BMDAL and related tasks. \cite{han_random_2021} propose sketching for infinite-width NTKs and also evaluate it for DAL. In contrast to these papers, we propose sketching for finite-width NTKs and allow combining the resulting kernel with different selection methods.

\subsection{Selection Methods}

Besides a Bayesian NN model, a Bayesian BMDAL method needs to specify an acquisition function that decides how to prioritize the pool samples. Many simple acquisition functions for quantifying uncertainty have been proposed \citep{kumar_active_2020}. Selecting the next sample where an ensemble disagrees most is known as Query-by-Committee (QbC) \citep{seung_query_1992}. \cite{krogh_neural_1994} employed QbC for DAL for regression. A more recent investigation of QbC to DAL for classification is performed by \cite{beluch_power_2018}. \cite{pop_deep_2018} combine ensembles with MC Dropout. \cite{tsymbalov_dropout-based_2018} use the predictive variance obtained by MC Dropout for DAL for regression. \cite{zaverkin_exploration_2021} use last-layer-based uncertainty in DAL for regression on atomistic data. Unlike the other approaches mentioned before, the approach by \cite{zaverkin_exploration_2021} can be applied to a single NN trained without Dropout.

Many uncertainty-based acquisition functions do not distinguish between epistemic uncertainty, i.e., lack of knowledge about the true functional relationship, and aleatoric uncertainty, i.e., inherent uncertainty due to label noise. \cite{houlsby_bayesian_2011} propose the BALD acquisition function, which aims to quantify epistemic uncertainty only. 
\cite{gal_deep_2017} apply BALD and other acquisition functions to BMDAL for classification with MC Dropout. To enforce diversity of the selected batch, \cite{kirsch_batchbald_2019} propose BatchBALD and evaluate it on classification problems with MC Dropout. \cite{ash_gone_2021} propose \Bait{}, which also incorporates representativity through Fisher information based on last-layer features, and is evaluated on classification and regression data sets.

Another approach towards BMDAL is to find core-sets that represent $\Xpool$ in a geometric sense. \cite{sener_active_2018} and \cite{geifman_deep_2017} propose algorithms to cover the pool set with $\Xbatch \cup \Xtrain$ in a last-layer feature space. \cite{ash_deep_2019} propose BADGE, which applies clustering in a similar feature space, but includes uncertainty via gradients through the softmax layer for classification. ACS-FW \citep{pinsler_bayesian_2019} can be seen as a hybrid between core-set and Bayesian approaches, trying to approximate the expected log-posterior on the pool set with a core-set, also using last-layer-based Bayesian approximations. Besides \Bait{}, ACS-FW is one of the few approaches that is designed and evaluated for both classification and regression. Our newly proposed selection method \textsc{LCMD} is clustering-based like the k-means++ method used in BADGE, but deterministic.

Many more approaches towards BMDAL exist, and they can be combined with additional steps such as pre-reduction of $\Xpool$ \citep{ghorbani_data_2022} or re-weighting of selected instances \citep{farquhar_statistical_2021}. Most of these BMDAL methods are geared towards classification, and for a broader overview, we refer to \cite{ren_survey_2021}. For (image) regression, \cite{ranganathan_deep_2020} introduce an auxiliary loss term on the pool set, which they use to perform DAL. It is unclear, though, to which extent their success is explained by implicitly performing semi-supervised learning.

Since we frequently consider Gaussian Processes (GPs) as approximations to Bayesian NNs in this paper, our work is also related to BMAL for GPs, although in our case the GPs are only used for selecting $\Xbatch$ and not for the regression itself. Popular BMAL methods for GPs have been suggested for example by \cite{seo_gaussian_2000} and \cite{krause_near-optimal_2008}.

\subsection{Data Sets}

In terms of benchmark data sets for BM(D)AL for regression, \cite{tsymbalov_dropout-based_2018} use seven large tabular data sets, some of which we have included in our benchmark, cf.\ \Cref{sec:appendix:data_sets}. \cite{pinsler_bayesian_2019} use only one large tabular regression data set. \cite{ash_gone_2021} use a small tabular regression data set and three image regression data sets, two of which are converted classification data sets. \cite{wu_pool-based_2018} benchmarks exclusively on small tabular data sets. \cite{zaverkin_exploration_2021} work with atomistic data sets, which require specialized NN architectures and longer training times, and are therefore less well-suited for a large-scale benchmark. \cite{ranganathan_deep_2020} use CNNs on five image regression data sets. Recently, a benchmark for BMDAL for drug discovery has been proposed, which uses four counterfactual regression data sets \citep{mehrjou_genedisco_2021}.
In this paper, we provide an open-source benchmark on 15 large tabular data sets, which includes more baselines and evaluation criteria than evaluations in previous papers.

\section{Kernels} \label{sec:kernels}

In this section, we discuss a variety of base kernels yielding various approximations to a trained NN $f_{\bftheta_T}$, as well as different kernel transformations that yield new kernels with different meanings or simply improved efficiency. In the following, we consider positive semi-definite kernels $k: \bbR^d \times \bbR^d \to \bbR$. For an introduction to kernels, we refer to the literature \citep[e.g.\ ][]{steinwart_support_2008}. The kernels considered here can usually be represented by a feature map $\phi$ with finite-dimensional feature space, that is, $\phi: \bbR^d \to \bbR^{\dfeat}$ with $k(\bfx_i, \bfx_j) = \langle \phi(\bfx_i), \phi(\bfx_j)\rangle$. For a sequence $\calX = (\bfx_1, \hdots, \bfx_n)$ of inputs, which we sometimes treat like a set $\calX \subseteq \bbR^d$ by a slight abuse of notation, we define the corresponding feature matrix
\begin{equation}
    \phi(\calX) = \begin{pmatrix}
    \phi(\bfx_1)^\top \\ \vdots \\ \phi(\bfx_n)^\top
    \end{pmatrix} \in \bbR^{n \times \dfeat}
\end{equation}
and kernel matrices $k(\bfx, \calX) = (k(\bfx, \bfx_i))_i \in \bbR^{1 \times n}$, $k(\calX, \calX) = (k(\bfx_i, \bfx_j))_{ij} \in \bbR^{n \times n}$, $k(\calX, \bfx) = (k(\bfx_i, \bfx))_i \in \bbR^{n \times 1}$.

\subsection{Base Kernels} \label{sec:base_kernels}

We first discuss various options for creating base kernels that induce some notion of similarity on the training and pool inputs. An overview of these base kernels can be found in \Cref{table:base_kernels}.

\begin{table}[bt]
\centering
\begin{tabular}{cccc}
Base kernel & Symbol & Feature map & Feature space dimension $\dfeat$ \\
\hline 
Linear & $k_{\mathrm{lin}}$ & $\phi_{\mathrm{lin}}(\bfx) = \bfx$ & $d$ \\
NNGP & $k_{\mathrm{nngp}}$ & not explicitly defined & $\infty$ \\
full gradient & $k_{\mathrm{grad}}$ & $\phi_{\mathrm{grad}}(\bfx) = \nabla_{\bftheta} f_{\bftheta_T}(\bfx)$ & $\sum_{l=1}^L d_l (d_{l-1}+1)$ \\
last-layer & $k_{\mathrm{ll}}$ & $\phi_{\mathrm{ll}}(\bfx) = \nabla_{\tilde \bfW^{(L)}} f_{\bftheta_T}(\bfx)$ & $d_L (d_{L-1}+1)$ \\
\end{tabular}
\caption{Overview of the introduced base kernels.} \label{table:base_kernels}
\end{table}

\subsubsection{Linear Kernel} A very simple baseline for other base kernels is the linear kernel $k_{\mathrm{lin}}(\bfx, \tbfx) = \langle \bfx, \tbfx \rangle$, corresponding to the identity feature map
\begin{IEEEeqnarray*}{+rCl+x*}
\phi_{\mathrm{lin}}(\bfx) \equalDef \bfx~.
\end{IEEEeqnarray*}
It is usually very fast to evaluate but does not represent the behavior of an NN well. Moreover, its feature space dimension depends on the input dimension, and hence may not be suited for selection methods that depend on having high-dimensional representations of the data. A more accurate representation of the behavior of an NN is given by the next kernel:

\subsubsection{Full gradient Kernel} If $\bftheta_T$ is the parameter vector of the trained NN, we define
\begin{IEEEeqnarray*}{+rCl+x*}
\phi_{\mathrm{grad}}(\bfx) \equalDef \nabla_{\bftheta} f_{\bftheta_T}(\bfx)~.
\end{IEEEeqnarray*}
This is motivated as follows: A linearization of the NN with respect to its parameters around $\bftheta_T$ is given by the first-order Taylor expansion
\begin{equation}
    f_{\bftheta}(\bfx) \approx \tilde f_{\bftheta}(\bfx) \equalDef f_{\bftheta_T}(\bfx) + \langle \phi_{\mathrm{grad}}(\bfx), \bftheta - \bftheta_T\rangle~. \label{eq:linearization}
\end{equation}
If we were to resume training from the parameters $\bftheta_T$ after labeling the next batch $\Xbatch$, the result of training on the extended data could hence be approximated by the function $f_{\bftheta_T} + f_\Delta$, where $f_\Delta$ is the result of linear regression with feature map $\phi_{\mathrm{grad}}$ on the data residuals $(\bfx_i, y_i - f_{\bftheta_T}(\bfx_i))$ for $(\bfx_i, y_i) \in \Dtrain \cup \Dbatch$. 

The kernel $k_{\mathrm{grad}}$ is also known as the (empirical / finite-width) \emph{neural tangent kernel} (NTK). It depends on the linearization point $\bftheta_T$, but can for certain training settings converge to a fixed kernel as the hidden layer widths go to infinity \citep{jacot_neural_2018, lee_wide_2019, arora_exact_2019}. %
In practical settings, however, it has been observed that $k_{\mathrm{grad}}$ often \quot{improves} during training \citep{fort_deep_2020, long_properties_2021, shan_theory_2021, atanasov_neural_2021}, especially in the beginning of training. This agrees with our observations in \Cref{sec:bmdal:experiments} and suggests that shorter training might already yield a gradient kernel that allows selecting a good $\Xbatch$. Indeed, \cite{coleman_selection_2019} found that shorter training and even smaller models can already be sufficient to select good batches for BMDAL for classification.

For fully-connected layers, we will now show that the feature map $\phi_{\mathrm{grad}}$ has an additional product structure that can be exploited to reduce the runtime and memory consumption of a kernel evaluation. For notational simplicity, we rewrite \eqref{eq:fcnn} as
\begin{IEEEeqnarray*}{+rCl+x*}
\bfz_i^{(l+1)} & = & \tilde \bfW^{(l+1)} \tilde \bfx_i^{(l)}, \\
\tilde \bfW^{(l+1)} & \equalDef & \begin{pmatrix}
        \bfW^{(l+1)} & \bfb^{(l+1)}
        \end{pmatrix} \in \bbR^{d_{l+1} \times (d_l + 1)}, \quad \tilde \bfx_i^{(l)} \equalDef \begin{pmatrix}
        \frac{\sigma_w}{\sqrt{d_l}} \bfx_i^{(l)} \\
        \sigma_b
        \end{pmatrix} \in \bbR^{d_l + 1}~, \IEEEyesnumber \label{eq:fcnn_combined}
\end{IEEEeqnarray*}
with parameters $\bftheta = (\tilde \bfW^{(1)}, \hdots, \tilde \bfW^{(L)})$. Using the notation from \eqref{eq:fcnn_combined}, we can write
\begin{equation}
    \phi_{\mathrm{grad}}(\bfx_i^{(0)}) = \left(\frac{\diff \bfz_i^{(L)}}{\diff \tilde \bfW^{(1)}}, \hdots, \frac{\diff \bfz_i^{(L)}}{\diff \tilde \bfW^{(L)}}\right) = \left(\frac{\diff \bfz_i^{(L)}}{\diff \bfz_i^{(1)}} (\tilde \bfx_i^{(0)})^\top, \hdots, \frac{\diff \bfz_i^{(L)}}{\diff \bfz_i^{(L)}} (\tilde \bfx_i^{(L-1)})^\top\right)~. \label{eq:grad_fm_factorization}
\end{equation}
For a kernel evaluation, the factorization of the weight matrix derivatives can be exploited via
\begin{IEEEeqnarray*}{+rCl+x*}
    k_{\mathrm{grad}}(\bfx_i^{(0)}, \bfx_j^{(0)}) & = & \sum_{l=1}^L \left\langle \frac{\diff \bfz_i^{(L)}}{\diff \bfz_i^{(l)}} (\tilde \bfx_i^{(l-1)})^\top, \frac{\diff \bfz_j^{(L)}}{\diff \bfz_j^{(l)}} (\tilde \bfx_j^{(l-1)})^\top \right\rangle_F \\
    & = & \sum_{l=1}^L \underbrace{\left\langle \tilde \bfx_i^{(l-1)}, \tilde \bfx_j^{(l-1)}\right\rangle}_{\defEqual k^{(l)}_{\mathrm{in}}(\bfx_i^{(0)}, \bfx_j^{(0)})} \cdot \underbrace{\left\langle \frac{\diff \bfz_i^{(L)}}{\diff \bfz_i^{(l)}}, \frac{\diff \bfz_j^{(L)}}{\diff \bfz_j^{(l)}}\right\rangle}_{\defEqual k^{(l)}_{\mathrm{out}}(\bfx_i^{(0)}, \bfx_j^{(0)})}~, \IEEEyesnumber \label{eq:grad_comp}
\end{IEEEeqnarray*}
since $\langle \bfa \bfb^\top, \bfc \bfd^\top \rangle_F = \tr(\bfb \bfa^\top \bfc \bfd^\top) = \tr(\bfa^\top \bfc \bfd^\top \bfb) = \bfa^\top \bfc \bfd^\top \bfb = \langle \bfa, \bfc\rangle \cdot \langle \bfb, \bfd \rangle$. This means that $k_{\mathrm{grad}}$ can be decomposed into sums of products of kernels with smaller feature space dimension:\footnote{For the sketching method defined later, we may exploit that $k^{(L)}_{\mathrm{out}}(\bfx, \tilde\bfx) = 1$, hence $k^{(L)}_{\mathrm{out}}$ can be omitted.}
\begin{equation}
k_{\mathrm{grad}}(\bfx, \tilde\bfx) = \sum_{l=1}^L k^{(l)}_{\mathrm{in}}(\bfx, \tilde\bfx) \cdot k^{(l)}_{\mathrm{out}}(\bfx, \tilde\bfx) \label{eq:grad_kernel_factorization}
\end{equation}
When using \eqref{eq:grad_comp}, the full gradients $\frac{\diff \bfz_i^{(L)}}{\diff \tilde \bfW^{(l)}}$ never have to be computed or stored explicitly. If $\frac{\diff \bfz_i^{(L)}}{\diff \bfz_i^{(l)}}$ and $\tilde \bfx_i^{(l-1)}$ are already computed and the hidden layers contain $m = d_1 = \hdots = d_{L-1}$ neurons each, \eqref{eq:grad_comp} reduces the runtime complexity of a kernel evaluation from $\Theta(m^2L)$ to $\Theta(mL)$, and similarly for the memory complexity of pre-computed features. In \Cref{sec:kernel_transformations:rp}, we will see how to further accelerate this kernel computation using sketching. Efficient computations of $k_{\mathrm{grad}}$ for more general types of layers and multiple output neurons are discussed by \cite{novak_fast_2022}.

Since $k_{\mathrm{grad}}$ consists of gradient contributions from multiple layers, it is potentially important that the magnitudes of the gradients in different layers are balanced. We achieve this, at least at initialization, through the use of the neural tangent parameterization \citep{jacot_neural_2018}. 
For other NN architectures, however, it might be desirable to re-weight gradient magnitudes from different layers to improve the results obtained with $k_{\mathrm{grad}}$.

\subsubsection{Last-layer Kernel} A simple and rough approximation to the full-gradient kernel is given by only considering the gradient with respect to the parameters in the last layer:
\begin{IEEEeqnarray*}{+rCl+x*}
\phi_{\mathrm{ll}}(\bfx) \equalDef \nabla_{\tilde \bfW^{(L)}} f_{\bftheta_T}(\bfx)~.
\end{IEEEeqnarray*}
From \eqref{eq:grad_fm_factorization}, it is evident that in the single-output regression case that we are considering, $\phi_{\mathrm{ll}}(\bfx_i^{(0)})$ is simply the input $\tilde \bfx_i^{(L-1)}$ to the last layer of the NN. The latter formulation can also be used in the multi-output setting, and versions of it (with $\bfx_i^{(L-1)}$ instead of $\tilde \bfx_i^{(L-1)}$) have been frequently used for BMDAL \citep{sener_active_2018, geifman_deep_2017, pinsler_bayesian_2019, ash_deep_2019, zaverkin_exploration_2021, ash_gone_2021}. %

\subsubsection{Infinite-width NNGP} It has been shown that as the widths $d_1, \hdots, d_{L-1}$ of the hidden NN layers converge to infinity, the distribution of the initial function $f_{\bftheta_0}$ converges to a Gaussian Process with mean zero and a covariance kernel $k_{\mathrm{nngp}}$ called the neural network Gaussian process (NNGP) kernel \citep{neal_priors_1994, lee_deep_2018, matthews_gaussian_2018}. This kernel depends on the network depth, the used activation function, and details such as the initialization variance and scaling factors like $\sigma_w$. In our experiments, we use the NNGP kernel corresponding to the employed NN setup, for which the formulas are given in \Cref{sec:appendix:nngp}.

As mentioned above, there exists an infinite-width limit of $k_{\mathrm{grad}}$, the so-called neural tangent kernel \citep{jacot_neural_2018}. We decided to omit it from our experiments in \Cref{sec:appendix:experiments} after preliminary experiments showed similarly bad performance as for the NNGP.

\subsection{Kernel Transformations} \label{sec:kernel_transformations}

The base kernels introduced in \Cref{sec:base_kernels} are constructed such that kernel regression with these kernels serves as a proxy for regression with the corresponding NN. By using kernels, we can model interactions $k(\bfx, \tbfx)$ between two inputs, which is crucial to incorporate diversity (DIV) into the selection methods. However, this is not always sufficient to apply a selection method.
For example, sometimes we want the kernel to represent uncertainties of the NN after observing the data, or we want to reduce the feature space dimension to render selection more efficient. Therefore, we introduce various ways to transform kernels in this section. 
When applying transformations $T_1, \hdots, T_n$ in this order to a base kernel $k_{\mathrm{base}}$, we denote the transformed kernel by $k_{\mathrm{base} \to T_1 \to T_2 \to \hdots \to T_n}$. Of course, we can only cover selected transformations relevant to our applications, and other transformations such as sums or products of kernels are possible as well.

\begin{table}[tb]
\centering
\begin{tabular}{cC{0.34\textwidth}ccc}
Notation & Description & $\dfeatin$ & $\dfeatout$ & Configurable $\sigma^2$? \\
\hline
$k_{\to \scale{\calX}}$ & Rescale kernel to normalize mean $k(\bfx, \bfx)$ on $\calX$ & any & $\dfeatin$ & no \\
$k_{\to \post{\calX,\sigma^2}}$ & GP posterior covariance after observing $\calX$ & any & $\dfeatin$ & yes  \\
$k_{\to \calX}$ & Short for $k_{\to\scale{\calX}\to\post{\calX,\sigma^2}}$ & any & $\dfeatin$ & yes \\
$k_{\to\rp{p}}$ & Sketching with $p$ features & $< \infty$ & $p$ & no \\
$k_{\to\ens{\Nens}}$ & Sum of kernels for $\Nens$ ensembled networks & any & $\Nens\dfeatin$ & no \\
$k_{\to\acsgrad}$ & Gradient-based kernel from \cite{pinsler_bayesian_2019} & any & $\dfeatin^2$ & yes \\
$k_{\to\acsrf{p}}$ & Kernel from \cite{pinsler_bayesian_2019} with $p$ random features & $< \infty$ & $p$ & yes \\
$k_{\to\acsrfhyper{p}}$ & Kernel from \cite{pinsler_bayesian_2019} with $p$ random features and hyperprior on $\sigma^2$ & $< \infty$ & $p$ & no
\end{tabular}
\caption{Overview of our considered kernel transformations that can be applied to a kernel $k$. Here, $\dfeatin$ refers to the feature space dimension of $k$ and $\dfeatout$ refers to the feature space dimension after the transformation. Moreover, $\sigma^2$ refers to the assumed noise variance in the GP model.} \label{table:kernel_transformations}
\end{table}

\subsubsection{Scaling} \label{sec:kernel_transformations:scaling} For a given kernel $k$ with feature map $\phi$ and scaling factor $\lambda \in \bbR$, we can construct the kernel $\lambda^2 k$ with feature map $\lambda \phi$. This scaling can make a difference if we subsequently consider a Gaussian Process (GP) with covariance function $\lambda^2 k$. In this case, $\lambda^2 k(\bfx, \tbfx)$ describes the covariance between $f(\bfx)$ and $f(\tbfx)$ under the prior distribution over functions $f$. Since we train with normalized labels, $\Ntrain^{-1} \sum_{y \in \Ytrain} y_i^2 \approx 1$, we would like to choose the scaling factor $\lambda$ such that $\Ntrain^{-1} \sum_{\bfx \in \Xtrain} \lambda^2 k(\bfx, \bfx) = 1$. Therefore, we propose the automatic scale normalization
\begin{IEEEeqnarray*}{+rCl+x*}
k_{\to\scale{\Xtrain}}(\bfx, \tilde\bfx) \equalDef \lambda^2 k(\bfx, \tilde\bfx), \qquad \lambda \equalDef \left(\frac{1}{\ntrain} \sum_{\bfx \in \Xtrain} k(\bfx, \bfx)\right)^{-1/2}~.
\end{IEEEeqnarray*}

\subsubsection{Gaussian Process Posterior Transformation}  \label{sec:kernel_transformations:post}
For a given kernel $k$ with corresponding feature map $\phi$, we can consider a Gaussian Process (GP) with kernel $k$, which is equivalent to a Bayesian linear regression model with feature map $\phi$: In feature space, we model our observations as $y_i = \bfw^\top \phi(\bfx_i) + \varepsilon_i$ with weight prior $\bfw \sim \calN(\bfzero, \bfI)$ and i.i.d.\ observation noise $\varepsilon_i \sim \calN(0, \sigma^2)$. The random function $f(\bfx_i) \equalDef \bfw^\top \phi(\bfx_i)$ now has the covariance function $\Cov(f(\bfx_i), f(\bfx_j)) = \phi(\bfx_i)^\top \phi(\bfx_j) = k(\bfx_i, \bfx_j)$.

It is well-known, see e.g.\ Section 2.1 and 2.2 in \cite{bishop_pattern_2006}, that the posterior distribution of a Gaussian process after observing the training data $\Dtrain$ with inputs $\Xtrain$ is also a Gaussian process with kernel
\begin{IEEEeqnarray*}{+rCl+x*}
&& k_{\to\post{\Xtrain, \sigma^2}}(\bfx, \tbfx) \\
 & \equalDef & \Cov(f(\bfx), f(\tbfx) \mid \Xtrain, \Ytrain) \\
    &=& k(\bfx, \tbfx) - k(\bfx, \Xtrain) (k(\Xtrain, \Xtrain) + \sigma^2 \bfI)^{-1} k(\Xtrain, \tbfx)\quad  \IEEEyesnumber \label{eq:posterior_kernel} \\
    &\stackrel{\text{see below}}{=}& \phi(\bfx)^\top (\sigma^{-2} \phi(\Xtrain)^\top \phi(\Xtrain) + \bfI)^{-1} \phi(\tbfx) \IEEEyesnumber \label{eq:posterior_fm} \\
    &=& \sigma^2 \phi(\bfx)^\top (\phi(\Xtrain)^\top \phi(\Xtrain) + \sigma^2 \bfI)^{-1} \phi(\tbfx)~. \IEEEyesnumber \label{eq:posterior_fm_other}
\end{IEEEeqnarray*}
Here, the equivalence between \eqref{eq:posterior_kernel} and \eqref{eq:posterior_fm} for $\sigma^2 > 0$ can be obtained using the Woodbury matrix identity. In our implementation, we use the feature map version, \eqref{eq:posterior_fm}, whenever $\dfeat \leq \max\{1024, 3|\Xtrain|\}$. An explicit feature map can be obtained from \eqref{eq:posterior_fm_other} as
\begin{IEEEeqnarray*}{+rCl+x*}
\phi_{\to\post{\Xtrain,\sigma^2}}(\bfx) & = & \sigma (\phi(\Xtrain)^\top \phi(\Xtrain) + \sigma^2 \bfI)^{-1/2} \phi(\bfx)~.
\end{IEEEeqnarray*}
If the posterior with respect to two disjoint sets of inputs $\calX_1, \calX_2 \subseteq \bbR^d$ is sought, it is equivalent to condition first on $\calX_1$ and then on $\calX_2$:
\begin{IEEEeqnarray*}{+rCl+x*}
k_{\to\post{\calX_1\cup\calX_2, \sigma^2}}(\bfx, \tilde\bfx) = k_{\to\post{\calX_1,\sigma^2}\to\post{\calX_2,\sigma^2}}(\bfx, \tilde\bfx)~. \IEEEyesnumber \label{eq:posterior_composition}
\end{IEEEeqnarray*}
In our experiments, we rescale kernels before applying the posterior transformation, which we abbreviate by
\begin{IEEEeqnarray*}{+rCl+x*}
k_{\to\Xtrain}(\bfx, \tilde\bfx) \equalDef k_{\to\scale{\Xtrain}\to\post{\Xtrain,\sigma^2}}(\bfx, \tilde\bfx)~.
\end{IEEEeqnarray*}

The application of the posterior transformation to network-dependent kernels can be seen as an instance of approximate inference with Bayesian NNs. Specifically, we show in \Cref{sec:appendix:posterior} that $k_{\mathrm{ll}\to\post{\Xtrain, \sigma^2}}$ and $k_{\mathrm{grad}\to\post{\Xtrain, \sigma^2}}$ correspond to last-layer and generalized Gauss-Newton (GGN) approximations to the Hessian in a Laplace approximation \citep{laplace_memoire_1774, mackay_bayesian_1992} for Bayesian NNs, see also \cite{khan_approximate_2019}. Moreover, in the case of the last-layer kernel, this procedure is equivalent to interpreting the last layer of the NN as a Bayesian linear regression model. %

\subsubsection{Sketching} \label{sec:kernel_transformations:rp} 
Sketching methods, which allow approximating matrices like $\phi(\calX)$ with smaller matrices in some sense, can be used to approximate a kernel $k$ with high-dimensional feature space by a kernel with a lower-dimensional feature space \citep[see e.g.][]{woodruff_sketching_2014}. 
For example, $k_{\mathrm{grad}}$ and kernels resulting from the ACS gradient transformation introduced in \Cref{sec:kernel_transformations:acsgrad} involve product kernels with very high-dimensional feature spaces ($\dfeat > 250,000$). In our experiments, we apply sketching mainly to these kernels. This is especially useful for methods 
such as the posterior transformation discussed previously and the \FrankWolfe{} and \Bait{} selection methods explained in \Cref{sec:specific_sel_methods}, which are not very efficient in the kernel formulation.

We sketch finite-dimensional feature maps as follows: %
\begin{enumerate}[(1)]
\item \textbf{Generic finite-dimensional feature maps:} Consider a generic kernel $k$ with finite-dimensional feature map $\phi: \bbR^{d_0} \to \bbR^{\dfeatin}$. For a random vector $\bfu \sim \calN(0, \bfI_{\dfeatin})$, a single random feature is given by the feature map $\phi_{\bfu}(\bfx) \equalDef \bfu^\top \phi(\bfx)$, which yields an unbiased estimate of the kernel since $\bbE_{\bfu} \langle \phi_{\bfu}(\bfx), \phi_{\bfu}(\tilde\bfx)\rangle = k(\bfx, \tilde\bfx)$. By combining multiple such random features, the accuracy of the kernel approximation can be improved. For $p$ random features, we obtain the random feature map %
\begin{IEEEeqnarray*}{+rCl+x*}
\phi_{\to\rp{p}}(\bfx) \equalDef \frac{1}{\sqrt{p}} \bfU \phi(\bfx) \in \bbR^{p}~, \IEEEyesnumber \label{eq:gaussian_sketch}
\end{IEEEeqnarray*}
where $\bfU \in \bbR^{p \times \dfeatin}$ is a random matrix with i.i.d.\ standard normal entries. This is also known as a Gaussian sketch. %

In terms of the kernel distance
\begin{IEEEeqnarray*}{+rCl+x*}
d_k(\bfx, \tilde\bfx) \equalDef \|\phi(\bfx) - \phi(\tilde\bfx)\|_2 = \sqrt{k(\bfx, \bfx) + k(\tilde\bfx, \tilde\bfx) - 2k(\bfx, \tilde\bfx)}~, \IEEEyesnumber \label{eq:kernel_distance}
\end{IEEEeqnarray*}
the approximation quality of the sketched kernel can be analyzed using variants of the celebrated Johnson-Lindenstrauss lemma \citep{johnson_extensions_1984}. For example, the following variant is proved in \Cref{sec:appendix:random_projections} based on a result by \cite{arriaga_algorithmic_1999}. \\

\begin{restatable}[Variant of the Johnson-Lindenstrauss Lemma]{theorem}{thmjl} \label{thm:jl}
Let $\varepsilon, \delta \in (0, 1)$ and let $\calX \subseteq \bbR^d$ be finite. If
\begin{equation}
    p \geq 8\log(|\calX|^2/\delta) / \varepsilon^2~, \label{eq:min_dim}
\end{equation}
then the following bound on all pairwise distances holds with probability $\geq 1-\delta$ for the Gaussian sketch in \eqref{eq:gaussian_sketch}:
\begin{equation}
    \forall \bfx, \tilde\bfx \in \calX: (1-\varepsilon) d_k(\bfx, \tilde\bfx) \leq d_{k_{\to\rp{p}}}(\bfx, \tilde\bfx) \leq (1+\varepsilon) d_k(\bfx, \tilde\bfx)~. \label{eq:eps_isometry}
\end{equation}
\end{restatable}
Note that, counterintuitively, the lower bound in \eqref{eq:min_dim} does not depend on the feature space dimension $\dfeatin$ of $k$.

\item \textbf{Sums of kernels:} Consider the sum kernel $k = k_1 + k_2$ of kernels $k_1, k_2$ with finite-dimensional feature maps $\phi_1, \phi_2$. A feature map for $k$ is given by $\phi(\bfx) \equalDef \begin{pmatrix}\phi_1(\bfx) \\ \phi_2(\bfx)\end{pmatrix}$. We can apply sketching as
\begin{equation}
\phi_{\to\rp{p}}(\bfx) \equalDef \phi_{1\to\rp{p}}(\bfx) + \phi_{2\to\rp{p}}(\bfx)~. \label{eq:sum_sketch}
\end{equation}
This again yields an unbiased estimate of the kernel $k$. If $\phi_{1\to\rp{p}}$ and $\phi_{2\to\rp{p}}$ are sketched as in \eqref{eq:gaussian_sketch}, then \eqref{eq:sum_sketch} is equivalent to sketching $\phi$ with \eqref{eq:gaussian_sketch} directly.

\item \textbf{Products of kernels:} Consider the product kernel $k = k_1 \cdot k_2$ of kernels $k_1, k_2$ with finite-dimensional feature maps $\phi_1, \phi_2$. A feature map for $k$ is given by $\phi(\bfx) \equalDef \phi_1(\bfx) \otimes \phi_2(\bfx)$, where $\otimes$ is the tensor product. Hence, if the feature spaces of $\phi_1$ and $\phi_2$ have dimensions $p_1$ and $p_2$, respectively, the feature space of $\phi$ has dimension $p_1 p_2$. While this dimension is still finite, using \eqref{eq:gaussian_sketch} for sketching would potentially require a large amount of memory for storing $\bfU$ as well as a large runtime for the matrix-vector product. Therefore, we sketch product kernels more efficiently as
\begin{equation}
\phi_{\to\rp{p}}(\bfx) \equalDef \sqrt{p} \phi_{1\to\rp{p}}(\bfx) \odot \phi_{2\to\rp{p}}(\bfx)~, \label{eq:prod_sketch}
\end{equation}
where $\odot$ denotes the element-wise product (or Hadamard product). This again yields an unbiased estimator of $k$ without the need to perform computations in the $p_1p_2$-dimensional feature space. While this simple tensor sketching method works sufficiently well for our purposes, its approximation properties are suboptimal and can be improved with a more complicated sketching method \citep{ahle_oblivious_2020}. %
\end{enumerate}

In the kernel $k_{\mathrm{grad}\to\rp{p}\to\post{\Xtrain, \sigma^2}}$, the inclusion of sketching can be considered a further approximation to the posterior predictive distribution for Bayesian NNs. In this context, a different sketching method has been proposed by \cite{sharma_sketching_2021}. It is also possible to apply sketching to kernels with infinite-dimensional feature space \citep[see e.g.][]{kar_random_2012, zandieh_scaling_2021, han_random_2021}, but such kernels are less relevant in our case.

\subsubsection{Ensembling} Ensembles of NNs have been demonstrated to yield good uncertainty estimates for DAL \citep{beluch_power_2018} and can improve the uncertainty estimates of MC Dropout \citep{pop_deep_2018}. This motivates the study of ensembled kernels. When multiple NNs are trained on the same data and a kernel $k^{(i)}$ is computed for each model $i \in \{1, \hdots, \Nens\}$ via a base kernel and a list of transformations, these kernels can be ensembled simply by adding them together:
\begin{IEEEeqnarray*}{+rCl+x*}
k_{\to\mathrm{ens}(\Nens)} = k^{(1)} + \hdots + k^{(\Nens)}~.
\end{IEEEeqnarray*}

In the context of Bayesian NNs, ensembling of posterior kernels is related to a mixture of Laplace approximations \citep{eschenhagen_mixtures_2021}, cf.\ \Cref{sec:appendix:kernel_transformations}. %

\subsubsection{ACS Random Features Transformation} \label{sec:kernel_transformations:acs_rf}
In the following two paragraphs, we will briefly introduce multiple kernel transformations corresponding to various alternative ways of applying the ACS-FW method by \cite{pinsler_bayesian_2019} to GP regression. For a more complete description, we refer to the original publication.
ACS-FW seeks to approximate the expected complete data log posterior, $f_{\mathrm{pool}}(\bftheta) = \bbE_{\Ypool \sim P(\Ypool | \Xpool, \Dtrain)} \log p(\bftheta \mid \Dtrain, \Xpool, \Ypool)$, with the expected log posterior of the train data and the next batch, $f_{\mathrm{batch}}(\bftheta) = \bbE_{\Ybatch \sim P(\Ybatch | \Xbatch, \Dtrain)} \log p(\bftheta \mid \Dtrain, \Xbatch, \Ybatch)$. Here, the labels $\Ypool$ and $\Ybatch$ are drawn from the posterior distribution after observing $\Dtrain$. For a given Bayesian model, they then define different kernels resulting from this objective. As a Bayesian model, we use the same Gaussian process model as for the posterior transformation above, with kernel $k_{\to\scale{\Xtrain}}$. 
In this case, as shown in \Cref{sec:appendix:acs_rf}, we have $f_{\mathrm{pool}}(\bftheta) - f_{\mathrm{batch}}(\bftheta) = \sum_{\bfx \in \Xpool \setminus \Xbatch} \facs(\bfx, \bftheta)$ with
\begin{IEEEeqnarray*}{+rCl+x*}
\facs(\bfx, \bftheta) \equalDef \frac{1}{2} \log\left(1 + \frac{k_{\to\Xtrain}(\bfx, \bfx)}{\sigma^2}\right) - \frac{(\bftheta^\top \phi_{\to\scale{\Xtrain}}(\bfx))^2 + k_{\to\Xtrain}(\bfx, \bfx)}{2\sigma^2}~.\quad \IEEEyesnumber \label{eq:acs_function}
\end{IEEEeqnarray*}
The weighted inner product by \citep{pinsler_bayesian_2019} can then be written as
\begin{IEEEeqnarray*}{+rCl+x*}
k_{\to\acs}(\bfx, \tilde\bfx) \equalDef \bbE_{\bftheta \sim P(\bftheta \mid \Dtrain)} [\facs(\bfx, \bftheta) \facs(\tilde\bfx, \bftheta)]~. \IEEEyesnumber \label{eq:acs_kernel}
\end{IEEEeqnarray*}
The expectation in \eqref{eq:acs_kernel} can be approximated using Monte Carlo quadrature as
\begin{IEEEeqnarray*}{+rCl+x*}
k_{\to\acs}(\bfx, \tilde\bfx) & \approx & k_{\to\acsrf{\dfeatout}} \equalDef \frac{1}{\dfeatout} \sum_{i=1}^{\dfeatout} [\facs(\bfx, \bftheta^{(i)}) \facs(\tilde\bfx, \bftheta^{(i)})]~,
\end{IEEEeqnarray*}
where $\bftheta^{(i)} \sim P(\bftheta \mid \Dtrain)$ are i.i.d.\ parameter samples from the posterior. This corresponds to the random features approximation proposed by \cite{pinsler_bayesian_2019}, given by
\begin{IEEEeqnarray*}{+rCl+x*}
\phi_{\to\acsrf{\dfeatout}}(\bfx) \equalDef \dfeatout^{-1/2} (\facs(\bfx, \bftheta^{(1)}), \hdots, \facs(\bfx, \bftheta^{(\dfeatout)}))^\top~.
\end{IEEEeqnarray*}

In their regression experiments, \cite{pinsler_bayesian_2019} use a slightly different feature map which we denote by $\phi_{\to\acsrfhyper{\dfeatout}}$ in our experiments. They state that this has been derived from a GP model with a hyperprior on $\sigma^2$, although no hints on its derivation are provided in their paper, so we directly use their source code for our implementation.

\subsubsection{ACS Gradient Transformation} \label{sec:kernel_transformations:acsgrad}
As an alternative to the random features approximation in the previous paragraph, \cite{pinsler_bayesian_2019} proposed the \emph{weighted Fisher inner product} given by
\begin{IEEEeqnarray*}{+rCl+x*}
k_{\to\acsgrad}(\bfx, \tilde\bfx) \equalDef \bbE_{\bftheta \sim P(\bftheta \mid \Dtrain)} [\langle \nabla_{\bftheta} \facs(\bfx, \bftheta),  \nabla_{\bftheta} \facs(\tilde\bfx, \bftheta)\rangle]~. \IEEEyesnumber \label{eq:acs_grad_kernel}
\end{IEEEeqnarray*}
Under the Gaussian process model, they showed that an explicit formula for $k_{\to\acsgrad}$ is given by
\begin{IEEEeqnarray*}{+rCl+x*}
k_{\to\acsgrad}(\bfx, \tilde\bfx) = \frac{1}{\sigma^4} k_{\to\scale{\Xtrain}}(\bfx, \tilde\bfx) k_{\to\Xtrain}(\bfx, \tilde\bfx)~.
\end{IEEEeqnarray*}
Since this product kernel can have a high-dimensional feature space, they used $k_{\to\acsgrad}$ only on small data sets. In our experiments, we apply sketching to this kernel to scale it to large data sets. Again, they specified that they included a hyperprior on $\sigma^2$ in their experiments, but their corresponding implementation appears to be equivalent to $k_{\to\acsgrad}$ for our purposes.

\subsection{Discussion}

Out of the base kernels and kernel transformations considered above, the training labels $\Ytrain$ only influence the base kernels $k_{\mathrm{ll}}$ and $k_{\mathrm{grad}}$ through the trained parameters $\bftheta_T$, and to some extent the transformation $k_{\to\acsrfhyper{p}}$. Using $k_{\mathrm{lin}}$ and $k_{\mathrm{nngp}}$ with the selection methods below thus leads to passive learning or experimental design, where the entire set of training inputs $\Xtrain$ is selected before any of the labels $\Ytrain$ are computed. This can be much cheaper because no NN retraining is required, but is also potentially less accurate.

Another consideration to be made when selecting a kernel is the feature space dimension $\dfeat$. While a low $\dfeat$ is usually beneficial for runtime purposes, larger $\dfeat$ might allow for a more accurate representation of over-parameterized NNs. For $k_{\mathrm{lin}}$, $\dfeat$ depends on the data set but is often rather small. For $k_{\mathrm{ll}}$, $\dfeat$ can be reduced using sketching, but an effective increase of $\dfeat$ requires increasing the width of the last hidden layer, which might not always be desirable. For $k_{\mathrm{grad}}$, $\dfeat$ is typically very large, and can be flexibly adjusted by using sketching. For $k_{\mathrm{nngp}}$, we have $\dfeat = \infty$, but sketching may also be applicable, see \cite{han_random_2021} for a similar application to infinite-width NTKs.

\section{Selection Methods} \label{sec:sel_methods}

In the following, we will discuss a variety of kernel-based selection methods. We first introduce the general iterative scheme that all evaluated methods (except \Bait{}-FB) use, with its two variants called P (for pool) and TP (for train+pool). Subsequently, we explain specific selection methods. 

\subsection{Iterative Selection Methods} \label{sec:iterative_sel_methods}

A natural approach towards selecting $\Xbatch$ is to formulate an acquisition function $a$ which scores an entire batch, such that $a(\Xbatch)$ should be maximized over all $\Xbatch \subseteq \Xpool$ of size $\Nbatch$. However, the corresponding optimization problem is often intractable \citep{gonzalez_clustering_1985, civril_exponential_2013}.
Many BMAL methods thus select points in a greedy/iterative fashion. To favor samples with high informativeness in an iterative selection scheme that tries to enforce diversity of the selected batch, two approaches can be used:
\begin{enumerate}[(a)]
\item[(P)] Informativeness can be incorporated through the kernel. For example, $k_{\to\Xtrain}(\bfx, \bfx)$ represents the posterior variance at $\bfx$ of a GP with scaled kernel $k_{\to\scale{\Xtrain}}$.
\item[(TP)] Informativeness can be incorporated implicitly by enforcing diversity of $\Xtrain \cup \Xbatch$ instead of only enforcing diversity of $\Xbatch$. In other words, a batch that is sufficiently different from the training set typically necessarily contains new information.
\end{enumerate}
An iterative selection template with the two variants P and TP is shown in \Cref{alg:select_simple}, where different choices of \textsc{NextBatch} lead to different selection methods as discussed in \Cref{sec:specific_sel_methods}. 
For simplicity of notation, \Cref{alg:select_simple} does not reuse information in subsequent calls to \textsc{NextBatch}, which however is necessary to make the selection methods more efficient. We provide efficiency-focused pseudocode, which is also used for our implementation, and an analysis of runtime and memory complexities in \Cref{sec:appendix:selection}. Additionally, our implementation usually accelerates kernel computations through suitable precomputations, often by precomputing the features $\phi(\bfx)$ for $\bfx \in \Xtrain \cup \Xpool$.
In our notation, we treat $\Xtrain$ and $\Xpool$ as sets, assuming that all values are distinct. In practice, if multiple identical $\bfx_i$ are contained in $\Xtrain$ and/or $\Xpool$, they should still be treated as distinct.

\begin{algorithm}[tb]
\caption{Iterative selection algorithm template with customizable function \textsc{NextSample}, for which different options will be discussed in \Cref{sec:specific_sel_methods}.} \label{alg:select_simple}
\begin{algorithmic}
\Function{Select}{$k$, $\Xtrain$, $\Xpool$, $\Nbatch$, mode $\in \{$P, TP$\}$}
	\State $\Xbatch \assign \emptyset$
	\State $\Xmode \assign \Xtrain$ if mode $=$ TP else $\emptyset$  \Comment{Points considered as \quot{selected}} %
	\For{$i$ from $1$ to $\Nbatch$}
		
		\State $\Xsel \assign \Xmode \cup \Xbatch$   \Comment{Currently \quot{selected} points}
		\State $\Xrem \assign \Xpool \setminus \Xbatch$   \Comment{Currently unselected points}
		\State $\Xbatch \assign \Xbatch \cup \{\Call{NextSample}{k, \Xsel, \Xrem}\}$
	\EndFor
	\State \Return $\Xbatch$
\EndFunction
\end{algorithmic}
\end{algorithm}

\subsection{Specific Methods} \label{sec:specific_sel_methods}

In the following, we will discuss a variety of choices for \textsc{NextSample} in \Cref{alg:select_simple}, leading to different selection methods. An overview of the resulting selection methods is given in \Cref{table:selection_methods}. \Cref{table:existing_algs} shows how BMAL methods from the literature relate to the presented selection methods and kernels.

\begin{table}[tb]
\centering
\renewcommand{\arraystretch}{1.3}
\begin{tabular}{cC{0.4\textwidth}C{0.32\textwidth}}
Selection method & Description & Runtime complexity \\
\hline
\Random{} & Random selection & $\BigO(\Npool \log \Npool)$ \\
\MaxDiag{} & Naive active learning, picking largest diagonal entries & $\BigO(\Npool(T_k + \log\Npool))$ \\
\MaxDet{} & Greedy determinant maximization & $\BigO(\Ncand\Nsel(T_k + \Nsel))$ or $\BigO(\Ncand\Nsel\dfeat)$ \\
\Bait{}-F & Forward-Greedy total uncertainty minimization & $\BigO(\Ncand\Nsel\dfeat + (\Ntrain+\Npool)\dfeat^2)$ \\
\Bait{}-FB & Forward-Backward-Greedy total uncertainty minimization & $\BigO(\Ncand\Nsel\dfeat + (\Ntrain+\Npool)\dfeat^2)$ \\
\FrankWolfe{} & Approximate kernel mean embedding using Frank-Wolfe & $\BigO((\Ncand + \Npool\Nbatch)\dfeat)$ or $\BigO(\Ncand^2 (T_k + 1))$ \\ %
\MaxDist{} & Greedy distance maximization & $\BigO(\Npool\Nsel (T_k + 1))$ \\
\KMeansPP{} & Next point probability proportional to squared distance & $\BigO(\Npool\Nsel (T_k + 1))$ \\
\LCMD{} (ours) & Greedy distance maximization in largest cluster & $\BigO(\Npool\Nsel (T_k + 1))$
\end{tabular}
\caption{Selection methods presented in this paper and their runtime complexities. For the runtime notation, we let $T_k$ denote the runtime of a kernel evaluation and $\dfeat$ the dimensionality of its (pre-computed) features. Moreover, we write $\Ncand \equalDef \Npool + |\Xmode|$ and $\Nsel \equalDef \Nbatch + |\Xmode|$, with $\Xmode$ as in \Cref{alg:select_simple}. The runtime complexities are derived in \Cref{sec:appendix:selection}. Further refinements of the runtime complexities for \Random{} and \MaxDiag{} are possible but not practically relevant to us, as these methods are already very efficient.
} \label{table:selection_methods}
\end{table}

\begin{table}[tb]
\centering
\renewcommand{\arraystretch}{1.3}
\small
\begin{tabular}{C{0.26\textwidth}cC{0.17\textwidth}C{0.28\textwidth}}
Known as & Selection method & Kernel & Remark \\
\hline
BALD \citep{houlsby_bayesian_2011} & \MaxDiag{} & $k_{\to\post{\Xtrain,\sigma^2}}$ & for GP with kernel $k$ \\
BatchBALD \citep{kirsch_batchbald_2019} & \MaxDet{}-P & $k_{\to\post{\Xtrain,\sigma^2}}$ & for GP with kernel $k$, proposed for classification \\
\Bait{} \citep{ash_gone_2021} & \Bait{}-FB-P & $k_{\mathrm{ll}\to\post{\Xtrain,\sigma^2}}$ \\
ACS-FW \citep{pinsler_bayesian_2019} & \FrankWolfe{}-P & $k_{\mathrm{ll}\to\acsrf{p}}$ or $k_{\mathrm{ll}\to\acsgrad}$ or $k_{\mathrm{ll}\to\acsrfhyper{p}}$ \\
Core-Set \citep{sener_active_2018} & \MaxDist{}-TP${}^*$ & similar to $k_{\mathrm{ll}}$ & proposed for classification \\
FF-Active \citep{geifman_deep_2017} & \MaxDist{}-TP & similar to $k_{\mathrm{ll}}$ & proposed for classification \\
BADGE \citep{ash_deep_2019} & \KMeansPP{}-P & similar to $k_{\mathrm{ll}}$ & proposed for classification \\
\hline
\multicolumn{4}{l}{\footnotesize ${}^*$ This refers to their simpler k-center-greedy selection method.}
\end{tabular}
\caption{Some (regression adaptations of) BM(D)AL methods from the literature and their corresponding selection methods and kernels.} \label{table:existing_algs}
\end{table}

\subsubsection{Random Selection} A simple baseline for comparison to other selection methods, denoted as \Random{}, is to select $\Xbatch$ randomly. We can formally express this as
\begin{IEEEeqnarray*}{+rCl+x*}
\textsc{NextSample}(k, \Xsel, \Xrem) \sim \calU(\Xrem)~,
\end{IEEEeqnarray*}
where $\calU(\Xrem)$ is the uniform distribution over $\Xrem$. Since \textsc{NextSample} does not use $\Xsel$, the P and TP versions of \Random{} are equivalent. %

\subsubsection{Naive Active Learning} If $k(\bfx, \bfx)$ is interpreted as a measure for the uncertainty of the model at $\bfx$, naive active learning can simply be formalized as 
\begin{IEEEeqnarray*}{+rCl+x*}
\NextSample(k, \Xsel, \Xrem) = \argmax_{\bfx \in \Xrem} k(\bfx, \bfx)~. \IEEEyesnumber \label{eq:nextsample_maxdiag}
\end{IEEEeqnarray*}
Since \NextSample{} only considers the diagonal of the kernel matrix $k(\Xrem, \Xrem)$, we call the corresponding selection method \MaxDiag{}. Similar to \Random{}, the P and TP versions of \MaxDiag{} are equivalent. If $k = \tilde k_{\to \Xtrain}$, $k(\bfx, \bfx) + \sigma^2$ represents the posterior predictive variance of a GP with kernel $\tilde k_{\to\scale{\Xtrain}}$ at $\bfx$. 
Unlike in the classification setting, the noise distribution $\varepsilon \sim \calN(0, \sigma^2)$ in the GP model is independent of $\bfx$, which renders different acquisition functions like maximum entropy \citep{shannon_mathematical_1948, mackay_information-based_1992} or BALD \citep{houlsby_bayesian_2011} equivalent to \eqref{eq:nextsample_maxdiag}.
The active learning approach proposed by \cite{zaverkin_exploration_2021} corresponds to applying \MaxDiag{} to $k_{\mathrm{ll}\to\Xtrain}$ in the limit $\sigma^2 \to 0$. Out of the three objectives presented in \Cref{sec:batch_active_learning}, \MaxDiag{} satisfies (INF), but not (DIV) and (REP). 
Indeed, if the pool set contains (almost-) duplicates, \MaxDiag{} may select a batch consisting of (almost) identical inputs.

\subsubsection{Greedy Determinant Maximization} \label{sec:specific_sel_methods:maxdet}
To take account of the inputs $\Xsel$ that have already been selected, it is possible to also condition the GP on the selected values $\Xsel$, since the posterior variance of a GP does not depend on the unknown labels for $\Xsel$. Picking the input with maximal uncertainty after conditioning is equivalent to maximizing a determinant, as we show in \Cref{sec:appendix:maxdet}:
\begin{IEEEeqnarray*}{+rCl+x*}
\NextSample(k, \Xsel, \Xrem) & = & \argmax_{\bfx \in \Xrem} k_{\to\post{\Xsel,\sigma^2}}(\bfx, \bfx) \\
& = & \argmax_{\bfx \in \Xrem} \det(k(\Xsel \cup \{\bfx\}, \Xsel \cup \{\bfx\}) + \sigma^2 \bfI)~, \IEEEyesnumber \label{eq:maxdet}
\end{IEEEeqnarray*}
We call the corresponding selection method \MaxDet{}. It is equivalent to performing non-batch mode active learning on the GP with kernel $k$, and has been applied to GPs by \cite{seo_gaussian_2000}. Moreover, as we show in \Cref{sec:appendix:maxdet}, it is also equivalent to applying BatchBALD \citep{kirsch_batchbald_2019} to the GP with kernel $k$. If $\sigma^2 = 0$, it is equivalent to the $P$-greedy method for kernel interpolation \citep{de_marchi_near-optimal_2005}, and it is also related to the greedy algorithm for D-optimal design \citep{wynn_sequential_1970}. In comparison to a naive implementation that computes each determinant separately, the runtime complexity of the determinant computation in \MaxDet{} can be reduced by a factor of $\BigO(\Nsel^2)$ to $\BigO(\Ncand\Nsel(T_k + \Nsel))$ when implementing \MaxDet{} via a partial pivoted matrix-free Cholesky decomposition as suggested in \cite{pazouki_bases_2011}. For the case $\Nsel \gg \dfeat$, we show in \Cref{sec:appendix:maxdet} how the runtime complexity can be reduced further. For given $\sigma^2 > 0$ in \eqref{eq:maxdet}, it follows from \eqref{eq:posterior_composition} that applying \MaxDet{}-TP to a kernel $k$ is equivalent to applying \MaxDet{}-P to $k_{\to\post{\Xtrain,\sigma^2}}$ (with the same $\sigma^2$). %

\subsubsection{Greedy total uncertainty minimization} \label{sec:specific_sel_methods:bait}
While \MaxDet{} satisfies (INF) and (DIV), it does not satisfy (REP) since it does not incorporate the pool set distribution. To fix this, \cite{ash_gone_2021} propose \Bait{}. The regression version of \Bait{} tries to minimize the sum of the GP posterior variances on the training and pool set.\footnote{There is an \quot{unregularized} version of \Bait{} that extends to classification, but we use the \quot{regularized} version with $\sigma^2 > 0$ here since it is more natural for GPs and avoids numerical issues.} In other words, \Bait{} aims to minimize the acquisition function
\begin{IEEEeqnarray*}{+rCl+x*}
a(\Xsel) & \equalDef & \sum_{\tbfx \in \Xtrain \cup \Xpool} k_{\to\post{\Xsel, \sigma^2}}(\tbfx, \tbfx)~. \IEEEyesnumber \label{eq:bait_acq}
\end{IEEEeqnarray*}
In \Cref{sec:appendix:bait}, we show that \eqref{eq:bait_acq} is equivalent to the original \Bait{} formulation.
This is also known as (Bayesian) V-optimal design \citep{montgomery_design_2017} and a similar method has been studied for NNs by \cite{cohn_neural_1996}. \cite{ash_gone_2021} propose two alternative methods for efficient approximate optimization of this acquisition function. The first one, which we call \Bait{}-F, is to simply use greedy selection as for the other methods in this section. The second alternative, which we call \Bait{}-FB, greedily selects $2\Nbatch$ points (forward step) and then greedily removes $\Nbatch$ points (backward step). In our framework, we can define \Bait{}-F by
\begin{IEEEeqnarray*}{+rCl+x*}
\NextSample(k, \Xsel, \Xrem) & \equalDef & \argmin_{\bfx \in \Xrem} \sum_{\tbfx \in \Xtrain \cup \Xpool} k_{\to \post{\Xsel \cup \{\bfx\}, \sigma^2}}(\tbfx, \tbfx)~.
\end{IEEEeqnarray*}
Details on \Bait{}-F and \Bait{}-FB are given in \Cref{sec:appendix:bait}. Like for \MaxDet{}, applying \Bait{}-F-TP or \Bait{}-FB-TP to a kernel $k$ is equivalent to applying \Bait{}-F-P or \Bait{}-FB-P to $k_{\to\post{\Xtrain,\sigma^2}}$ (with the same $\sigma^2$).

\subsubsection{Frank-Wolfe optimization} In order to make $\Xbatch$ representative of the pool set, \cite{pinsler_bayesian_2019} suggest to choose $\Xbatch$ such that $\sum_{\bfx \in \Xpool} \phi(\bfx)$ is well-approximated by $\sum_{\bfx \in \Xbatch} w_{\bfx} \phi(\bfx)$, where $w_{\bfx}$ are non-negative weights. Specifically, they propose to apply the Frank-Wolfe optimization algorithm to a corresponding optimization problem, which automatically selects elements of $\Xbatch$ iteratively. This can be seen as an attempt to represent the distribution of $\Xpool$ with $\Xbatch$ by approximating the empirical kernel mean embedding $\Npool^{-1} \sum_{\bfx \in \Xpool} k(\bfx, \cdot)$ using $\Xbatch$. The corresponding selection method can be implemented in kernel space or feature space. Since the kernel space version scales quadratically with $\Npool$, \cite{pinsler_bayesian_2019} use the feature space version for large pool sets. We also use the feature space version for our experiments and show the pseudocode in \Cref{sec:appendix:frankwolfe}. While the version by \cite{pinsler_bayesian_2019} allows to select the same $\bfx \in \Xpool$ multiple times, we prohibit this as it would allow to select smaller batches and thus prevent a fair comparison to other methods. %

\subsubsection{Greedy distance maximization} A simple strategy to enforce the diversity of a set of points is to greedily select points with maximum distance to all previously selected points. The resulting algorithm has been frequently proposed in the literature under different names, see \Cref{sec:appendix:maxdist}. In our case, the kernel $k$ gives rise to a distance measure $d_k(\bfx, \tbfx)$ as in \eqref{eq:kernel_distance}.
With this distance measure, the \MaxDist{} selection method is specified by
\begin{IEEEeqnarray*}{+rCl+x*}
\NextSample(k, \Xsel, \Xrem) & = & \argmax_{\bfx \in \Xrem} \min_{\bfx' \in \Xsel} d_k(\bfx, \bfx')~.
\end{IEEEeqnarray*}
If the $\argmax$ is not unique, an arbitrary maximizer is chosen. If $\Xsel$ is empty, we choose $\argmax_{\bfx \in \Xrem} k(\bfx, \bfx)$.

The use of \MaxDist{}-TP  with $k_{\mathrm{lin}}$ for BMAL has been suggested by \cite{yu_passive_2010}, and with a kernel similar to $k_{\mathrm{ll}}$ for BMDAL by \cite{sener_active_2018} and \cite{geifman_deep_2017}. \cite{sener_active_2018} also alternatively propose a more involved discrete optimization algorithm. 
In their experiments, \MaxDist{} yielded only slightly worse results than the more involved optimization algorithm while being significantly faster and easier to implement. They note that the batch selected by \MaxDist{} is suboptimal with respect to a covering objective by at most a factor of two. In \Cref{sec:appendix:maxdist}, we show that a similar guarantee can be given when applying \MaxDist{} to a sketched approximation of the desired kernel. The use of dimensionality reduction for \MaxDist{} has also been analyzed by \cite{eppstein_approximate_2020}. 
Inspired by the reasoning of \cite{sener_active_2018}, we interpret the distances as uncertainty estimates: If both the optimal regression function $f_*$ and the learned regression function $f_{\bftheta_T}$ are $L$-Lipschitz with respect to $d_k$, and we have $y_i = f_*(\bfx_i) = f_{\bftheta_T}(\bfx_i)$ on the training set, then we have the worst-case bound
\begin{IEEEeqnarray*}{+rCl+x*}
|f_{\bftheta_T}(\bfx) - f_*(\bfx)| & \leq & 2L \min_{\tilde\bfx \in \Xtrain} d_k(\bfx, \tilde\bfx)~. \IEEEyesnumber \label{eq:dist_uncertainty}
\end{IEEEeqnarray*}
Of course, the Lipschitz constant $L$ might itself depend on $\Xtrain$, so this should only be interpreted as a crude heuristic.
\cite{wenzel_novel_2021} show that for kernel interpolation with Sobolev kernels, \MaxDist{} and \MaxDet{} with $\sigma^2 = 0$ yield asymptotically equivalent convergence rates.

\subsubsection{$k$-means++ seeding} 
Similar to \MaxDet{}, \MaxDist{} enforces (INF) and (DIV) but not (REP). To incorporate (REP), i.e., sample more points from regions with higher pool set density, we can view batch selection as a clustering problem: 
For example, if the distance-based uncertainty estimate in \eqref{eq:dist_uncertainty} holds, we could try to minimize the corresponding upper bound on the pool set MSE $\frac{1}{\Npool} \sum_{\bfx \in \Xpool} |f_{\bftheta_T}(\bfx) - f_*(\bfx)|^2$ after adding $\Xbatch$:
\begin{IEEEeqnarray*}{+rCl+x*}
\Xbatch = \argmin_{\Xbatch \subseteq \Xpool, |\Xbatch| = \Nbatch} \frac{1}{\Npool} \sum_{\bfx \in \Xpool} \min_{\tilde\bfx \in \Xmode \cup \Xbatch} d_k(\bfx, \tilde\bfx)^2~. \IEEEyesnumber \label{eq:sum_squared_dists}
\end{IEEEeqnarray*}
This optimization problem is essentially the k-medoids problem \citep{kaufman_finding_1990}, which combines the objective for the k-means clustering algorithm \citep{lloyd_least_1982} with the constraint that the cluster centers must be chosen from within the data to be clustered. For large pool sets, common k-medoids algorithms can be computationally infeasible.
An efficient approximate k-medoids solution can be computed using the seeding method of the k-means++ algorithm \citep{arthur_k-means_2007, ostrovsky_effectiveness_2006}, which simply chooses the next batch element randomly via the distribution
\begin{IEEEeqnarray*}{+rCl+x*}
\forall \bfx \in \Xrem: P(\NextSample(k, \Xsel, \Xrem) = \bfx) = \frac{\min_{\tilde\bfx \in \Xsel} d_k(\bfx, \tilde\bfx)^2}{\sum_{\bfx' \in \Xrem} \min_{\tilde\bfx \in \Xsel} d_k(\bfx', \tilde\bfx)^2}~,
\end{IEEEeqnarray*}
and if $\Xsel$ is empty, it selects $\NextSample(k, \Xsel, \Xrem)$ uniformly at random from $\Xrem$.
We refer to the corresponding selection method as \KMeansPP{}. For the case $\Xmode = \emptyset$, \cite{arthur_k-means_2007} showed that with respect to the objective in \eqref{eq:sum_squared_dists}, the batch selected by \KMeansPP{} is suboptimal by a factor of at most $16 + 8\log(\Nbatch)$ in expectation. The use of \KMeansPP-P{} for BMDAL has been proposed in the so-called BADGE method by \cite{ash_deep_2019}. In contrast to our setting, BADGE is designed for classification and introduces an uncertainty estimate into $k_{\mathrm{ll}}$ not through a posterior transformation but through the influence of the softmax output layer on the magnitude of the gradients.

\begin{figure}[tb]
\centering
\begin{subfigure}[b]{0.48\textwidth}
\centering
\includegraphics[width=\textwidth]{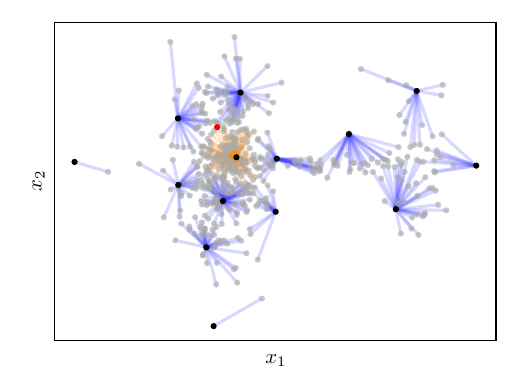}
\caption{One \LCMD{} step.} \label{fig:lcmd_step_1}
\end{subfigure}
\begin{subfigure}[b]{0.48\textwidth}
\centering
\includegraphics[width=\textwidth]{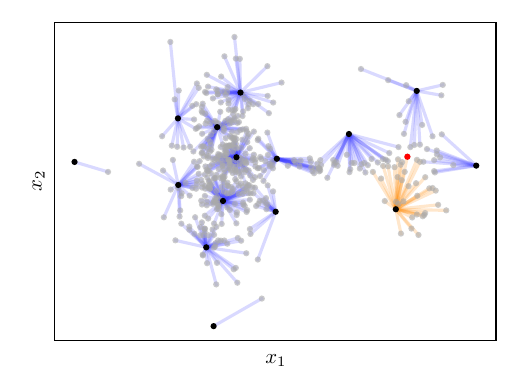}
\caption{Subsequent \LCMD{} step.} \label{fig:lcmd_step_2}
\end{subfigure}
\caption{Two steps of \LCMD{} selection for points $\bfx \in \bbR^2$ with linear feature map $\phi(\bfx) = \bfx$. The black points represent the already selected points $\Xsel$ and the gray points represent the remaining points $\Xrem$. The lines associate each remaining point to the closest selected point, forming clusters. The orange lines represent the cluster with the largest sum of squared distances. The red point, which is the remaining point with the largest distance to the cluster center within the orange cluster, is selected next. In the left plot, it can be seen that the smaller-radius cluster on the left is preferred over the larger-radius clusters on the right due to the higher point density on the left. After adding a point from the cluster, its size shrinks and a cluster on the right becomes dominant, which is shown in the right plot.} \label{fig:lcmd_steps}
\end{figure}

\subsubsection{Largest cluster maximum distance} As an alternative to the randomized \KMeansPP{} method, we propose a novel deterministic method that is inspired by the same objective (\eqref{eq:sum_squared_dists}). This method, which we call \LCMD{} (largest cluster maximum distance), is visualized in \Cref{fig:lcmd_steps}. Intuitively, we enforce (REP) by limiting the selection to the largest cluster, while we also enforce (DIV) by picking the maximum distance point within this cluster. \LCMD{} can be formally defined as follows: We interpret points $\tbfx \in \Xsel$ as cluster centers. For each point $\bfx \in \Xrem$, we define its associated center $c(\bfx) \in \Xsel$ as
\begin{IEEEeqnarray*}{+rCl+x*}
c(\bfx) & \equalDef & \argmin_{\tilde\bfx \in \Xsel} d_k(\bfx, \tilde\bfx)~.
\end{IEEEeqnarray*}
If there are multiple minimizers, we pick an arbitrary one of them. Then, for each center $\tbfx \in \Xsel$, we define the cluster size
\begin{IEEEeqnarray*}{+rCl+x*}
s(\tbfx) \equalDef \sum_{\bfx \in \Xrem: c(\bfx) = \tbfx} d_k(\bfx, \tbfx)^2~.
\end{IEEEeqnarray*} 
We then pick the maximum-distance point from the cluster with maximum size:
\begin{IEEEeqnarray*}{+rCl+x*}
\NextSample(k, \Xsel, \Xrem) = \argmax_{\bfx \in \Xrem: s(c(\bfx)) = \max_{\tbfx \in \Xsel} s(\tbfx)} d_k(\bfx, c(\bfx))~.
\end{IEEEeqnarray*}
As for \MaxDist{}, if $\Xsel$ is empty, we choose $\argmax_{\bfx \in \Xrem} k(\bfx, \bfx)$ instead. If the selection of pool points should be adapted to the distribution of another set $\calX$ instead of $\Xpool$, one may simply compute the cluster sizes based on $\calX$ instead. Importantly, like \KMeansPP{} but unlike some other k-medoids methods, \LCMD{} can be implemented with a runtime complexity that is linear in $\Npool$, as discussed in \Cref{sec:appendix:lcmd}.

\subsubsection{Other options} As mentioned above, \MaxDist{} can be interpreted as a greedy optimization algorithm for a covering objective that is NP-hard to (approximately) optimize \citep{gonzalez_clustering_1985, feder_optimal_1988}, but for which other approximate optimization algorithms have been proposed \citep{sener_active_2018}. Similarly, \MaxDet{}-P attempts to greedily maximize $\det(k(\Xbatch, \Xbatch) + \sigma^2 \bfI)$, which is NP-hard to (approximately) optimize \citep{civril_exponential_2013}, but for which other approximate optimization algorithms have been proposed \citep{biyik_batch_2019}. We do not investigate these advanced optimization algorithms here as they come with greatly increased runtime cost, and \MaxDet{} and \MaxDist{} already enjoy some approximation guarantees, as discussed in \Cref{sec:appendix:selection}.  

\cite{yu_passive_2010}, \cite{wu_pool-based_2018}, and \cite{zhdanov_diverse_2019} suggest less scalable clustering-based approaches for BM(D)AL. Alternatively, it might be interesting to try the greedy $k$-means++ algorithm \citep{celebi_comparative_2013}, which provides a slightly less efficient alternative to the $k$-means++ algorithm.

\cite{caselton_optimal_1984} propose to optimize mutual information between the batch samples and the remaining pool samples, which is analyzed for GPs by \cite{krause_near-optimal_2008}, but does not scale well with the pool set size for GPs, at least in a general kernel-space formulation. Another option is to remove the non-negativity constraint on the weights $w_{\bfx}$ used in \FrankWolfe{}. This setting is also treated in a generalized fashion in \cite{santin_sampling_2021}. An investigation of this method is left to future work.

\subsection{Discussion}

When considering the design criteria from \Cref{sec:batch_active_learning}, we can say that \MaxDiag{} only satisfies (INF), while \MaxDet{} and \MaxDist{} also satisfy (DIV). Arguably, \Bait{}, \KMeansPP{}, and \LCMD{} satisfy all three properties (INF), (DIV) and (REP). The \FrankWolfe{} method is only designed to satisfy (REP), based on which one could argue that (INF) and (DIV) are also satisfied to some extent.

In terms of runtime complexity, as can be seen in \Cref{table:selection_methods}, all considered selection methods are well-behaved for moderate feature space dimension $\dfeat$. When considering kernels such as $k_{\mathrm{grad}}$, whose evaluation is tractable despite having very high feature space dimension, distance-based selection methods are still efficient while \Bait{} and \FrankWolfe{} can become intractable for large pool set sizes, and \MaxDet{} exhibits worse scaling with respect to $\Nbatch$. Moreover, if $\phi(\Xtrain)$ has full rank and $\dfeat \leq \Ntrain$, it follows from \eqref{eq:posterior_fm_other} that in the limit $\sigma^2 \to 0$, the GP posterior uncertainty becomes zero everywhere. Hence, Bayesian posterior-based methods like \MaxDet{}, \Bait{} and \FrankWolfe{} might require $\dfeat \gtrsim \Ntrain$ for good performance, which in turn deteriorates their runtime. 

In the non-batch active learning setting, that is, for $\Nbatch = 1$, some selection methods become equivalent: \LCMD{}-P, \MaxDist{}-P, and \MaxDet{}-P become equivalent to \MaxDiag{}-P; moreover, \KMeansPP{}-P becomes equivalent to \Random{}. This suggests that for $\Nbatch = 1$, TP-mode is necessary for \KMeansPP{} and \LCMD{} to ensure (REP).

\section{Experiments} \label{sec:bmdal:experiments}

To evaluate a variety of combinations of kernels, kernel transformations, and selection methods, we introduce a new open-source benchmark for BMDAL for regression. 
Our benchmark uses 15 large tabular regression data sets, with input dimensions ranging between two and 379, that are selected mostly from the UCI and OpenML repositories, cf.\ the detailed description in \Cref{sec:appendix:data_sets}.
The initial pool set size $\Npool$ for these data sets lies between $31335$ and $198720$. As an NN model, we use a three-layer fully connected NN with 512 neurons in both hidden layers, parameterized as in \Cref{sec:reg_fcnn}. We train the NN using Adam \citep{kingma_adam_2015} for 256 epochs with batch size 256, using early stopping based on a validation set. Results shown here are for the ReLU activation function, but we also re-ran most of our experiments for the SiLU (a.k.a.\ swish) activation \citep{elfwing_sigmoid-weighted_2018}, and unless indicated otherwise, our insights discussed below apply to results for both activation functions. We manually optimized the parameters $\sigma_w, \sigma_b$ and the learning rate separately for both activation functions to optimize the average logarithmic RMSE of \Random{} selection. Details on the NN architecture and training are described in \Cref{sec:appendix:nn_config}. For the hyperparameter $\sigma^2$, which occurs in \MaxDet{}, \Bait{}, and various posterior-based transformations, we found that smaller values typically yield better average results but may cause numerical instabilities. As a compromise, we chose $\sigma^2 = 10^{-6}$ in our experiments and use 64-bit floats for computations involving $\sigma^2$.

In our evaluation, we start with $\ntrain = 256$ and then acquire 16 batches with $\Nbatch = 256$ samples each using the respective BMAL method. We repeat this 20 times with different seeds for NN initialization and different splits of the data into training, validation, pool, and test sets. We measure the mean absolute error (MAE), root mean squared error (RMSE), 95\% and 99\% quantiles, and the maximum error (MAXE) on the test set after each BMAL step. For each of those five error metrics, we average the logarithms of the metric over the 20 repetitions, and, depending on the experiment, over the 16 steps and/or the 15 data sets. Note that a difference of $\delta$ between two logarithmic values corresponds to a ratio $e^\delta \approx 1+\delta$ between the values; for example, a reduction by $\delta = 0.1$ corresponds to a reduction of the geometric mean error by about 10\%. Our most important metric is the RMSE, but we will also put some focus on MAXE since it can be interpreted as a measure of robustness to distribution shifts. Generally, RMSE is more affected by rare but large errors than MAE, while the quantiles and MAXE exclusively focus on rare but large errors.

In the following, we will discuss some of the benchmark results. More detailed results can be found in \Cref{sec:appendix:results}.

\subsection{Comparison to Existing Methods} Based on our detailed evaluation in \Cref{table:all_algs} and \Cref{table:all_algs_silu}, we propose a new BMDAL algorithm as the combination of the \LCMD{}-TP selection method and the kernel $k_{\mathrm{grad}\to\rp{512}}$. \Cref{fig:existing_algs} and \Cref{table:existing_algs_results} show that our proposed combination clearly outperforms other methods from the literature in terms of averaged logarithmic RMSE over our benchmark data sets and random splits. We incorporate the methods from the literature into our framework as shown in \Cref{table:existing_algs_results}, which involves the following modifications:
\begin{itemize}
\item The BALD \citep{houlsby_bayesian_2011} and BatchBALD \citep{kirsch_batchbald_2019} acquisition functions are applied to a last-layer Gaussian Process model.
\item For BAIT, we rescale $k_{\mathrm{ll}}$ based on the training set before applying the posterior transformation, see \Cref{sec:kernel_transformations:scaling}, and we apply regularization by using a small $\sigma^2 > 0$.
\item For ACS-FW \citep{pinsler_bayesian_2019}, we use the \FrankWolfe{} selection method with $k_{\mathrm{ll} \to \acsrfhyper{512}}$. Compared to the experiments by \cite{pinsler_bayesian_2019}, there are several differences: First, we do not permit \FrankWolfe{} to select smaller batches by selecting the same point multiple times. Second, we use 512 random features instead of 10. Third, our acs-rf-hyper transformation first rescales $k_{\mathrm{ll}}$ based on the training set, which we found to improve performance. Fourth, our $k_{\mathrm{ll}}$ kernel incorporates the last-layer bias and not only the weights.
\item By Core-Set, we refer to the k-center-greedy method of \cite{sener_active_2018} applied to $k_{\mathrm{ll}}$, which is also equivalent to FF-Active \citep{geifman_deep_2017}.
\item For BADGE \citep{ash_deep_2019}, which originally incorporates uncertainties into $\phi_{\mathrm{ll}}$ through softmax gradients, we use $\phi_{\mathrm{ll} \to \Xtrain}$ instead of $\phi_{\mathrm{ll}}$.
\end{itemize}
As argued in \Cref{sec:batch_active_learning}, we do not compare to methods that require training with ensembles \citep{krogh_neural_1994}, since ensembles are more expensive to train and these methods are typically designed for non-batch active learning. Moreover, we do not compare to methods that require training with Dropout \citep{tsymbalov_dropout-based_2018} or custom loss functions \citep{ranganathan_deep_2020}, since these methods can change the error for the underlying NN, which makes them difficult to compare fairly and more inconvenient to use.

\begin{table}[tb]
\centering
\renewcommand{\arraystretch}{1.3}
\small
\begin{tabular}{C{0.32\textwidth}cccc}
BMDAL method & Selection method & Kernel & \multicolumn{2}{c}{mean log RMSE ($\downarrow$)} \\
& & & ReLU & SiLU \\
\hline
Supervised learning & \Random{} & --- & -1.401 & -1.406 \\
BALD \citep{houlsby_bayesian_2011} with last-layer GP & \MaxDiag{} & $k_{\mathrm{ll}\to\Xtrain}$ & -1.285 & -1.300 \\
BatchBALD \citep{kirsch_batchbald_2019} with last-layer GP & \MaxDet{}-P & $k_{\mathrm{ll}\to\Xtrain}$ & -1.463 & -1.467 \\
\Bait{} \citep{ash_gone_2021} & \Bait{}-FB-P & $k_{\mathrm{ll}\to\Xtrain}$ & -1.541 & -1.522 \\
ACS-FW \citep{pinsler_bayesian_2019} & \FrankWolfe{}-P & $k_{\mathrm{ll}\to\acsrfhyper{512}}$ & -1.439 & -1.437 \\
Core-Set${}^*$ \citep{sener_active_2018}, FF-Active \citep{geifman_deep_2017} & \MaxDist{}-TP & $k_{\mathrm{ll}}$ & -1.491 & -1.515 \\
BADGE \citep{ash_deep_2019} with last-layer GP uncertainty & \KMeansPP{}-P & $k_{\mathrm{ll}\to\Xtrain}$ & -1.530 & -1.484 \\
Ours & \LCMD{}-TP & $k_{\mathrm{grad}\to\rp{512}}$ & \textbf{-1.590} & \textbf{-1.597} \\
\hline
\multicolumn{5}{l}{\footnotesize ${}^*$ This refers to their simpler k-center-greedy selection method.}
\end{tabular}
\caption{%
Comparison of our BMDAL method against other methods from the literature (cf.\ \Cref{table:existing_algs}). The mean log RMSE is averaged over all data sets, repetitions, and BMAL steps for the respective experiments with ReLU or SiLU activation function. We make small adjustments to the literature methods as described in \Cref{sec:bmdal:experiments}.
} \label{table:existing_algs_results}
\end{table}

\subsection{Evaluated Combinations} Our framework allows us to obtain a vast number of BMDAL algorithms via combinations of base kernels, kernel transformations, and the P and TP modes of different selection methods. \Cref{table:all_algs} and \Cref{table:all_algs_silu} in \Cref{sec:appendix:results} show a large number of such combinations for ReLU and SiLU activations, respectively. These combinations have been selected according to the following principles:
\begin{itemize}
\item Kernels for P-mode selection use posterior-based transformations, while kernels for TP-mode selection do not (see \Cref{sec:iterative_sel_methods}).
\item Sketching and random features always use 512 target features. Similar to the hidden layer size of 512, this number has been selected to be a bit larger than the usually employed $\Nbatch = 256$. Note that due to the bias in the last layer, $k_{\mathrm{ll}}$ has a 513-dimensional feature space.
\item \FrankWolfe{} is only run in P mode (as proposed in \cite{pinsler_bayesian_2019}), since in TP mode, kernel mean embeddings in 512-dimensional feature space would be approximated using more than 512 samples. Note that \cite{pinsler_bayesian_2019} only use 10 instead of 512 random features in their experiments, leading to a worse approximation.
\item Due to the equivalence between P mode with posterior transformation and TP mode without posterior transformation mentioned in \Cref{sec:specific_sel_methods:maxdet}, \textsc{Bait} is always run in P mode, and \MaxDet{} is mostly run in P mode except for $k_{\mathrm{grad}}$ and $k_{\mathrm{nngp}}$ due to their high-dimensional feature space.
\end{itemize}

In general, we observe the following trends in our results across selection methods:
\begin{itemize}
\item The network-dependent base kernels $k_{\mathrm{ll}}$ and $k_{\mathrm{grad}}$ clearly outperform the network-independent base kernels $k_{\mathrm{lin}}$ and $k_{\mathrm{nngp}}$ across different selection methods, modes and kernel transformations.
\item Out of the network-dependent base kernels, $k_{\mathrm{grad}}$ typically outperforms $k_{\mathrm{ll}}$, at least for NN hyperparameters optimized for \Random{} (cf.\ \Cref{sec:appendix:nn_config}). It should be noted that in our ReLU experiments, these optimized hyperparameters typically lead to many dead neurons in the last hidden layer, which may affect $k_{\mathrm{ll}}$ by reducing the effective feature space dimension.\footnote{In the extreme case where all neurons in the last hidden layer are dead, the network-dependent base kernels become degenerate, which can cause numerical problems in selection methods. Once a selection method suggests an invalid (e.g.\ already selected) sample for the batch, we fill up the rest of the batch with random samples. Out of the 597900 BMDAL steps in our ReLU experiments, such invalid samples were suggested in just 4 steps in total.}
We define the effective feature space dimension of the pool set for a kernel $k$ as
\begin{IEEEeqnarray*}{+rCl+x*}
d_{\mathrm{eff}} \equalDef \frac{\tr(k(\Xpool, \Xpool))}{\|k(\Xpool, \Xpool)\|_2} = \frac{\lambda_1 + \hdots + \lambda_{\dfeat}}{\lambda_1}~,
\end{IEEEeqnarray*}
where $\lambda_1 \geq \hdots \geq \lambda_{\dfeat}$ are the eigenvalues of the feature covariance matrix 
\begin{IEEEeqnarray*}{+rCl+x*}
\phi(\Xpool)^\top \phi(\Xpool) \in \bbR^{\dfeat \times \dfeat}~.
\end{IEEEeqnarray*}
With this definition, $d_{\mathrm{eff}}$ is indeed typically much larger for $k_{\mathrm{grad}\to\rp{512}}$ than for $k_{\mathrm{ll}}$ in our experiments.\footnote{Specifically, averaged over all corresponding ReLU experiments with $\Nbatch=256$ and over all BMAL steps, $k_{\mathrm{grad}\to\rp{512}}$ leads to an average $d_{\mathrm{eff}}$ of about $5.5$, while $k_{\mathrm{ll}}$ leads to an average $d_{\mathrm{eff}}$ of about $1.7$. On corresponding BMAL steps, the effective dimension is larger for $k_{\mathrm{grad}\to\rp{512}}$ about 95\% of the time. For SiLU, the results are slightly less extreme, with effective dimensions of $4$ and $2.3$, and the effective dimension of $k_{\mathrm{grad}\to\rp{512}}$ being larger about 90\% of the time.}
Another difference is that, for the ReLU activation function, $k_{\mathrm{grad}}$ is discontinuous while $k_{\mathrm{ll}}$ is not. 
\item For $k_{\mathrm{grad}}$, applying sketching does not strongly affect the resulting accuracy while leading to considerably faster runtimes.
\item When evaluating the use of ensembled NN kernels, we want to differentiate between the effect of ensembling on the accuracy of supervised learning and the effect of ensembling on the quality of the selected batches $\Xbatch$. To eliminate the former effect, we only consider the averaged errors of the individual ensemble members and not the error of their averaged predictions. With this method of evaluation, we find that ensembling of network-dependent kernels only leads to small improvements in the error, at least for the ensembling configurations we tested.
This is in contrast to other papers where the uncertainty of the ensemble predictions turned out to be more beneficial \citep{beluch_power_2018, pop_deep_2018}. Perhaps ensembling is less useful in our case because our non-ensembled kernels already provide good uncertainty measures.
\item Out of the acs-grad, acs-rf and acs-rf-hyper transformations, acs-rf often performs best, except for \FrankWolfe{}-P with base kernel $k_{\mathrm{grad}}$, where acs-rf-hyper performs best.
\item In contrast to \cite{ash_gone_2021}, we find that \Bait{}-FB does not perform better than \Bait{}-F.
\item The relative gains for BMDAL methods compared to \Random{} selection are typically largest on metrics such as MAXE or 99\% quantile, and worst on MAE.
\item All investigated BMDAL methods only take a few seconds to select a batch in our experiments on our NVIDIA RTX 3090 GPUs (cf.\ \Cref{sec:appendix:experiments}), which is typically faster than the time for training the corresponding NN. Hence, we expect all investigated BMDAL methods to be much faster than the time for labeling in most scenarios where BMDAL is desirable. Note that the runtime of TP-mode selection methods is comparable to those of P-mode selection methods only because we typically run P-mode selection with 64-bit floats to avoid numerical issues for posteriors.
TP-mode selection methods need to consider $\Ntrain + \Nbatch$ instead of $\Nbatch$ selected points, which can be significantly slower than P-mode if $\Ntrain \gg \Nbatch$. Especially for TP-mode selection methods, it may therefore be desirable to let $\Nbatch$ grow proportionally to $\Ntrain$.
\end{itemize}

\begin{table}
\centering
\footnotesize
\begin{tabular}{ccccc}
Selection method & Sel.\ mode & Selected kernel & mean log RMSE & Avg.\ time [s] \\
\hline
\Random{} & --- & --- & -1.401 & 0.001 \\
\MaxDiag{} & --- & $k_{\mathrm{grad}\to\rp{512}\to\acsrf{512}}$ & -1.370 & 0.650 \\
\MaxDet{} & P & $k_{\mathrm{grad}\to\rp{512}\to\Xtrain}$ & -1.512 & 0.770 \\
\Bait{} & F-P & $k_{\mathrm{grad}\to\rp{512}\to\Xtrain}$ & -1.585 & 1.508 \\
\FrankWolfe{} & P & $k_{\mathrm{grad}\to\rp{512}\to\acsrfhyper{512}}$ & -1.542 & 0.823 \\
\MaxDist{} & P & $k_{\mathrm{grad}\to\rp{512}\to\Xtrain}$ & -1.514 & 0.713 \\
\KMeansPP{} & P & $k_{\mathrm{grad}\to\rp{512}\to\acsrf{512}}$ & -1.569 & 0.836 \\
\LCMD{} & TP & $k_{\mathrm{grad}\to\rp{512}}$ & \textbf{-1.590} & 0.981 \\
\end{tabular}
\caption{Selected kernels and modes per selection method that are shown in our plots. The mean log RMSE is averaged over all data sets, repetitions, and BMAL steps. The average time for the batch selection is averaged over all data sets and BMAL steps, measured at one repetition with only one process running on each NVIDIA RTX 3090 GPU. An overview of all results can be found in \Cref{table:all_algs}.} \label{table:selected_kernels}
\end{table}

\subsection{Best Kernels and Modes for each Selection Method} Our base kernels, kernel transformations, and selection modes yield numerous ways to apply each selection method. To compare selection methods, we choose for each selection method the best-performing combination according to the averaged logarithmic RMSE, excluding kernels with ensembling and $k_{\mathrm{grad}}$ without sketching since they considerably increase computational cost while providing comparable accuracy to their more efficient counterparts. The selected combinations are shown in \Cref{table:selected_kernels}. Compared to the literature methods in \Cref{table:existing_algs_results}, we see that optimizing the kernel and mode can yield a considerable difference in performance. 
Note that the relative performance between the configurations for SiLU is slightly different, as can be seen in \Cref{table:all_algs_silu}. If we selected combinations according to the SiLU results, the combinations for \MaxDist{} and \KMeansPP{} in \Cref{table:selected_kernels} would use TP-mode and $k_{\mathrm{grad}\to\rp{512}}$ instead. When considering only kernels based on $k_{\mathrm{ll}}$, the results from \Cref{table:all_algs} and \Cref{table:all_algs_silu} show that a comparison of selection methods would look qualitatively similar.

\begin{figure}[p]
\centering
\includegraphics[height=0.78\textheight]{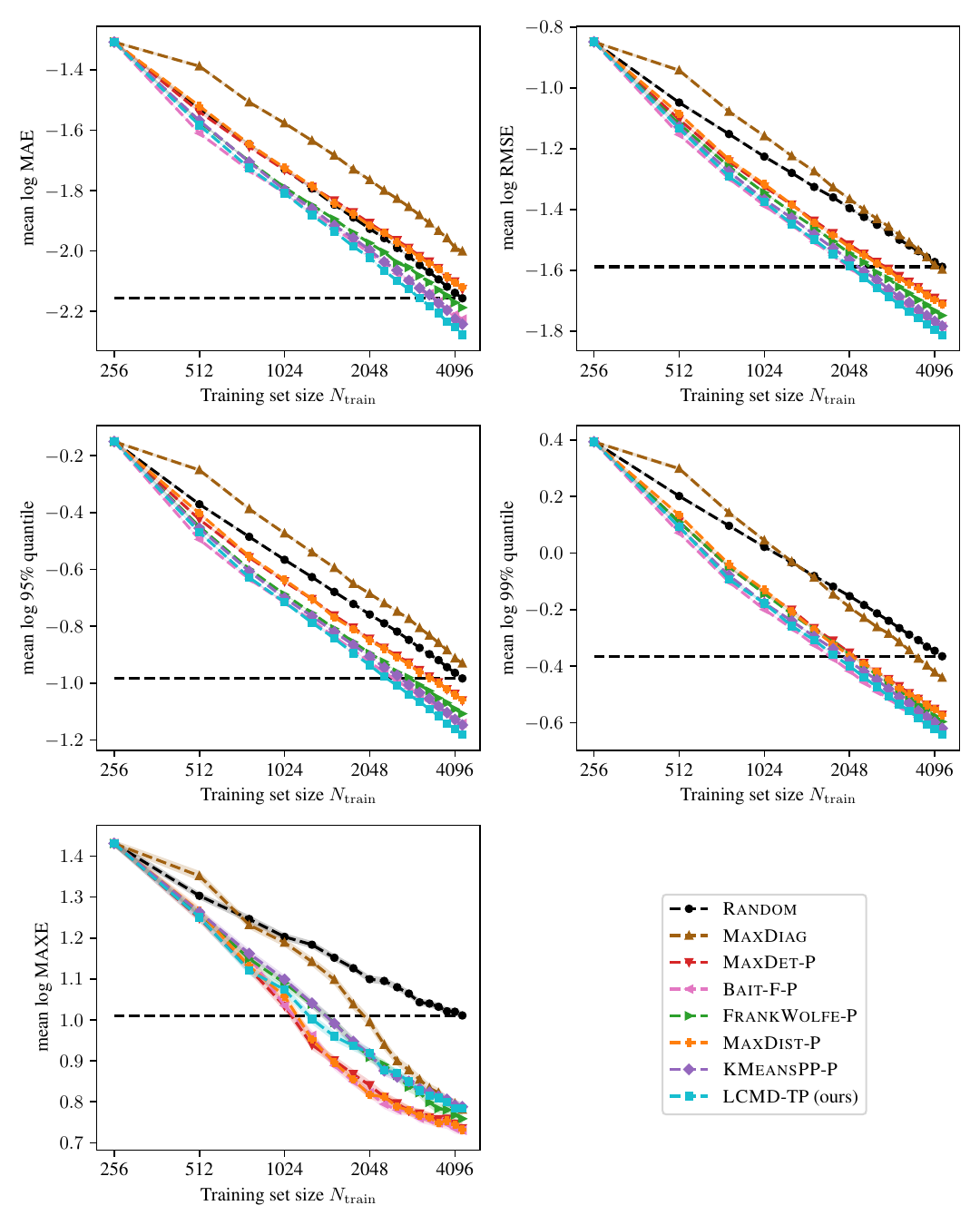}
\caption{%
This figure shows how fast the errors decrease during BMAL for different selection methods and their corresponding kernels from \Cref{table:selected_kernels}.
Specifically, for each of the five error metrics, the corresponding plot shows the logarithmic error metric between each BMAL step for $\Nbatch = 256$, averaged over all repetitions and data sets.
The performance of \Random{} can be interpreted as the performance of supervised learning without active learning. The black horizontal dashed line corresponds to the final performance of \Random{} at $\Ntrain = 4352$. The shaded area, which is nearly invisible for all metrics except MAXE, corresponds to one estimated standard deviation of the mean estimator, cf.\ \Cref{sec:appendix:results}.} \label{fig:learning_curves}
\end{figure}

\begin{figure}[p]
\centering
\includegraphics[height=0.83\textheight]{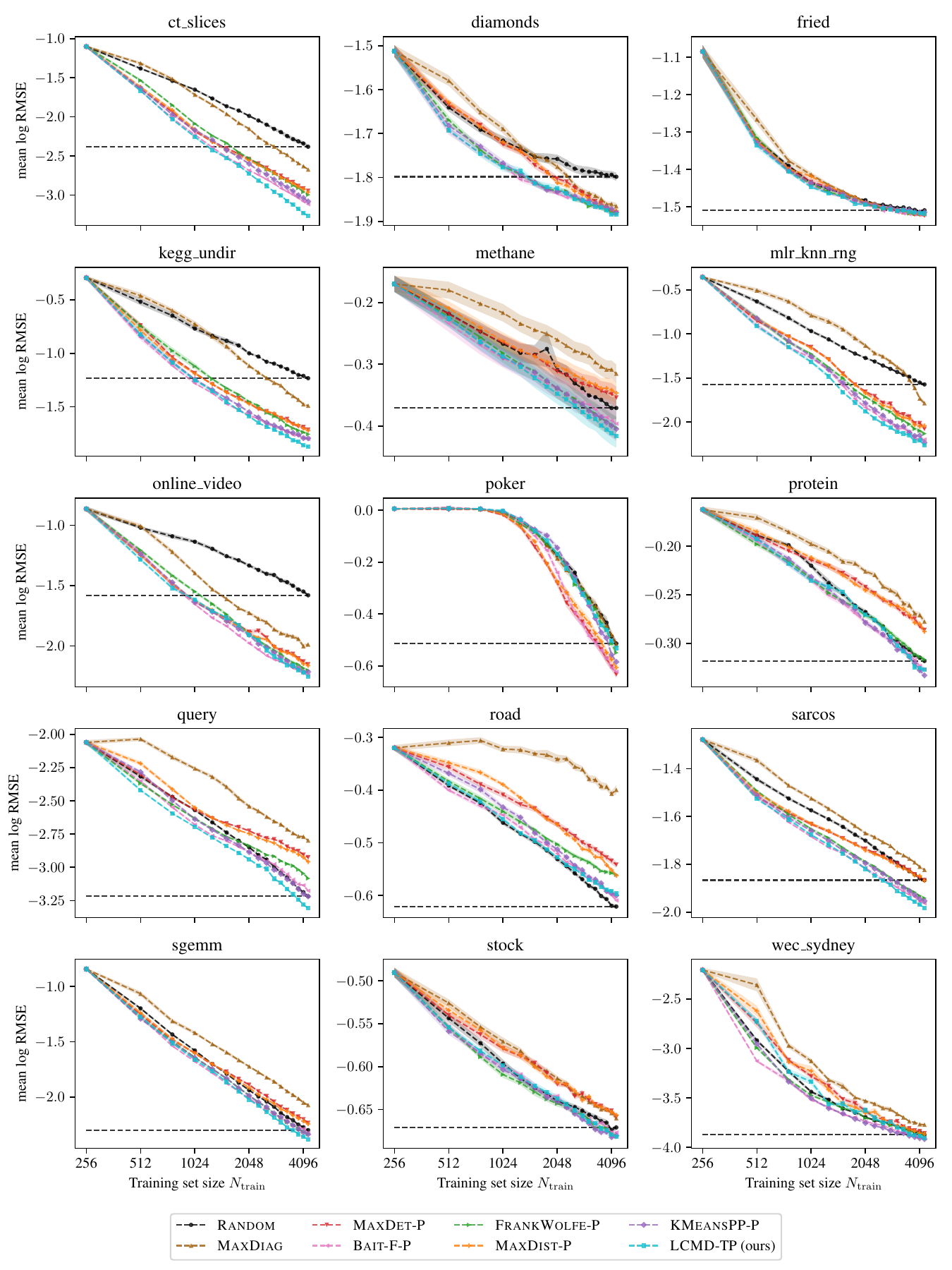}
\captionsetup{font=small}
\caption{%
This figure shows how fast the RMSE decreases during BMAL on the individual benchmark data sets for different selection methods and their corresponding kernels from \Cref{table:selected_kernels}.
Specifically, the plots above show the logarithmic RMSE between each BMAL step for $\Nbatch = 256$, averaged over all repetitions.
The black horizontal dashed line corresponds to the final performance of \Random{} at $\Ntrain = 4352$. The shaded area corresponds to one estimated standard deviation of the mean estimator, cf.\ \Cref{sec:appendix:results}.} \label{fig:learning_curves_individual_rmse}
\end{figure}

\begin{figure}[p]
\centering
\includegraphics[height=0.75\textheight]{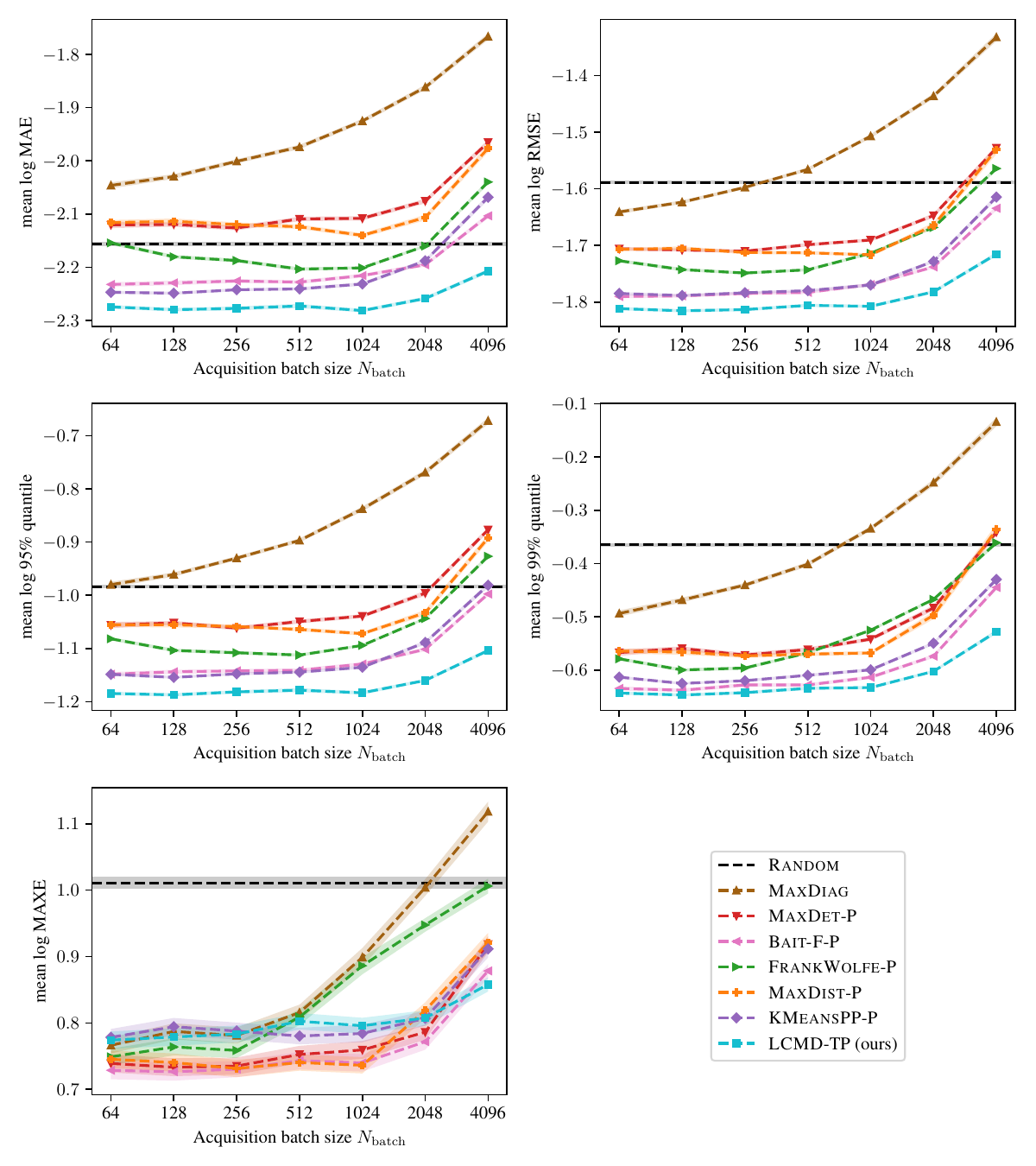}
\captionsetup{font=small}
\caption{%
This figure shows how much the final accuracy of different BMDAL methods deteriorates when fewer BMAL steps with larger batch sizes are used. Specifically, we use different selection methods with the corresponding kernels from \Cref{table:selected_kernels}, starting with $\Ntrain = 256$ and then performing $2^m$ BMAL steps with batch size $\Nbatch = 2^{12-m}$ for $m \in \{0, \hdots, 6\}$, such that the final training set size is $4352$ in each case. 
For each of the five error metrics, the corresponding plot shows the final logarithmic error metric, averaged over all data sets and repetitions. Note that the performance of \Random{} selection does not depend on $\Nbatch$ but only on the final training set size, hence it is shown as a constant line here. The shaded area, which is nearly invisible for all metrics except MAXE, corresponds to one estimated standard deviation of the mean estimator, cf.\ \Cref{sec:appendix:results}.
} \label{fig:batch_sizes}
\end{figure}

\subsection{Comparison of Selection Methods}

We compare the selected configurations from \Cref{table:selected_kernels} in several aspects. \Cref{fig:learning_curves} shows the evolution of the mean logarithmic MAE, RMSE, MAXE, 95\% quantile, and 99\% quantile over the BMAL steps, for a batch size of $\Nbatch = 256$. This demonstrates that the best considered BMDAL methods can match the average performance of \Random{} selection with about half of the samples for RMSE, and even fewer samples for MAXE. On individual data sets, this may differ, as shown in \Cref{fig:learning_curves_individual_rmse}. From this figure, it is apparent that when considering RMSE, \LCMD{}-TP outperforms other methods not only in terms of average performance but across the majority of the data sets. Specifically, \Cref{table:table_data_sets} shows that \LCMD{}-TP matches or exceeds the performance of the other selected BMDAL methods on 8 out of the 15 data sets in terms of RMSE. \Cref{fig:learning_curves_individual_rmse} also shows that the selected \MaxDet{}-P and \MaxDist{}-P configurations yield very similar performance on all data sets.

\Cref{fig:batch_sizes} shows the influence of the chosen batch size $\Nbatch$ on the final performance at $\ntrain = 4352$ training samples. As expected, the naive active learning scheme \MaxDiag{} is particularly sensitive to the batch size. The other selection methods are less sensitive to changes in $\Nbatch$ and exhibit almost no degradation in performance up to $\Nbatch = 1024$. Note that this \quot{threshold} might depend on the initial and final training set sizes as well as the feature-space dimension $\dfeat$. %

Overall, the discussed figures and the detailed results in \Cref{table:all_algs} show that \LCMD{}-TP performs best in terms of MAE, RMSE, and 95\% quantile, followed by \Bait{}-F-P and \KMeansPP{}. 
For MAXE, \MaxDist{}, \MaxDet{} and \Bait{}-F-P exhibit the best performances. Since \MaxDet{} and \MaxDist{} are motivated by worst-case considerations, it is not surprising that they perform well on MAXE, while the strong performance of \Bait{}-F-P on MAXE is unexpected. Moreover, it is perhaps surprising that in \Cref{fig:learning_curves}, the relative performances for the 99\% quantile are arguably more similar to the RMSE performances than to the MAXE performances. For the 95\% quantile, the relative performances for \MaxDet{}, \MaxDist{} and \MaxDiag{} are even worse than for the RMSE. Thus, the use of \MaxDist{} instead of \LCMD{} may only be advisable if one expects strong distribution shifts between pool and test sets.

\subsection{When should BMDAL be Applied?} 

\begin{figure}[tb]
\centering
\includegraphics[height=0.4\textheight]{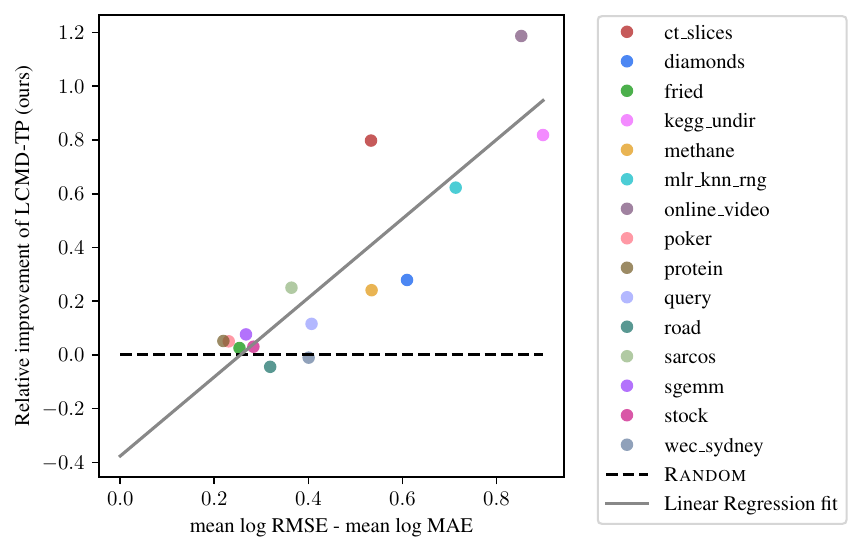}
\caption{This figure shows, for each data set, the improvement in sample efficiency of \LCMD{}-TP over \Random{} (on the $y$-axis) versus the variation of the error distribution (on the $x$-axis). The variation is measured as mean log RMSE $-$ mean log MAE on the initial training set ($\Ntrain = 256$). The improvement in sample efficiency is measured by $\frac{\text{mean log RMSE(\LCMD{}-TP)} - \text{mean log RMSE(\Random{})}}{\text{mean log RMSE(\Random{})} - (\text{mean log RMSE at }\Ntrain=256)}$. Here, the means are taken over all 20 repetitions and, unless indicated otherwise, over the trained networks after each of the 16 BMDAL steps. 
The Pearson correlation coefficient for the plotted data is $R \approx 0.88$.
} \label{fig:skewness_plot}
\end{figure}

While \LCMD{}-TP achieves excellent average performance, its benefits over \Random{} selection vary strongly between data sets, as is evident from \Cref{fig:learning_curves_individual_rmse}. Overall, \LCMD{}-TP with $k_{\mathrm{grad}\to\rp{512}}$ outperforms random selection in terms of RMSE on 13 out of the 15 data sets and barely performs worse on the other two. Nonetheless, an \textit{a priori} estimate of the benefits of \LCMD{}-TP over \Random{} selection on individual data sets could be useful to inform practitioners on whether they should be interested in applying BMDAL. \Cref{fig:skewness_plot} shows that on our 15 benchmark data sets, the variation of test errors after training on the initial training set ($\Ntrain = 256$), measured by the quotient $\frac{\mathrm{RMSE}}{\mathrm{MAE}}$, is strongly correlated with the improvement in sample efficiency through \LCMD{}-TP over \Random{}. In other words, the larger $\frac{\mathrm{RMSE}}{\mathrm{MAE}}$ on the initial training set, the more benefit we can expect from \LCMD{}-TP with $k_{\mathrm{grad}\to\rp{512}}$ over random selection.

\section{Conclusion} \label{sec:conclusion}

In this paper, we introduced a framework to compose BMDAL algorithms out of base kernels, kernel transformations, and selection methods. We then evaluated different combinations of these components on a new benchmark consisting of 15 large tabular regression data sets. In our benchmark results, for all considered selection methods, replacing the wide-spread last-layer kernel $k_{\mathrm{ll}}$ by a sketched finite-width neural tangent kernel $k_{\mathrm{grad}\to\rp{p}}$ leads to accuracy improvements at similar runtime and memory cost. Moreover, our novel \LCMD{} selection method sets new state-of-the-art results in our benchmark in terms of RMSE and MAE. 

\subsection{Limitations} %
The BMDAL methods in our framework are very attractive for practitioners using NNs for regression since they are scalable to large data sets and can be applied to a wide variety of NN architectures and training methods without requiring modifications to the NN. However, while our benchmark contains many large data sets, it cannot cover all possible application scenarios that the considered BMDAL methods could be applied to. For example, it is unclear whether our insights can be transferred to applications like drug discovery \citep{mehrjou_genedisco_2021} or atomistic ML \citep{zaverkin_exploration_2021}, where other types of data and other NNs are employed. Even in the tabular data setting, the relevance of our results for smaller data sets or recently proposed NN architectures \citep[e.g.][]{gorishniy_revisiting_2021, somepalli_saint_2022, kadra_well-tuned_2021} is unclear. Moreover, the current benchmark does not involve distribution shifts between pool and test data, which would be interesting for some practical applications.

\subsection{Remaining Questions} Our results give rise to some interesting questions for future research, of which we list some in the following: Can $k_{\mathrm{grad}}$ be adapted to incorporate effects of optimizers such as Adam? How can it be efficiently evaluated and sketched for other types of layers? 
How can we decide which method to use on a data set, beyond just using the one with the best average performance across the benchmark? Are there some characteristics of data sets or training setups that can be used to predict which method will perform best?
Are clustering-based methods like \LCMD{}-TP also superior to other methods for non-batch AL?
How much better performance could be attained if the methods had access to the pool labels, and what kinds of batches would such a method select?
Would it have similar properties as \cite{zhou_towards_2021} found for classification?
How can our framework be generalized to classification or multi-output regression?

Our framework is formulated for the pool-based AL setting, where samples should be selected from a pool of unlabeled samples. By using a different kind of selection methods with the same kernels and kernel transformations, our framework could be adapted to the streaming AL setting, where unlabeled samples arrive sequentially and one has to decide immediately whether to label them or not. For the membership-query AL setting, where unlabeled samples can be chosen arbitrarily, the situation is more difficult: Since the feature map $\phi$ is typically not surjective, samples cannot be chosen in the feature space directly, and a direct choice in the input space might require differentiating the kernel for efficient optimization. Nonetheless, extending our methods to other AL settings could be an interesting avenue for further research.

\acks{We want to thank Philipp Hennig, Tizian Wenzel, Daniel Winkle, Benjamin Unger, and Paul Bürkner for helpful comments.
Funded by Deutsche Forschungsgemeinschaft (DFG, German Research Foundation) under Germany's Excellence Strategy - EXC 2075 – 390740016. The authors thank the International Max Planck Research School for Intelligent Systems (IMPRS-IS) for supporting David Holzmüller. Viktor Zaverkin acknowledges the financial support received in the form of a Ph.D.\ scholarship from the Studienstiftung des Deutschen  Volkes (German National Academic Foundation).}

\begin{appendixenv}

\section{Overview} \label{sec:appendix:overview}

The appendix structure mirrors the structure of the main paper, providing more details on the corresponding sections of the main paper: We discuss further details on base kernels in \Cref{sec:appendix:kernels}, on kernel transformations in \Cref{sec:appendix:kernel_transformations}, and on selection methods in \Cref{sec:appendix:selection}. The latter section includes efficiency-focused pseudocode as well as discussions on relations to the literature. Finally, we provide details on the setup and results of our experiments in \Cref{sec:appendix:experiments}.

\section{Details on Base Kernels} \label{sec:appendix:kernels}

In the following, we will provide details for the infinite-width NNGP kernel.

\subsection{NNGP Kernel} \label{sec:appendix:nngp}

For the fully-connected NN model considered in \Cref{sec:reg_fcnn} with a ReLU activation function, the NNGP kernel is given by
\begin{IEEEeqnarray*}{+rCl+x*}
k_{\mathrm{nngp}}(\bfx, \tilde\bfx) & \equalDef & k_{\mathrm{nngp}}^{(L)}(\bfx, \tilde\bfx)~,
\end{IEEEeqnarray*}
where we roughly follow \cite{lee_wide_2019} and recursively define
\begin{IEEEeqnarray*}{+rCl+x*}
k_{\mathrm{nngp}}^{(1)}(\bfx, \tilde\bfx) & \equalDef & \frac{\sigma_w^2}{d} \langle \bfx, \tilde\bfx\rangle \\
k_{\mathrm{nngp}}^{(l+1)}(\bfx, \tilde\bfx) & \equalDef & \sigma_w^2 f(k_{\mathrm{nngp}}^{(l)}(\bfx, \bfx), k_{\mathrm{nngp}}^{(l)}(\bfx, \tilde \bfx), k_{\mathrm{nngp}}^{(l)}(\tilde \bfx, \tilde \bfx)) \\
f(a, b, c) & \equalDef & \frac{\sqrt{ac}}{2\pi} \left(\sqrt{1-u^2} + u(\pi - \arccos(u))\right) \text{ with } u \equalDef \frac{b}{\sqrt{ac}}~.
\end{IEEEeqnarray*}
Note that we do not include the $\sigma_b$ terms here since we initialize the biases to zero unlike \cite{lee_wide_2019}.

\section{Details on Kernel Transformations} \label{sec:appendix:kernel_transformations}

In the following, we will discuss additional aspects of various kernel transformations.

\subsection{Gaussian Process Posterior Transformation} \label{sec:appendix:posterior}

\cite{khan_approximate_2019} showed that certain posterior approximation methods for NNs turn them into Gaussian processes with the finite-width NTK $k_{\mathrm{grad}}$. In the following, we will present a self-contained derivation of this relationship in our framework, including the last-layer kernel $k_{\mathrm{ll}}$. In our exposition, we roughly follow \cite{daxberger_laplace_2021}. For a Bayesian NN, we impose a prior $p(\bftheta) = \calN(\bftheta \mid \bfzero, \lambda^2 \bfI)$ on the NN parameters $\bftheta$. Using an observation model $y_i = f_{\bftheta}(\bfx_i) + \varepsilon_i, \varepsilon_i \sim \calN(0, \sigma^2)$ i.i.d., the negative log-likelihood of the data is given by
\begin{IEEEeqnarray*}{+rCl+x*}
-\log p(\Ytrain \mid \Xtrain, \bftheta) & = & C + \frac{1}{2\sigma^2}\sum_{i=1}^{\Ntrain} (y_i - f_{\bftheta}(\bfx_i))^2 = C_1 + \frac{1}{2\sigma^2}\Ntrain \calL(\bftheta)
\end{IEEEeqnarray*}
for some constant $C_1 \in \bbR$. The negative log-posterior is hence given by
\begin{IEEEeqnarray*}{+rCl+x*}
\tilde\calL(\bftheta) & \equalDef & -\log p(\bftheta \mid \Ytrain, \Xtrain) \\
& = & \log(Z)-\log p(\Ytrain \mid \Xtrain, \bftheta)-\log p(\bftheta) \\
& = & C_2 + \frac{1}{2\sigma^2}\Ntrain \calL(\bftheta) + \frac{1}{2\lambda^2} \|\bftheta\|_2^2~,
\end{IEEEeqnarray*}
If $\bftheta^*$ minimizes $\tilde\calL$, that is, if $\bftheta^*$ is a maximum \textit{a posteriori} (MAP) estimate, we obtain the second-order Taylor approximation
\begin{IEEEeqnarray*}{+rCl+x*}
\tilde\calL(\bftheta) \approx \tilde\calL(\bftheta^*) + \frac{1}{2} (\bftheta - \bftheta^*)^\top \bfH (\bftheta - \bftheta^*), \quad \bfH \equalDef \nabla_{\bftheta}^2 \tilde\calL(\bftheta^*)~,
\end{IEEEeqnarray*}
which yields a Gaussian approximation to the posterior, the so-called Laplace approximation \citep{laplace_memoire_1774, mackay_bayesian_1992}:
\begin{IEEEeqnarray*}{+rCl+x*}
p(\bftheta \mid \Ytrain, \Xtrain) \approx \calN(\bftheta \mid \bftheta^*, \bfH^{-1})~.
\end{IEEEeqnarray*}
The Hessian is given by
\begin{IEEEeqnarray*}{+rCl+x*}
\bfH & = & \frac{1}{\lambda^2} \bfI + \frac{1}{2\sigma^2} \sum_{(\bfx, y) \in \Dtrain} \nabla_{\bftheta}^2 (y - f_{\bftheta^*}(\bfx))^2 \\
& = & \lambda^{-2} \bfI + \sigma^{-2} \sum_{(\bfx, y) \in \Dtrain} \left((\nabla_{\bftheta} f_{\bftheta^*}(\bfx))(\nabla_{\bftheta} f_{\bftheta^*}(\bfx))^\top + (f_{\bftheta^*}(\bfx) - y)\nabla_{\bftheta}^2 f_{\bftheta^*}(\bfx)\right)~.
\end{IEEEeqnarray*}
By ignoring the terms $(f_{\bftheta^*}(\bfx) - y)\nabla_{\bftheta}^2 f_{\bftheta^*}(\bfx)$, which are expected to be small since the first factor is small, we arrive at the generalized Gauss-Newton (GGN) approximation to the Hessian \citep{schraudolph_fast_2002}:
\begin{IEEEeqnarray*}{+rCl+x*}
\bfH_{\mathrm{GGN}} & = & \lambda^{-2} \bfI + \sigma^{-2} \sum_{(\bfx, y) \in \Dtrain} (\nabla_{\bftheta} f_{\bftheta^*}(\bfx))(\nabla_{\bftheta} f_{\bftheta^*}(\bfx))^\top~.
\end{IEEEeqnarray*}
By pretending that $\bftheta^* = \bftheta_T$, i.e.\ that the parameters at the end of training are the minimizer of $\tilde\calL$, we can relate this to $\phi_{\mathrm{grad}}$:
\begin{IEEEeqnarray*}{+rCl+x*}
\bfH_{\mathrm{GGN}} & = & \lambda^{-2} \bfI + \sigma^{-2} \sum_{(\bfx, y) \in \Dtrain} \phi_{\mathrm{grad}}(\bfx)\phi_{\mathrm{grad}}(\bfx)^\top \\
& = & \lambda^{-2} \bfI + \sigma^{-2} \phi_{\mathrm{grad}}(\Xtrain)^\top \phi_{\mathrm{grad}}(\Xtrain)~.
\end{IEEEeqnarray*}
If we want to compute the predictive distribution of $f_{\bftheta}(\bfx)$ for $\bftheta \sim p(\bftheta \mid \Xtrain)$, we can further use the linearization
\begin{IEEEeqnarray*}{+rCl+x*}
f_{\bftheta}(\bfx) & \approx & f_{\bftheta^*}(\bfx) + \langle \bftheta - \bftheta^*, \nabla_{\bftheta} f_{\bftheta^*}(\bfx)\rangle = f_{\bftheta^*}(\bfx) + \langle \bftheta - \bftheta^*, \phi_{\mathrm{grad}}(\bfx)\rangle~,
\end{IEEEeqnarray*}
which, according to \cite{immer_improving_2021}, improves the results for the predictive distribution. The predictive distribution can then be approximated as
\begin{IEEEeqnarray*}{+rCl+x*}
p(\calY \mid \calX, \Dtrain) & \approx & \calN(\calY \mid f_{\bftheta^*}(\calX), \phi_{\mathrm{grad}}(\calX)\bfH_{\mathrm{GGN}}^{-1} \phi_{\mathrm{grad}}(\calX)^\top)~, 
\end{IEEEeqnarray*}
with the covariance matrix
\begin{IEEEeqnarray*}{+rCl+x*}
&& \phi_{\mathrm{grad}}(\calX)\bfH_{\mathrm{GGN}}^{-1} \phi_{\mathrm{grad}}(\calX)^\top \\
& = & \phi_{\mathrm{grad}}(\calX)(\lambda^{-2} \bfI + \sigma^{-2} \phi_{\mathrm{grad}}(\Xtrain)^\top \phi_{\mathrm{grad}}(\Xtrain))^{-1}\phi_{\mathrm{grad}}(\calX)^\top \\
& = & \sigma^2\lambda\phi_{\mathrm{grad}}(\calX)(\lambda\phi_{\mathrm{grad}}(\Xtrain)^\top \lambda\phi_{\mathrm{grad}}(\Xtrain) + \sigma^2 \bfI)^{-1}\lambda\phi_{\mathrm{grad}}(\calX)^\top \\
& = & (\lambda^2 k_{\mathrm{grad}})_{\to\post{\Xtrain, \sigma^2}}(\calX, \calX)~.
\end{IEEEeqnarray*}
This demonstrates that $(\lambda^2 k_{\mathrm{grad}})_{\to\post{\Xtrain, \sigma^2}}(\calX, \calX)$ yields the posterior predictive covariance on $\calX$ of a Bayesian NN under the following approximations:
\begin{enumerate}[(1)]
\item The parameter posterior is approximated using the Laplace approximation,
\item The Hessian matrix in the Laplace approximation is approximated using the GGN approximation,
\item The predictive posterior is further approximated using a linearization of the NN, and
\item The MAP estimate is approximated by the trained parameters of the NN.
\end{enumerate}

If we perform Bayesian inference only over the last-layer parameters (i.e., $\bftheta = \tilde\bfW^{(L)}$), the derivation above shows instead that the posterior predictive covariance is approximated by $(\lambda^2 k_{\mathrm{ll}})_{\to\post{\Xtrain, \sigma^2}}(\calX, \calX)$. Moreover, for the last-layer parameters, the approximations (1) -- (3) are exact, since the NN is affine linear in the last-layer parameters, and the approximation (4) does not change the resulting kernel. Such last-layer Bayesian models have been employed, for example, by \cite{lazaro-gredilla_marginalized_2010, snoek_scalable_2015, ober_benchmarking_2019, kristiadi_being_2020}, and are also known as neural linear models \citep{ober_benchmarking_2019}.

\cite{eschenhagen_mixtures_2021} experimentally demonstrated that taking a mixture of multiple Laplace approximations around different local minima of the loss function can improve uncertainty predictions for Bayesian NNs. If we consider $\Nens$ local minima $\bftheta^{(i)}$ with corresponding base kernels $k^{(i)}$ and combine the Laplace approximations with uniform weights, we obtain the posterior distributions
\begin{IEEEeqnarray*}{+rCl+x*}
p(\calY \mid \calX, \Dtrain) & \approx & \frac{1}{\Nens} \sum_{i=1}^{\Nens} \calN(\calY \mid f_{\bftheta^{(i)}}(\calX), (\lambda k^{(i)})_{\to\post{\Xtrain, \sigma^2}}(\calX, \calX))~.
\end{IEEEeqnarray*}
By the law of total covariance, we have
\begin{IEEEeqnarray*}{+rCl+x*}
\Cov(y_1, y_2 \mid \bfx_1, \bfx_2, \Dtrain) & \approx & \bbE_{i \sim \calU\{1, \hdots, \Nens\}} [(\lambda^2 k^{(i)})_{\to\post{\Xtrain, \sigma^2}}(\bfx_1, \bfx_2)] \\
&& ~+~ \Cov_{i \sim \calU\{1, \hdots, \Nens\}}(f_{\bftheta^{(i)}}(\bfx_i), f_{\bftheta^{(i)}}(\bfx_j)) \\
& = & (\lambda^2 k)_{\to\post{\Xtrain, \sigma^2}\to\ens{\Nens}}(\bfx_1, \bfx_2) \\
&& ~+~ \Cov_{i \sim \calU\{1, \hdots, \Nens\}}(f_{\bftheta^{(i)}}(\bfx_i), f_{\bftheta^{(i)}}(\bfx_j))~.
\end{IEEEeqnarray*}
Hence, the predictive covariance for a mixture of Laplace approximations is approximately given by an ensembled posterior kernel plus the covariance of the ensemble predictions. Note that due to the summation, it is important that the (posterior) kernel is scaled correctly. We leave an experimental evaluation of this approach to future work.

\subsection{Sketching} \label{sec:appendix:random_projections}

In the following, we prove the variant of the Johnson-Lindenstrauss theorem mentioned in \Cref{sec:kernel_transformations:rp}.

\thmjl*

\begin{proof}
By Theorem 1 from \cite{arriaga_algorithmic_1999}, the following bound holds for fixed $\bfx, \tilde\bfx \in \calX$ with probability\footnote{We inserted the factor $2$ in front of $e^{-\varepsilon^2 p / 8}$ that has been forgotten in their Theorem 1.} $\geq 1 - 2e^{-\varepsilon^2 p / 8}$:
\begin{equation}
    (1-\varepsilon) d_k(\bfx, \tilde\bfx)^2 \leq d_{k_{\to\rp{p}}}(\bfx, \tilde\bfx)^2 \leq (1+\varepsilon) d_k(\bfx, \tilde\bfx)^2~. \label{eq:eps_isometry_sq}
\end{equation}
Now, \eqref{eq:eps_isometry_sq} implies \eqref{eq:eps_isometry} for a single pair $\bfx, \tilde\bfx$ because of
\begin{IEEEeqnarray*}{+rCl+x*}
(1-\varepsilon)^2 \leq (1-\varepsilon) \leq (1+\varepsilon) \leq (1+\varepsilon)^2~.
\end{IEEEeqnarray*}
To obtain the bound for all pairs $(\bfx, \tilde\bfx)$, we note that the bound is trivial for pairs $(\bfx, \bfx)$ and it is equivalent for the pairs $(\bfx, \tilde\bfx)$ and $(\tilde\bfx, \bfx)$. Hence, we only have to consider less than $\frac{|\calX|^2}{2}$ pairs. By the union bound, the probability that \eqref{eq:eps_isometry} holds for all pairs $(\bfx, \tilde\bfx)$ is at least
\begin{equation*}
    1 - 2\frac{|\calX|^2}{2}e^{-\varepsilon^2 p/8} \stackrel{\text{\eqref{eq:min_dim}}}{\geq} 1 - \delta~. \qedhere
\end{equation*}
\end{proof}

\subsection{ACS Random Features Transformation} \label{sec:appendix:acs_rf}

We will now derive the form of $k_{\to\acs}$ presented in \Cref{sec:kernel_transformations:acs_rf}. Partially following the notation of \cite{pinsler_bayesian_2019}, we want to compute
\begin{IEEEeqnarray*}{+rCl+x*}
k_{\to\acs}(\bfx_n, \bfx_m) \equalDef \langle \calL_n, \calL_m \rangle = \bbE_{\bftheta \sim p(\bftheta|\Dtrain)} [\calL_n(\bftheta)\calL_m(\bftheta)]~,
\end{IEEEeqnarray*}
where
\begin{IEEEeqnarray*}{+rCl+x*}
\calL_m(\bftheta) = \bbE_{y_m \sim p(\cdot \mid \bfx_m, \Dtrain)} [\log p(y_m \mid \bfx_m, \bftheta)] + \bbH[y_m \mid \bfx_m, \Dtrain]
\end{IEEEeqnarray*}
with $\bbH[y_m \mid \bfx_m, \Dtrain]$ denoting the conditional entropy of $y_m$ given $\bfx_m$ and $\Dtrain$.

In a GP model without hyper-prior on $\sigma^2$, we have %
\begin{IEEEeqnarray*}{+rCl+x*}
p(y_m \mid \bfx_m, \bftheta) & = & \calN(y_m \mid f_{\bftheta_T}(\bfx_m) + \bftheta^\top \phi_{\to\scale{\Xtrain}}(\bfx_m), \sigma^2) \\
p(y_m \mid \bfx_m, \Dtrain) & = & \calN(y_m \mid f_{\bftheta_T}(\bfx_m), k_{\to\Xtrain}(\bfx_m, \bfx_m) + \sigma^2) \\
\bbH(\calN(\bfmu, \bfSigma)) & = & \frac{n}{2} \ln(2\pi e) + \frac{1}{2} \ln(\det(\bfSigma)) \quad \text{for } \bfmu \in \bbR^n, \bfSigma \in \bbR^{n \times n}~,
\end{IEEEeqnarray*}
which allows us to derive
\begin{IEEEeqnarray*}{+rCl+x*}
\bbH[y_m \mid \bfx_m, \Dtrain] & = & \frac{1}{2} + \frac{1}{2} \log(2\pi (k_{\to\Xtrain}(\bfx_m, \bfx_m) + \sigma^2))~.
\end{IEEEeqnarray*}
By shifting the prior and the log-likelihood by $f_{\bftheta_T}(\bfx_m)$, we obtain
\begin{IEEEeqnarray*}{+rCl+x*}
&& \bbE_{y_m \sim p(\cdot \mid \bfx_m, \calD_0)} [\log p(y_m \mid \bfx_m, \bftheta)] \\
& = & \bbE_{y_m \sim \calN(0, k_{\to\Xtrain}(\bfx_m, \bfx_m) + \sigma^2)} [\log \calN(y_m \mid \bftheta^\top \phi_{\to\scale{\Xtrain}}(\bfx_m), \sigma^2)] \\
& = & -\frac{1}{2}\log(2\pi \sigma^2) - \bbE_{y_m \sim \calN(0, k_{\to\Xtrain}(\bfx_m, \bfx_m) + \sigma^2)} \frac{(y_m - \bftheta^\top \phi_{\to\scale{\Xtrain}}(\bfx_m))^2}{2\sigma^2} \\
& = & -\frac{1}{2}\log(2\pi \sigma^2) - \frac{1}{2\sigma^2} \left((\bftheta^\top \phi_{\to\scale{\Xtrain}}(\bfx_m))^2 + k_{\to\Xtrain}(\bfx_m, \bfx_m) + \sigma^2\right)~.
\end{IEEEeqnarray*}
Therefore, we have $k_{\to\acs}(\bfx, \tilde\bfx) \equalDef \bbE_{\bftheta \sim P(\bftheta \mid \Dtrain)} [\facs(\bfx, \bftheta) \facs(\tilde\bfx, \bftheta)]$ with
\begin{IEEEeqnarray*}{+rCl+x*}
\facs(\bfx_m, \bftheta) \equalDef \calL_m(\bftheta) & = & \frac{1}{2} \log\left(1 + \frac{k_{\to\Xtrain}(\bfx_m, \bfx_m)}{\sigma^2}\right) \\
&& ~-~ \frac{(\bftheta^\top \phi_{\to\scale{\Xtrain}}(\bfx_m))^2 + k_{\to\Xtrain}(\bfx_m, \bfx_m)}{2\sigma^2}~.
\end{IEEEeqnarray*}
Moreover, it follows from Eq.~(4) in \cite{pinsler_bayesian_2019} that
\begin{IEEEeqnarray*}{+rCl+x*}
f_{\mathrm{pool}}(\bftheta) - f_{\mathrm{batch}}(\bftheta) = \sum_{\bfx_m \in \Xpool \setminus \Xbatch} \calL_m(\bftheta) = \sum_{\bfx \in \Xpool \setminus \Xbatch} \facs(\bfx, \bftheta)~.
\end{IEEEeqnarray*}

\section{Details on Selection Methods} \label{sec:appendix:selection}

In this section, we will provide efficiency-focused pseudocode for all selection methods and analyze its runtime and memory complexity. Hereby, we will neglect that the required integer bit size for indexing elements of $\Xpool$ and $\Xtrain$ grows logarithmically with $\Npool + \Ntrain$. For some selection methods, we will additionally discuss relations to the literature and theoretical properties. \Cref{sec:appendix:iterative_selection} will first provide a pseudocode structure for iterative selection, where the missing components are then specified for the respective selection methods in the subsequent sections. The following notation will be used throughout this section:

We allow having vectors and matrices indexed by points $\bfx$ instead of indices $i \in \bbN$, which we write with square brackets as $\bfv[\bfx]$ or $\bfM[\bfx, \tbfx]$. In a practical implementation, where the points $\bfx \in \Xpool$ are for example numbered as $\bfx_1, \hdots, \bfx_{\Npool}$, one may simply use $\bfv[i]$ instead of $\bfv[\bfx_i]$. Again, we assume in pseudocode that all $\bfx$ are distinct, such that we can use set notation, but identical copies of $\bfx$ should be treated as distinct. This problem also disappears when using indices. We denote by $\bfu \odot \bfv$ the element-wise (Hadamard) product of the vectors $\bfu$ and $\bfv$. Whenever an $\argmax$ is not unique, we leave the choice of the maximizer to the implementation.

\subsection{Iterative Selection Scheme} \label{sec:appendix:iterative_selection}

While our iterative selection template shown in \Cref{alg:select_simple} is sufficient for a high-level understanding, it is not well-suited for an efficient implementation. To this end, we present a more detailed iterative selection template in \Cref{alg:select_simple}, which closely matches our open-source implementation. The template in \Cref{alg:select_simple} involves three methods called \textsc{Init}, \textsc{Add}, and \textsc{Next}, which are allowed to have side effects, i.e.\ access and modify common variables. For each selection method except \Random{} and \MaxDiag{}, our implementation directly mirrors this structure, containing a class providing the three methods together with \textsc{Select} from \Cref{alg:select_simple}.

\begin{algorithm}[tb]
\caption{Iterative selection algorithm template involving three customizable functions \textsc{Init}, \textsc{Add} and \textsc{Next} that are allowed to have side effects (i.e., read/write variables in \textsc{Select}).} \label{alg:iterative_selection_detailed}
\begin{algorithmic}
\Function{Select}{$k$, $\Xtrain$, $\Xpool$, $\Nbatch$, mode $\in \{$P, TP$\}$}
	\State $\Xmode \assign \Xtrain$ if mode = TP else $\emptyset$
	\State $\Xcand \equalDef \Xmode \cup \Xpool$
	\State $\Xbatch \assign \emptyset$
	\State \Call{Init}{}
	\For{$\bfx$ in $\Xmode$}
		\State \Call{Add}{$\bfx$}
	\EndFor
	\For{$i$ from $1$ to $\Nbatch$}
		\State $\bfx \assign $\Call{Next}{}  %
		\If{$\bfx \in \Xbatch \cup \Xtrain$ (failed selection)}
			\State fill up $\Xbatch$ with $\Nbatch - |\Xbatch|$ random samples from $\Xpool \setminus \Xbatch$ and return $\Xbatch$
		\EndIf
		\State $\Xbatch \assign \Xbatch \cup \{\bfx\}$
		\State \Call{Add}{$\bfx$}
	\EndFor
	\State \Return $\Xbatch$
\EndFunction
\end{algorithmic}
\end{algorithm}

\subsection{\Random{}} \label{sec:appendix:random}

We implement \Random{} by taking the first $\Nbatch$ indices out of a random permutation of $\{1, \hdots, \Npool\}$. Since there are $\Npool!$ possible random permutations, this requires at least $\log_2(\Npool!) = \BigO(\Npool \log \Npool)$ random bits, so the runtime for this suboptimal implementation is, in theory, \quot{only} $\BigO(\Npool \log \Npool)$, which is still extremely fast in practice. The memory complexity for our suboptimal implementation is $\BigO(\Npool)$.

\subsection{\MaxDiag{}} \label{sec:appendix:maxdiag}

A simple implementation of \MaxDiag{} is shown in \Cref{alg:maxdiag}. The runtime of this implementation is $\BigO(\Npool \log \Npool)$ due to sorting. While other algorithms might be faster for $\Nbatch \ll \Npool$, the runtime is already very fast in practice. The memory complexity is $\BigO(\Npool)$.

\begin{algorithm}[tb]
\caption{\MaxDiag{} pseudocode implementation using \Cref{alg:iterative_selection_detailed}.} \label{alg:maxdiag}
\begin{algorithmic}
\Function{Init}{}
	\State Sort elements in $\Xpool$ as $\tilde\bfx_1, \hdots, \tilde\bfx_{\Npool}$ such that $k(\tilde\bfx_1, \tilde\bfx_1) \geq \hdots \geq k(\tilde\bfx_{\Npool}, \tilde\bfx_{\Npool})$
\EndFunction\\
\Function{Add}{$\bfx$}
\EndFunction\\
\Function{Next}{}
	\State \Return $\tilde\bfx_i$
\EndFunction
\end{algorithmic}
\end{algorithm}

\subsection{\MaxDet{}} \label{sec:appendix:maxdet}

\subsubsection{Equivalence of \MaxDet{} to Non-batch Mode Active Learning With Fixed Kernel} Using the Schur determinant formula
\begin{IEEEeqnarray*}{+rCl+x*}
\det \begin{pmatrix}
\bfA & \bfB \\
\bfC & \bfD
\end{pmatrix} = \det(\bfA) \det(\bfD - \bfC \bfA^{-1} \bfB)~,
\end{IEEEeqnarray*}
we can compute
\begin{IEEEeqnarray*}{+rCl+x*}
&& \det(k(\calX \cup \{\bfx\}, \calX \cup \{\bfx\}) + \sigma^2 \bfI) \\
& = & \det \begin{pmatrix}
k(\calX, \calX) + \sigma^2 \bfI & k(\calX, \bfx) \\
k(\bfx, \calX) & k(\bfx, \bfx) + \sigma^2
\end{pmatrix} \\
& = & \det(k(\calX, \calX) + \sigma^2 \bfI) \det(k(\bfx, \bfx) + \sigma^2 - k(\bfx, \calX) (k(\calX, \calX) + \sigma^2 \bfI)^{-1} k(\calX, \bfx)) \\
& = & \det(k(\calX, \calX) + \sigma^2 \bfI) \cdot (\sigma^2 + k_{\to\post{\calX, \sigma^2}}(\bfx, \bfx))~. \IEEEyesnumber \label{eq:maxdet_onestep_equiv}
\end{IEEEeqnarray*}
This shows that
\begin{IEEEeqnarray*}{+rCl+x*}
\argmax_{\bfx \in \Xrem} k_{\to\post{\Xsel, \sigma^2}}(\bfx, \bfx) = \argmax_{\bfx \in \Xrem} \det(k(\Xsel \cup \{\bfx\}, \Xsel \cup \{\bfx\}) + \sigma^2 \bfI)~.
\end{IEEEeqnarray*}

\subsubsection{Equivalence of \MaxDet{} to BatchBALD on a GP} As in \Cref{sec:kernel_transformations:post}, we consider a GP model in feature space, given by $y_i = \bfw^\top \phi(\bfx_i) + \varepsilon_i$ with weight prior $\bfw \sim \calN(\bfzero, \bfI)$ and i.i.d.\ observation noise $\varepsilon_i \sim \calN(0, \sigma^2)$. The objective of BatchBALD \citep{kirsch_batchbald_2019} is to maximize the mutual information
\begin{equation*}
    a(\Xbatch) \equalDef \bbH(\Ybatch \mid \Xbatch, \Dtrain) - \bbE_{\bfw \sim p(\bfw|\Dtrain)} \bbH(\Ybatch \mid \Xbatch, \Dtrain, \bfw)~,
\end{equation*}
where $\bbH$ refers to the (conditional) entropy. Writing $\Ybatch^* \equalDef \bbE[\Ybatch\mid\Xbatch, \Dtrain]$, we have
\begin{align*}
    p(\Ybatch \mid \Xbatch, \Dtrain) &= \calN(\Ybatch \mid \Ybatch^*, k(\Xbatch, \Xbatch) + \sigma^2 \bfI) \\
    p(\Ybatch \mid \Xbatch, \Dtrain, \bfw) &= \calN(\Ybatch \mid \phi(\Xtrain)^\top \bfw, \sigma^2 \bfI) \\
    \bbH(\calN(\bfmu, \bfSigma)) &= \frac{n}{2} \ln(2\pi e) + \frac{1}{2} \ln(\det(\bfSigma)) \quad \text{for } \bfmu \in \bbR^n, \bfSigma \in \bbR^{n \times n}~.
\end{align*}
Hence, we can compute
\begin{equation}
    a(\Xbatch) = \frac{1}{2} \ln(\det(k(\Xbatch, \Xbatch) + \sigma^2 \bfI)) - \Nbatch \ln(\sigma)~. \label{eq:batchbald_gp}
\end{equation}
This shows that greedy maximization of the BatchBALD acquisition function, as proposed by \cite{kirsch_batchbald_2019}, is equivalent to \MaxDet{}. \cite{kirsch_batchbald_2019} showed that for any Bayesian model, greedy optimization of $a$ is suboptimal by a factor of at most $(1-1/e)$. By applying this result to the GP model, we obtain that the same suboptimality bound applies to \MaxDet{} for the form of $a$ given in \eqref{eq:batchbald_gp}.

\subsubsection{Equivalence of \MaxDet{} to the P-greedy Algorithm} In our notation, the P-greedy algorithm \citep{de_marchi_near-optimal_2005} can be written as
\begin{IEEEeqnarray*}{+rCl+x*}
\NextSample(k, \Xsel, \Xrem) = \argmax_{\bfx \in \Xrem} P_{k, \Xsel}(\bfx)~,
\end{IEEEeqnarray*}
where by Lemma 4.1 in \cite{de_marchi_near-optimal_2005}, the non-negative power function $P_{k, \Xsel}$ can be written as
\begin{IEEEeqnarray*}{+rCl+x*}
P_{k, \Xsel}(\bfx)^2 & = & k(\bfx, \bfx) - k(\bfx, \Xsel)k(\Xsel, \Xsel)^{-1} k(\Xsel, \bfx)~.
\end{IEEEeqnarray*}
A calculation analogous to \eqref{eq:maxdet_onestep_equiv} therefore shows that P-greedy is equivalent to \MaxDet{} with $\sigma^2 = 0$. 

\subsubsection{Relation to the Greedy Algorithm for D-optimal Design} In our notation, the D-optimal design problem \citep{wald_efficient_1943} is to maximize the determinant $\det(\phi(\Xsel)^\top \phi(\Xsel))$, which can only be nonzero in the case $\Nsel \geq \dfeat$. It can be seen as the $\sigma \to 0$ limit for $\Nsel \geq \dfeat$ of the determinant-maximization objective that motivates \MaxDet{} since an eigenvalue-based argument shows that
\begin{IEEEeqnarray*}{+rCl+x*}
\det(\phi(\Xsel)^\top \phi(\Xsel)+\sigma^2 \bfI) = \sigma^{\dfeat - \Nsel} \det(k(\Xsel, \Xsel) + \sigma^2 \bfI)~.
\end{IEEEeqnarray*}
In this sense, the corresponding greedy algorithm \citep{wynn_sequential_1970} is the underparameterized ($\dfeat \leq \Nsel$) analog to the P-greedy algorithm, since the latter can only be well-defined in the overparameterized regime ($\dfeat \geq \Nsel$). Some guarantees for this greedy algorithm are given by \cite{wynn_sequential_1970} and \cite{madan_combinatorial_2019}. While the classical D-optimal design uses $\sigma = 0$, Bayesian D-optimal design uses $\sigma > 0$ and is thus even more directly related to \MaxDet{} \citep{chaloner_bayesian_1995}.

\subsubsection{Kernel-space Implementation of \MaxDet{}} In the following, we want to derive an efficient kernel-space implementation of \MaxDet{}. Let $\Xcand \equalDef \Xmode \cup \Xpool$. We perform a partial pivoted Cholesky decomposition of the matrix $\bfM = k(\Xcand, \Xcand) + \sigma^2 \bfI$, which has been suggested for $P$-greedy by \cite{pazouki_bases_2011} and in the context of determinantal point processes by \cite{chen_fast_2018}. We denote submatrices of $\bfM$ for example by $\bfM[\calX, \bfx] \equalDef (M_{\tilde\bfx, \bfx})_{\tilde\bfx \in \calX}$.

Suppose that at the current step, the points $\calX \equalDef \Xmode \cup \Xbatch$ have already been added. Consider the Cholesky decomposition $\bfM[\calX, \calX] = \bfL(\calX)\bfL(\calX)^\top$. Then, the Cholesky decomposition for $\calX \cup \{\bfx\}$ is of the form
\begin{IEEEeqnarray*}{+rCl+x*}
\begin{pmatrix}
\bfM[\calX, \calX] & \bfM[\calX, \bfx] \\
\bfM[\bfx, \calX] & \bfM[\bfx, \bfx]
\end{pmatrix} & = & \bfL(\calX \cup \{\bfx\}) \bfL(\calX \cup \{\bfx\})^\top \\
& = & \begin{pmatrix}
\bfL(\calX) & \bfzero \\
\bfb(\calX, \bfx)^\top & \sqrt{c(\calX, \bfx)}
\end{pmatrix} \begin{pmatrix}
\bfL(\calX) & \bfzero \\
\bfb(\calX, \bfx)^\top & \sqrt{c(\calX, \bfx)}
\end{pmatrix}^\top~, \IEEEyesnumber \label{eq:cholesky_recursive}
\end{IEEEeqnarray*}
which implies $\bfb(\calX, \bfx) = \bfL(\calX)^{-1}\bfM[\calX, \bfx]$ and $c(\calX, \bfx) = \bfM[\bfx, \bfx] - \|\bfb(\calX, \bfx)\|_2^2$. Using the general inversion formula for block-triangular matrices given by
\begin{IEEEeqnarray*}{+rCl+x*}
\begin{pmatrix}
\bfA & \bfzero \\
\bfB & \bfC
\end{pmatrix}^{-1} = \begin{pmatrix}
\bfA^{-1} & \bfzero \\
-\bfC^{-1} \bfB \bfA^{-1} & \bfC^{-1}
\end{pmatrix}~,
\end{IEEEeqnarray*}
we obtain
\begin{IEEEeqnarray*}{+rCl+x*}
\bfb(\calX \cup \{\bfx\}, \tbfx) & = & \bfL(\calX \cup \{\bfx\})^{-1} \bfM[\calX \cup \{\bfx\}, \tbfx] \\
& = & \begin{pmatrix}
\bfL(\calX)^{-1} & \bfzero \\
-c(\calX, \bfx)^{-1/2} \bfb(\calX, \bfx)^\top \bfL(\calX)^{-1} & c(\calX, \bfx)^{-1/2}
\end{pmatrix} \begin{pmatrix}
\bfM[\calX, \tbfx] \\ \bfM[\bfx, \tbfx]
\end{pmatrix} \\
& = & \begin{pmatrix}
\bfb(\calX, \tbfx) \\
c(\calX, \bfx)^{-1/2}(\bfM[\bfx, \tbfx] - \bfb(\calX, \bfx)^\top \bfb(\calX, \tbfx)) \IEEEyesnumber \label{eq:b_update}
\end{pmatrix}
\end{IEEEeqnarray*}
and therefore
\begin{IEEEeqnarray*}{+rCl+x*}
c(\calX \cup \{\bfx\}, \tbfx) = c(\calX, \tbfx) - c(\calX, \bfx)^{-1}(\bfM[\bfx, \tbfx] - \bfb(\calX, \bfx)^\top \bfb(\calX, \tbfx))^2~. \IEEEyesnumber \label{eq:c_update}
\end{IEEEeqnarray*}

With the above considerations, we can implement \MaxDet{} as follows, which is given as pseudocode in \Cref{alg:maxdet_kernel}:
\begin{itemize}
\item For \textsc{Init}, we initialize $\bfb(\emptyset, \bfx)$ to an empty vector and $\bfc(\emptyset, \bfx) = M[\bfx, \bfx]$. We do not precompute $\bfM$ since not all entries of $\bfM$ will be used, otherwise, the runtime complexity would be quadratic in $\Ncand$.
\item For \textsc{Next}, note that by \eqref{eq:cholesky_recursive},
\begin{IEEEeqnarray*}{+rCl+x*}
\det \bfM[\calX \cup \{\bfx\}, \calX \cup \{\bfx\}] = c(\calX, \bfx) \cdot \det(\bfL(\calX))^2~,
\end{IEEEeqnarray*}
hence
\begin{IEEEeqnarray*}{+rCl+x*}
\argmax_{\bfx \in \Xpool \setminus \Xbatch} \det \bfM[\calX \cup \{\bfx\}, \calX \cup \{\bfx\}] = \argmax_{\bfx \in \Xpool \setminus \Xbatch} c(\calX, \bfx)~.
\end{IEEEeqnarray*}
\item For \textsc{Add}, we update the $\bfb$ and $c$ values as in \eqref{eq:b_update} and \eqref{eq:c_update}.
\end{itemize}

Regarding the runtime complexity, the \textsc{Add} step is clearly the most expensive one, requiring $\BigO(\Ncand \Nsel)$ operations for computing $\bfB^\top \bfB[\cdot, \bfx]$ and $\BigO(\Ncand T_k)$ operations for computing $k(\Xcand, \bfx)$. Since \textsc{Add} is called $\Nsel$ times, the total runtime complexity is therefore $\BigO(\Ncand \Nsel(T_k + \Nsel))$. The total memory complexity is $\BigO(\Ncand \Nsel)$, required for storing $\bfB$.

\begin{algorithm}[tb]
\caption{\MaxDet{} pseudocode implementation in kernel space for given $\sigma^2 \geq 0$ using \Cref{alg:iterative_selection_detailed}.} \label{alg:maxdet_kernel}
\begin{algorithmic}
\Function{Init}{}
	\State $\bfB \assign$ empty $0 \times \Xcand$ matrix
	\State $\bfc \assign (k(\bfx, \bfx) + \sigma^2)_{\bfx \in \Xcand}$
\EndFunction\\
\Function{Add}{$\bfx$}
	\State $\bfv \assign \sqrt{\bfc} \odot ((k(\Xcand, \bfx) + \sigma^2\bfI[\Xcand, \bfx] - \bfB^\top \bfB[\cdot, \bfx])$ \Comment{$\sqrt{\bfc}$ should be understood element-wise}
	\State $\bfB \assign \begin{pmatrix}
	\bfB \\ \bfv^\top
	\end{pmatrix}$
	\State $\bfc \assign \bfc - \bfv \odot \bfv$
\EndFunction\\
\Function{Next}{}
	\State \Return $\argmax_{\bfx \in \Xpool \setminus \Xbatch} c[\bfx]$
\EndFunction
\end{algorithmic}
\end{algorithm}

\subsubsection{Feature-space Implementation of \MaxDet{}} In cases where the feature space dimension $\dfeat$ of $\phi$ is smaller than the number $\Nsel$ of added samples, it can be beneficial to implement \MaxDet{} in feature space instead. For $\Xsel \equalDef \Xmode \cup \Xbatch$, we want to compute 
\begin{IEEEeqnarray*}{+rCl+x*}
\argmax_{\bfx \in \Xpool \setminus \Xbatch} k_{\to\post{\Xsel, \sigma^2}}(\bfx, \bfx)~.
\end{IEEEeqnarray*}
Now, from \eqref{eq:posterior_fm}, we know that
\begin{IEEEeqnarray*}{+rCl+x*}
k_{\to\post{\Xsel, \sigma^2}}(\bfx, \bfx) & = & \phi(\bfx)^\top \hat\bfSigma_{\Xsel}^{-1} \phi(\bfx), \qquad \hat\bfSigma_{\Xsel} \equalDef \sigma^{-2} \phi(\Xsel)^\top \phi(\Xsel) + \bfI~.
\end{IEEEeqnarray*}
Adding a point to $\Xsel$ leads to a rank-1 update of $\hat\bfSigma_{\Xsel}$, and the corresponding update of $\hat\bfSigma_{\Xsel}^{-1}$ can be computed using the Sherman-Morrison formula. This gives rise to three approaches towards implementing \MaxDet{} in feature space:
\begin{enumerate}[(1)]
\item Keep track of $\phi(\bfx)$ and $\hat\bfSigma_{\Xsel}^{-1}$. Compute $k_{\to\post{\Xsel, \sigma^2}}(\bfx, \bfx) = \phi(\bfx)^\top \hat\bfSigma_{\Xsel}^{-1} \phi(\bfx)$ in each step.
\item Keep track of $\phi(\bfx)$ and $\psi(\bfx) \equalDef \hat\bfSigma_{\Xsel}^{-1} \phi(\bfx)$. Compute 
\begin{IEEEeqnarray*}{+rCl+x*}
k_{\to\post{\Xsel, \sigma^2}}(\bfx, \bfx) = \phi(\bfx)^\top \psi(\bfx)
\end{IEEEeqnarray*}
in each step.
\item Keep track of $\phi_{\to\post{\Xsel, \sigma^2}}(\bfx)$, one possible realization of which is $\phi_{\to\post{\Xsel, \sigma^2}}(\bfx) = \hat\bfSigma_{\Xsel}^{-1/2} \phi(\bfx)$. Compute $k_{\to\post{\Xsel, \sigma^2}}(\bfx, \bfx) = \phi_{\to\post{\Xsel, \sigma^2}}(\bfx)^\top \phi_{\to\post{\Xsel, \sigma^2}}(\bfx)$ in each step.
\end{enumerate}
Option (1) is less computationally efficient than (2), (3) since one needs to compute matrix-vector products instead of only inner products. Version (2) and (3) are similar, but here we favor (3) since it only requires storing one vector instead of two for each $\bfx$. 
Since
\begin{IEEEeqnarray*}{+rCl+x*}
\phi_{\to\post{\Xsel\cup\{\bfx\}, \sigma^2}} = \phi_{\to\post{\Xsel, \sigma^2}\to\post{\{\bfx\}, \sigma^2}}~,
\end{IEEEeqnarray*}
we will now consider how to efficiently compute a single posterior update $\phi_{\to\post{\{\bfx\}, \sigma^2}}$. To this end, we first consider how to compute matrix square roots of specific rank-1 updates:

\begin{lemma} \label{lemma:rank_1_update_sqrt}
Let $\bfv \in \bbR^p$ and let $c \geq -\frac{1}{\bfv^\top \bfv}$. Then,
\begin{IEEEeqnarray*}{+rCl+x*}
\bfI + c\bfv \bfv^\top & = & \left(\bfI + \frac{c}{1 + \sqrt{1+c\bfv^\top \bfv}} \bfv \bfv^\top \right)^2~.
\end{IEEEeqnarray*}
\begin{proof}
Due to the condition on $c$, the square root is well-defined. We have
\begin{IEEEeqnarray*}{+rCl+x*}
\left(\bfI + \frac{c}{1 + \sqrt{1+c\bfv^\top \bfv}} \bfv \bfv^\top \right)^2 & = & \bfI + C \bfv \bfv^\top~,
\end{IEEEeqnarray*}
where
\begin{IEEEeqnarray*}{+rCl+x*}
C & = & \frac{2c}{1 + \sqrt{1+c\bfv^\top \bfv}} + \frac{c^2\bfv^\top \bfv}{(1 + \sqrt{1+c\bfv^\top \bfv})^2} = \frac{2c(1+\sqrt{1+c\bfv^\top \bfv}) + c^2\bfv^\top\bfv}{2+2\sqrt{1+c\bfv^\top \bfv}+c\bfv^\top \bfv} = c~. & \qedhere
\end{IEEEeqnarray*}
\end{proof}
\end{lemma}

The following proposition shows how to update the posterior feature map after observing a point $\bfx$:

\begin{proposition}[Forward update] \label{prop:forward_update}
Let $\sigma^2 > 0$, let $k$ be a kernel and let $\tilde k \equalDef k_{\to \post{\{\bfx\}, \sigma^2}}$. Then,
\begin{IEEEeqnarray*}{+rCl+x*}
\tilde k(\bfx', \bfx'') & = & k(\bfx', \bfx'') - k(\bfx', \bfx)(k(\bfx, \bfx) + \sigma^2 \bfI)^{-1} k(\bfx, \bfx'')
\end{IEEEeqnarray*}
Consequently, if $\phi$ is a feature map for $k$, then
\begin{IEEEeqnarray*}{+rCl+x*}
\tilde \phi(\bfx') \equalDef \left(\bfI - \frac{\phi(\bfx)\phi(\bfx)^\top}{\sigma^2 + \phi(\bfx)^\top \phi(\bfx)}\right)^{1/2} \phi(\bfx') = \left(\bfI - \beta \phi(\bfx)\phi(\bfx)^\top\right) \phi(\bfx')
\end{IEEEeqnarray*}
is a feature map for $\tilde k$, where
\begin{IEEEeqnarray*}{+rCl+x*}
\beta \equalDef \frac{1}{\sqrt{\sigma^2 + \phi(\bfx)^\top \phi(\bfx)}\left(\sqrt{\sigma^2 + \phi(\bfx)^\top \phi(\bfx)} + \sigma\right)}~.
\end{IEEEeqnarray*}

\begin{proof}
The kernel update equation follows directly from \eqref{eq:posterior_kernel}.

\textbf{Step 1: Feature map.} The specified feature map $\tilde\phi$ satisfies
\begin{IEEEeqnarray*}{+rCl+x*}
\tilde\phi(\bfx')^\top \tilde\phi(\bfx'') & = & \phi(\bfx')^\top \left(\bfI - \frac{\phi(\bfx)\phi(\bfx)^\top}{\sigma^2 + \phi(\bfx)^\top \phi(\bfx)}\right)  \phi(\bfx'') \\
& = & k(\bfx', \bfx'') - \tilde k(\bfx', \bfx) (\sigma^2 + k(\bfx, \bfx))^{-1} k(\bfx, \bfx'') \\
& = & \tilde k(\bfx', \bfx'')~,
\end{IEEEeqnarray*}
hence it is a feature map for $\tilde k$.

\textbf{Step 2: Square root.} According to \Cref{lemma:rank_1_update_sqrt}, we have
\begin{IEEEeqnarray*}{+rCl+x*}
&& \left(\bfI - \frac{\phi(\bfx) \phi(\bfx)^\top}{\sigma^2 + \phi(\bfx)^\top \phi(\bfx)}\right)^{1/2} \\
& = & \bfI - \frac{1}{(\sigma^2 + \phi(\bfx)^\top \phi(\bfx))\left(1 + \sqrt{1 - \phi(\bfx)^\top \phi(\bfx)/(\sigma^2 + \phi(\bfx)^\top \phi(\bfx))}\right)} \phi(\bfx)\phi(\bfx)^\top \\
& = & \bfI - \frac{1}{(\sigma^2 + \phi(\bfx)^\top \phi(\bfx))(1 + \sigma(\sigma^2 + \phi(\bfx)^\top \phi(\bfx))^{-1/2})} \phi(\bfx)\phi(\bfx)^\top \\
& = & \bfI - \frac{1}{\sqrt{\sigma^2 + \phi(\bfx)^\top \phi(\bfx)}(\sqrt{\sigma^2 + \phi(\bfx)^\top \phi(\bfx)} + \sigma)} \phi(\bfx)\phi(\bfx)^\top~. & \qedhere
\end{IEEEeqnarray*}
\end{proof}
\end{proposition}

In our implementation, we keep track of the following quantities:
\begin{IEEEeqnarray*}{+rCl+x*}
\bfPhi_{\calX}[\bfx] & \equalDef & \phi_{\to\post{\calX, \sigma^2}}(\bfx) \\
c_{\calX}[\bfx] & \equalDef & \bfPhi_{\calX}[\bfx]^\top \bfPhi_{\calX}[\bfx]~.
\end{IEEEeqnarray*}
Note that $\phi_{\to\post{\calX, \sigma^2}}$ is not uniquely defined, since any rotation of $\phi_{\to\post{\calX, \sigma^2}}$ leads to the same kernel. When computing $\bfPhi$ as defined above, we do not care which version of $\phi_{\to\post{\calX, \sigma^2}}$ it corresponds to, as long as the same version of $\phi_{\to\post{\calX, \sigma^2}}$ is used for all $\bfx$.

Following \Cref{prop:forward_update}, $\bfPhi$ and $c$ can be updated with a new observation $\bfx \notin \calX$ as follows:
\begin{IEEEeqnarray*}{+rCl+x*}
\bfPhi_{\calX \cup \{\bfx\}}[\bfx'] & = & \bfPhi_{\calX}[\bfx'] - \beta_{\calX}(\bfx) \bfPhi_{\calX}[\bfx] \langle \bfPhi_{\calX}[\bfx], \bfPhi_{\calX}[\bfx']\rangle \\
\bfc_{\calX \cup \{\bfx\}}[\bfx'] & = & \bfPhi_{\calX}[\bfx']^\top \left(\bfI - \frac{\bfPhi_{\calX}[\bfx]\bfPhi_{\calX}[\bfx]^\top}{\sigma^2 + \bfPhi_{\calX}[\bfx]^\top \bfPhi_{\calX}[\bfx]}\right) \bfPhi_{\calX}[\bfx'] \\
& = & \bfc_{\calX}[\bfx'] - \gamma_{\calX}(\bfx)^{-2} \langle \Phi_{\calX}[\bfx], \Phi_{\calX}[\bfx']\rangle^2~,
\end{IEEEeqnarray*}
where
\begin{IEEEeqnarray*}{+rCl+x*}
\gamma_{\calX}(\bfx) & \equalDef & \sqrt{\sigma^2 + \bfc_{\calX}[\bfx]} \\
\beta_{\calX}(\bfx) & \equalDef & \frac{1}{\gamma_{\calX}(\bfx)(\gamma_{\calX}(\bfx) + \sigma)}~.
\end{IEEEeqnarray*}
Together, these considerations lead to a feature-space implementation of \MaxDet{}, presented in \Cref{alg:maxdet_features}.

For the complexity analysis of \Cref{alg:maxdet_features}, we exclude the computation of the feature matrix $\phi(\Xcand)$. As for the kernel-space version of \MaxDet{}, the runtime is then dominated by the runtime of \textsc{Add}, which has a runtime complexity of $\BigO(\Ncand \dfeat)$. Since it is called $\Nsel$ times, the total runtime complexity of \Cref{alg:maxdet_features} is $\BigO(\Ncand \Nsel \dfeat)$. If the kernel $k$ is evaluated by an inner product of the pre-computed features, the runtime of a kernel evaluation scales as $T_k = \Theta(\dfeat)$. In this case, the runtime of the kernel-space version in \Cref{alg:maxdet_kernel} has a runtime complexity of $\BigO(\Ncand\Nsel(\dfeat + \Nsel))$, which is not asymptotically better than the one for \Cref{alg:maxdet_features}. However, in our implementation, we observe that the kernel-space version typically runs faster for $\Nsel \lesssim 3\dfeat$, which can be attributed to the smaller constant in the runtime of \textsc{Add}. It is easily verified that the memory complexity of \Cref{alg:maxdet_features} scales as $\BigO(\Ncand \dfeat)$.

\begin{algorithm}[tb]
\caption{\MaxDet{} pseudocode implementation in feature space for given $\sigma^2 \geq 0$ using \Cref{alg:iterative_selection_detailed}.} \label{alg:maxdet_features}
\begin{algorithmic}
\Function{Init}{}
	\State $\bfPhi \assign \phi(\Xcand) \in \bbR^{\Ncand \times \dfeat}$ \Comment{feature matrix}
	\State $\bfc \assign (\langle \bfPhi[\bfx, \cdot]^\top, \bfPhi[\bfx, \cdot]^\top\rangle)_{\bfx \in \Xcand}$ \Comment{vector containing the kernel diagonal}
\EndFunction\\
\Function{Add}{$\bfx$}
	\State $\gamma \assign \sqrt{\sigma^2 + \bfc[\bfx]}$
	\State $\beta \assign (\gamma(\gamma + \sigma))^{-1}$
	\State $\bfu \assign \bfPhi \bfPhi[\bfx, \cdot]^\top$
	\State $\bfc \assign \bfc - \gamma^{-2} (\bfu \odot \bfu)$
	\State $\bfPhi \assign \bfPhi - \beta \bfu \bfPhi[\bfx, \cdot]$
\EndFunction\\
\Function{Next}{}
	\State \Return $\argmax_{\bfx \in \Xpool \setminus \Xbatch} \bfc[\bfx]$
\EndFunction
\end{algorithmic}
\end{algorithm}

\subsection{\Bait{}} \label{sec:appendix:bait}

In this section, we write $\Xtp \equalDef \Xtrain \cup \Xpool$.

\subsubsection{Connection Between Kernel and Feature Map Formulations} We can rewrite \Bait{}'s acquisition function from \eqref{eq:bait_acq} as
\begin{IEEEeqnarray*}{+rCl+x*}
a(\calX) & \equalDef & \sum_{\tbfx \in \Xtp} k_{\to\post{\calX, \sigma^2}}(\tbfx, \tbfx) \stackrel{\text{\eqref{eq:posterior_fm_other}}}{=} \sigma^2 \sum_{\tbfx \in \Xtp} \phi(\tbfx)^\top (\phi(\calX)^\top \phi(\calX) + \sigma^2 \bfI)^{-1} \phi(\tbfx) \\
& = & \sigma^2\sum_{\tbfx \in \Xtp} \tr\left(\phi(\tbfx)^\top (\phi(\calX)^\top \phi(\calX) + \sigma^2 \bfI)^{-1} \phi(\tbfx)\right) \\
& = & \sigma^2\sum_{\tbfx \in \Xtp} \tr\left((\phi(\calX)^\top \phi(\calX) + \sigma^2 \bfI)^{-1} \phi(\tbfx) \phi(\tbfx)^\top\right) \\
& = & \sigma^2\tr\left((\phi(\calX)^\top \phi(\calX) + \sigma^2 \bfI)^{-1} \phi(\Xtp)^\top \phi(\Xtp)\right) \\
& = & \sigma^2 \tr\left((\bfA_{\calX} + \sigma^2 \bfI)^{-1} \bfA_{\Xtp}\right)~,
\end{IEEEeqnarray*}
where
\begin{IEEEeqnarray*}{+rCl+x*}
\bfA_{\calX} \equalDef \phi(\calX)^\top \phi(\calX)~.
\end{IEEEeqnarray*}
The latter trace-based formulation corresponds to the formulation by \cite{ash_gone_2021}. 

\subsubsection{Forward Version} We will first derive an efficient implementation of \Bait{}-F in feature space, which builds on our derivation of \MaxDet{} in feature space and serves as a basis for \Bait{}-FB in feature space. As for the feature space version of \MaxDet{}, we choose to update the features using square roots of rank-1 updates. We will not derive a kernel space version of \Bait{} here since it appears that a kernel space version would not scale to large data sets. In the following, we will assume $\sigma^2 > 0$.

In a single iteration of the \Bait{}-F selection method, we want to find $\bfx \in \Xpool \setminus \Xbatch$ \emph{maximizing}
\begin{IEEEeqnarray*}{+rCl+x*}
&& a(\calX) - a(\calX \cup \{\bfx\}) \\
& \stackrel{\text{\Cref{prop:forward_update}}}{=} & \sum_{\tbfx \in \Xtp} \frac{k_{\to \post{\calX, \sigma^2}}(\tbfx, \bfx)^2}{k_{\to \post{\calX, \sigma^2}}(\bfx, \bfx) + \sigma^2} \\
& = & \frac{\phi_{\to \post{\calX, \sigma^2}}(\bfx)^\top \phi_{\to \post{\calX, \sigma^2}}(\Xtp)^\top \phi_{\to \post{\calX, \sigma^2}}(\Xtp) \phi_{\to \post{\calX, \sigma^2}}(\bfx)}{\phi_{\to \post{\calX, \sigma^2}}(\bfx)^\top \phi_{\to \post{\calX, \sigma^2}}(\bfx) + \sigma^2}~.
\end{IEEEeqnarray*}
While the inner product in the denominator of the last expression corresponds to the quantity $\bfc_{\calX}[\bfx]$ from the MaxDet feature space implementation, we still need a way to efficiently compute the numerator.
Therefore, in addition to the quantities $\bfPhi_{\calX}$ and $\bfc_{\calX}$ tracked in the feature-space implementation of \MaxDet{}, we track the following quantities for \Bait{}:
\begin{IEEEeqnarray*}{+rCl+x*}
\bfSigma_{\calX} & \equalDef & \phi_{\to \post{\calX, \sigma^2}}(\Xtp)^\top \phi_{\to \post{\calX, \sigma^2}}(\Xtp) \\
\bfv_{\calX}[\bfx] & \equalDef & \Phi_{\calX}[\bfx]^\top \bfSigma_{\calX} \Phi_{\calX}[\bfx]~.
\end{IEEEeqnarray*}
Using these quantities, we can write
\begin{IEEEeqnarray*}{+rCl+x*}
a(\calX) - a(\calX \cup \{\bfx\}) & = & \frac{\bfv_{\calX}[\bfx]}{\bfc_{\calX}[\bfx] + \sigma^2}~.
\end{IEEEeqnarray*}
Following \Cref{prop:forward_update}, we obtain the following update equations:
\begin{IEEEeqnarray*}{+rCl+x*}
\bfSigma_{\calX \cup \{\bfx\}} & = & (\bfI - \beta_{\calX}(\bfx) \bfPhi_{\calX}[\bfx] \bfPhi_{\calX}[\bfx]^\top) \bfSigma_{\calX} (\bfI - \beta_{\calX}(\bfx) \bfPhi_{\calX}[\bfx] \bfPhi_{\calX}[\bfx]^\top) \\
& = & \bfSigma_{\calX} - \beta_{\calX}(\bfx) \bfPhi_{\calX}[\bfx] \bfPhi_{\calX}[\bfx]^\top \bfSigma_{\calX} - \beta_{\calX}(\bfx) \bfSigma_{\calX} \bfPhi_{\calX}[\bfx] \bfPhi_{\calX}[\bfx]^\top \\
&& ~+~ \beta_{\calX}(\bfx)^2 \bfPhi_{\calX}[\bfx] \bfPhi_{\calX}[\bfx]^\top \bfSigma_{\calX} \bfPhi_{\calX}[\bfx] \bfPhi_{\calX}[\bfx]^\top \\
\bfv_{\calX \cup \{\bfx\}}[\bfx'] & = & \bfPhi_{\calX}[\bfx']^\top (\bfI - \gamma_{\calX}^{-2}(\bfx) \bfPhi_{\calX}[\bfx] \bfPhi_{\calX}[\bfx]^\top) \bfSigma_{\calX} (\bfI - \gamma_{\calX}(\bfx)^{-2} \bfPhi_{\calX}[\bfx] \bfPhi_{\calX}[\bfx]^\top) \bfPhi_{\calX}[\bfx'] \\
& = & \bfv_{\calX}[\bfx'] - 2\gamma_{\calX}(\bfx)^{-2} \bfPhi_{\calX}[\bfx']^\top \bfPhi_{\calX}[\bfx] \bfPhi_{\calX}[\bfx]^\top \bfSigma_{\calX} \bfPhi_{\calX}[\bfx'] \\
&& ~+~ \gamma_{\calX}(\bfx)^{-4} \bfPhi_{\calX}[\bfx']^\top \bfPhi_{\calX}[\bfx] \bfPhi_{\calX}[\bfx]^\top \bfSigma_{\calX} \bfPhi_{\calX}[\bfx] \bfPhi_{\calX}[\bfx]^\top \bfPhi_{\calX}[\bfx']~.
\end{IEEEeqnarray*}

This leads to the pseudocode in \Cref{alg:bait_f_features}. For the runtime analysis, we neglect the time for evaluating $\phi$ as usual. Then, the runtime of \textsc{Init} is $\BigO((\Ntrain+\Npool)\dfeat^2)$, the runtime of \textsc{Add} is $\BigO((\Ncand + \dfeat) \dfeat)$ and the runtime of \textsc{Next} is $\BigO(\Npool)$. Hence, the overall runtime of \Bait{}-F is $\BigO(\Ncand\Nsel\dfeat + (\Ntrain+\Npool)\dfeat^2)$. The memory complexity is $\BigO((\Ncand + \dfeat)\dfeat)$.

\begin{algorithm}[tb]
\caption{\Bait{}-F pseudocode implementation in feature space for given $\sigma^2 > 0$ using \Cref{alg:iterative_selection_detailed}.} \label{alg:bait_f_features}
\begin{algorithmic}
\Function{Init}{}
	\State $\bfPhi \assign \phi(\Xcand) \in \bbR^{\Ncand \times \dfeat}$ \Comment{feature matrix}
	\State $\bfc \assign (\langle \bfPhi[\bfx, \cdot]^\top, \bfPhi[\bfx, \cdot]^\top\rangle)_{\bfx \in \Xcand}$ \Comment{vector containing the kernel diagonal}
	\State $\bfSigma \assign \phi(\Xtp)^\top \phi(\Xtp)$ \Comment{Train and pool second moment matrix}
	\State $\bfv \assign (\langle \bfPhi[\bfx, \cdot]^\top, \bfSigma \bfPhi[\bfx, \cdot]^\top\rangle)_{\bfx \in \Xcand}$ \Comment{Numerator of the acquisition function}
\EndFunction\\
\Function{Add}{$\bfx$}
	\State $\gamma \assign \sqrt{\sigma^2 + \bfc[\bfx]}$
	\State $\beta \assign (\gamma(\gamma + \sigma))^{-1}$
	\State $\bfu \assign \bfPhi \bfPhi[\bfx, \cdot]^\top$
	\State $\tilde\bfu \assign \bfu \odot \bfu$

	\State $\bfw \assign \bfSigma \bfPhi[\bfx, \cdot]^\top$
	\State $\tilde v \assign \bfv[\bfx]$
	\State $\bfv \assign \bfv - 2\gamma^{-2} (\bfPhi \bfw) \odot \bfu + \gamma^{-4} \tilde v \tilde\bfu$
	
	\State $\bfA \assign \bfw \Phi[\bfx, \cdot]$
	\State $\bfSigma \assign \bfSigma - \beta (\bfA + \bfA^\top) + \beta^2 \tilde v \Phi[\bfx, \cdot]^\top \Phi[\bfx, \cdot]$
	
	\State $\bfc \assign \bfc - \gamma^{-2} \tilde\bfu$
	\State $\bfPhi \assign \bfPhi - \beta \bfu \bfPhi[\bfx, \cdot]$
\EndFunction\\
\Function{Next}{}
	\State \Return $\argmax_{\bfx \in \Xpool \setminus \Xbatch} \frac{\bfv[\bfx]}{\sigma^2 + \bfc[\bfx]}$
\EndFunction
\end{algorithmic}
\end{algorithm}

\subsubsection{Forward-Backward Version} To fit \Bait{}-FB into our framework, we first extend our iterative selection template to include a backward selection step. The extended template is shown in \Cref{alg:forward_backward_selection}. Here, we have an additional parameter $\Nextra$ specifying how many additional batch elements will be selected in the forward step and then removed in the backward step. Following \cite{ash_gone_2021}, we set $\Nextra \equalDef \min\{\Nbatch, \Npool-\Nbatch\}$ in our experiments.

\begin{algorithm}
\caption{Forward-backward selection algorithm template involving five customizable functions \textsc{Init}, \textsc{Add}, \textsc{Next}, \textsc{Remove}, \textsc{NextBackward} that are allowed to have side effects (i.e., read/write variables in \textsc{Select}).} \label{alg:forward_backward_selection}
\begin{algorithmic}
\Function{Select}{$k$, $\Xtrain$, $\Xpool$, $\Nbatch$, $\Nextra$, mode $\in \{$P, TP$\}$}
	\State $\Xmode \assign \Xtrain$ if mode = TP else $\emptyset$
	\State $\Xcand \equalDef \Xmode \cup \Xpool$
	\State $\Xbatch \assign \emptyset$
	\State \Call{Init}{}
	\For{$\bfx$ in $\Xmode$}
		\State \Call{Add}{$\bfx$}
	\EndFor
	\For{$i$ from $1$ to $\Nbatch+\Nextra$}
		\State $\bfx \assign $\Call{Next}{}  %
		\If{$\bfx \in \Xbatch \cup \Xtrain$ (failed selection)}
			\State ensure $|\Xbatch| = \Nbatch$ by removing the latest samples or filling up with random samples
			\State \Return $\Xbatch$
		\EndIf
		\State $\Xbatch \assign \Xbatch \cup \{\bfx\}$
		\State \Call{Add}{$\bfx$}
	\EndFor
	\For{$i$ from $1$ to $\Nextra$}
		\State $\bfx \assign $\Call{NextBackward}{}  %
		\If{$\bfx \notin \Xbatch \cup \Xtrain$ (failed selection)}
			\State ensure $|\Xbatch| = \Nbatch$ by removing the latest samples
			\State \Return $\Xbatch$
		\EndIf
		\State $\Xbatch \assign \Xbatch \setminus \{\bfx\}$
		\State \Call{Remove}{$\bfx$}
	\EndFor
	\State \Return $\Xbatch$
\EndFunction
\end{algorithmic}
\end{algorithm}

To implement \Bait{}-FB within \Cref{alg:forward_backward_selection}, we can reuse the methods \textsc{Init}, \textsc{Add} and \textsc{Next} from \Bait{}-F. In addition, we need to implement the methods \textsc{Remove} and \textsc{NextBackward} for the backward step. To this end, the following proposition shows how the kernel and feature map update when \emph{removing} a point $\bfx$ from the set $\calX$ of observed points:

\begin{proposition}[Backward update] \label{prop:backward_update}
For a kernel $k$ and $\sigma^2 > 0$, let $\tilde k \equalDef k_{\to \post{\{\bfx\}, \sigma^2}}$. Then,
\begin{IEEEeqnarray*}{+rCl+x*}
k(\bfx', \bfx'') = \tilde k(\bfx', \bfx'') + \tilde k(\bfx', \bfx) (\sigma^2 - \tilde k(\bfx, \bfx))^{-1} \tilde k(\bfx, \bfx'')~.
\end{IEEEeqnarray*}
Consequently, if $\tilde \phi$ is a feature map for $\tilde k$, then
\begin{IEEEeqnarray*}{+rCl+x*}
\phi(\bfx') \equalDef \left(\bfI + \frac{\tilde\phi(\bfx)\tilde\phi(\bfx)^\top}{\sigma^2 - \tilde\phi(\bfx)^\top \tilde\phi(\bfx)}\right)^{1/2} \tilde \phi(\bfx') = \left(\bfI + \tilde \beta \tilde\phi(\bfx)\tilde\phi(\bfx)^\top\right) \tilde \phi(\bfx')
\end{IEEEeqnarray*}
is a feature map for $k$, where
\begin{IEEEeqnarray*}{+rCl+x*}
\tilde \beta \equalDef \frac{1}{\sqrt{\sigma^2 - \tilde\phi(\bfx)^\top \tilde\phi(\bfx)}\left(\sqrt{\sigma^2 - \tilde\phi(\bfx)^\top \tilde\phi(\bfx)} + \sigma\right)}~.
\end{IEEEeqnarray*}

\begin{proof}
\textbf{Step 1: Finding $k(\bfx, \bfx)$.} We have
\begin{IEEEeqnarray*}{+rCl+x*}
\tilde k(\bfx, \bfx) = k(\bfx, \bfx) - \frac{k(\bfx, \bfx)^2}{k(\bfx, \bfx) + \sigma^2} = \frac{k(\bfx, \bfx)\sigma^2}{k(\bfx, \bfx) + \sigma^2} < \sigma^2~.
\end{IEEEeqnarray*}
Hence,
\begin{IEEEeqnarray*}{+rCl+x*}
\sigma^2 - \tilde k(\bfx, \bfx) = \frac{\sigma^2(k(\bfx, \bfx) + \sigma^2)}{k(\bfx, \bfx) + \sigma^2} - \frac{k(\bfx, \bfx)\sigma^2}{k(\bfx, \bfx) + \sigma^2} = \frac{\sigma^4}{k(\bfx, \bfx) + \sigma^2}~,
\end{IEEEeqnarray*}
which yields
\begin{IEEEeqnarray*}{+rCl+x*}
\frac{\sigma^2}{\sigma^2 - \tilde k(\bfx, \bfx)} = \frac{k(\bfx, \bfx) + \sigma^2}{\sigma^2}~.
\end{IEEEeqnarray*}

\textbf{Step 2: Finding $k(\bfx', \bfx)$.} Now, we compute
\begin{IEEEeqnarray*}{+rCl+x*}
\tilde k(\bfx', \bfx) = k(\bfx', \bfx) - k(\bfx', \bfx) \frac{k(\bfx, \bfx)}{k(\bfx, \bfx) + \sigma^2} = k(\bfx', \bfx) \frac{\sigma^2}{k(\bfx, \bfx) + \sigma^2}~,
\end{IEEEeqnarray*}
which yields
\begin{IEEEeqnarray*}{+rCl+x*}
k(\bfx', \bfx) = \tilde k(\bfx', \bfx) \frac{k(\bfx, \bfx) + \sigma^2}{\sigma^2} = \tilde k(\bfx', \bfx) \frac{\sigma^2}{\sigma^2 - \tilde k(\bfx, \bfx)}~.
\end{IEEEeqnarray*}

\textbf{Step 3: Finding $k(\bfx', \bfx'')$.} Finally, we have
\begin{IEEEeqnarray*}{+rCl+x*}
\tilde k(\bfx', \bfx'') & = & k(\bfx', \bfx'') - \frac{k(\bfx', \bfx) k(\bfx, \bfx'')}{k(\bfx, \bfx) + \sigma^2}~,
\end{IEEEeqnarray*}
which yields
\begin{IEEEeqnarray*}{+rCl+x*}
k(\bfx', \bfx'') & = & \tilde k(\bfx', \bfx'') + k(\bfx', \bfx) \cdot \frac{1}{k(\bfx, \bfx) + \sigma^2} \cdot k(\bfx, \bfx'') \\
& = & \tilde k(\bfx', \bfx'') + \tilde k(\bfx', \bfx) \frac{k(\bfx, \bfx) + \sigma^2}{\sigma^2} \cdot \frac{1}{k(\bfx, \bfx) + \sigma^2} \cdot \tilde k(\bfx, \bfx'')  \frac{\sigma^2}{\sigma^2 - \tilde k(\bfx, \bfx)} \\
& = & \tilde k(\bfx', \bfx'') + \frac{\tilde k(\bfx', \bfx) \tilde k(\bfx, \bfx'')}{\sigma^2 - \tilde k(\bfx, \bfx)}~.
\end{IEEEeqnarray*}

\textbf{Step 4: Feature map.} The specified feature map $\phi$ satisfies
\begin{IEEEeqnarray*}{+rCl+x*}
\phi(\bfx')^\top \phi(\bfx'') & = & \tilde \phi(\bfx')^\top \left(\bfI + \frac{\tilde\phi(\bfx)\tilde\phi(\bfx)^\top}{\sigma^2 - \tilde\phi(\bfx)^\top \tilde\phi(\bfx)}\right) \tilde \phi(\bfx'') \\
& = & \tilde k(\bfx', \bfx'') + \tilde k(\bfx', \bfx) (\sigma^2 - \tilde k(\bfx, \bfx))^{-1} \tilde k(\bfx, \bfx'') \\
& = & k(\bfx', \bfx'')~,
\end{IEEEeqnarray*}
hence it is a feature map for $k$.

\textbf{Step 5: Square root.} According to \Cref{lemma:rank_1_update_sqrt}, we have
\begin{IEEEeqnarray*}{+rCl+x*}
&& \left(\bfI + \frac{\tilde \phi(\bfx) \tilde\phi(\bfx)^\top}{\sigma^2 - \tilde\phi(\bfx)^\top \tilde\phi(\bfx)}\right)^{1/2} \\
& = & \bfI + \frac{1}{(\sigma^2 - \tilde\phi(\bfx)^\top \tilde\phi(\bfx))\left(1 + \sqrt{1 + \tilde\phi(\bfx)^\top \tilde\phi(\bfx)/(\sigma^2 - \tilde\phi(\bfx)^\top \tilde\phi(\bfx))}\right)} \tilde\phi(\bfx)\tilde\phi(\bfx)^\top \\
& = & \bfI + \frac{1}{(\sigma^2 - \tilde\phi(\bfx)^\top \tilde\phi(\bfx))(1 + \sigma(\sigma^2 - \tilde\phi(\bfx)^\top \tilde\phi(\bfx))^{-1/2})} \tilde\phi(\bfx)\tilde\phi(\bfx)^\top \\
& = & \bfI + \frac{1}{\sqrt{\sigma^2 - \tilde\phi(\bfx)^\top \tilde\phi(\bfx)}(\sqrt{\sigma^2 - \tilde\phi(\bfx)^\top \tilde\phi(\bfx)} + \sigma)} \tilde\phi(\bfx)\tilde\phi(\bfx)^\top~. & \qedhere
\end{IEEEeqnarray*}
\end{proof}
\end{proposition}

To formulate the backward update for our tracked quantities $\bfPhi, \bfc, \bfSigma$ and $\bfv$, we define for $\bfx \in \calX$ the scalars
\begin{IEEEeqnarray*}{+rCl+x*}
\tilde\gamma_{\calX}(\bfx) & \equalDef & \sqrt{\sigma^2 - \bfc[\bfx]} \\
\tilde\beta_{\calX}(\bfx) & \equalDef & \frac{1}{\gamma(\gamma + \sigma)}~.
\end{IEEEeqnarray*}
Then, using \Cref{prop:backward_update}, we can compute the backwards update as
\begin{IEEEeqnarray*}{+rCl+x*}
\bfPhi_{\calX \setminus \{\bfx\}}[\bfx'] & = & \bfPhi_{\calX}[\bfx'] + \tilde\beta_{\calX}(\bfx) \bfPhi_{\calX}[\bfx] \langle \bfPhi_{\calX}[\bfx], \bfPhi_{\calX}[\bfx']\rangle \\
\bfc_{\calX \setminus \{\bfx\}}[\bfx'] & = & \bfPhi_{\calX}[\bfx']^\top \left(\bfI + \frac{\bfPhi_{\calX}[\bfx]\bfPhi_{\calX}[\bfx]^\top}{\sigma^2 - \bfPhi_{\calX}[\bfx]^\top \bfPhi_{\calX}[\bfx]}\right) \bfPhi_{\calX}[\bfx'] \\
& = & \bfc_{\calX}[\bfx'] + \tilde\gamma_{\calX}(\bfx)^{-2} \langle \Phi_{\calX}[\bfx], \Phi_{\calX}[\bfx']\rangle^2 \\
\bfSigma_{\calX \setminus \{\bfx\}} & = & (\bfI - \beta_{\calX}(\bfx) \bfPhi_{\calX}[\bfx] \bfPhi_{\calX}[\bfx]^\top) \bfSigma_{\calX} (\bfI - \beta_{\calX}(\bfx) \bfPhi_{\calX}[\bfx] \bfPhi_{\calX}[\bfx]^\top) \\
& = & \bfSigma_{\calX} + \tilde\beta_{\calX}(\bfx) \bfPhi_{\calX}[\bfx] \bfPhi_{\calX}[\bfx]^\top \bfSigma_{\calX} + \tilde\beta_{\calX}(\bfx) \bfSigma_{\calX} \bfPhi_{\calX}[\bfx] \bfPhi_{\calX}[\bfx]^\top \\
&& ~+~ \beta_{\calX}(\bfx)^2 \bfPhi_{\calX}[\bfx] \bfPhi_{\calX}[\bfx]^\top \bfSigma_{\calX} \bfPhi_{\calX}[\bfx] \bfPhi_{\calX}[\bfx]^\top \\
\bfv_{\calX \setminus \{\bfx\}}[\bfx'] & = & \bfPhi_{\calX}[\bfx']^\top (\bfI + \tilde\gamma_{\calX}^{-2}(\bfx) \bfPhi_{\calX}[\bfx] \bfPhi_{\calX}[\bfx]^\top) \bfSigma_{\calX} (\bfI + \tilde\gamma_{\calX}(\bfx)^{-2} \bfPhi_{\calX}[\bfx] \bfPhi_{\calX}[\bfx]^\top) \bfPhi_{\calX}[\bfx'] \\
& = & \bfv_{\calX}[\bfx'] + 2\tilde\gamma_{\calX}(\bfx)^{-2} \bfPhi_{\calX}[\bfx']^\top \bfPhi_{\calX}[\bfx] \bfPhi_{\calX}[\bfx]^\top \bfSigma_{\calX} \bfPhi_{\calX}[\bfx'] \\
&& ~+~ \tilde\gamma_{\calX}(\bfx)^{-4} \bfPhi_{\calX}[\bfx']^\top \bfPhi_{\calX}[\bfx] \bfPhi_{\calX}[\bfx]^\top \bfSigma_{\calX} \bfPhi_{\calX}[\bfx] \bfPhi_{\calX}[\bfx]^\top \bfPhi_{\calX}[\bfx']~.
\end{IEEEeqnarray*}
For the backward step, we want to find $\bfx \in \calX$ \emph{minimizing}
\begin{IEEEeqnarray*}{+rCl+x*}
&& a(\calX \setminus \{\bfx\}) - a(\calX) \\
& \stackrel{\text{\Cref{prop:backward_update}}}{=} & \sum_{\tbfx \in \Xtp} \frac{k_{\to \post{\calX, \sigma^2}}(\tbfx, \bfx)^2}{\sigma^2 - k_{\to \post{\calX, \sigma^2}}(\bfx, \bfx)} \\
& = & \frac{\phi_{\to \post{\calX, \sigma^2}}(\bfx)^\top \phi_{\to \post{\calX, \sigma^2}}(\Xtp)^\top \phi_{\to \post{\calX, \sigma^2}}(\Xtp) \phi_{\to \post{\calX, \sigma^2}}(\bfx)}{\sigma^2 - \phi_{\to \post{\calX, \sigma^2}}(\bfx)^\top \phi_{\to \post{\calX, \sigma^2}}(\bfx)} \\
& = & \frac{\bfv_{\calX}[\bfx]}{\sigma^2 - \bfc_{\calX}[\bfx]}~.
\end{IEEEeqnarray*}

The corresponding implementation of \textsc{Remove} and \textsc{NextBackward} is given in \Cref{alg:bait_fb_features}, completing the implementation of \Bait{}-FB. The runtimes of \textsc{Remove} and \textsc{NextBackward} are equivalent to those of \textsc{Add} and \textsc{Next}, respectively, since the implementation is almost identical. Hence, the runtime complexity of \Bait{}-FB is given by $\BigO(\Ncand(\Nsel+2\Nextra)\dfeat + (\Ntrain+\Npool)\dfeat^2)$. The memory complexity is again $\BigO((\Ncand + \dfeat)\dfeat)$.

\begin{algorithm}[tb]
\caption{Functions \textsc{Remove} and \textsc{NextBackward} that, together with \Cref{alg:bait_f_features} and \Cref{alg:forward_backward_selection}, yield a pseudocode implementation of \Bait{}-FB in feature space for given $\sigma^2 > 0$.} \label{alg:bait_fb_features}
\begin{algorithmic}
\Function{Remove}{$\bfx$}
	\State $\tilde\gamma \assign \sqrt{\sigma^2 - \bfc[\bfx]}$
	\State $\tilde\beta \assign (\gamma(\gamma + \sigma))^{-1}$
	\State $\bfu \assign \bfPhi \bfPhi[\bfx, \cdot]^\top$
	\State $\tilde\bfu \assign \bfu \odot \bfu$

	\State $\bfw \assign \bfSigma \bfPhi[\bfx, \cdot]^\top$
	\State $\tilde v \assign \bfv[\bfx]$
	\State $\bfv \assign \bfv + 2\tilde\gamma^{-2} (\bfPhi \bfw) \odot \bfu + \tilde\gamma^{-4} \tilde v \tilde\bfu$
	
	\State $\bfA \assign \bfw \Phi[\bfx, \cdot]$
	\State $\bfSigma \assign \bfSigma + \tilde\beta (\bfA + \bfA^\top) + \tilde\beta^2 \tilde v \Phi[\bfx, \cdot]^\top \Phi[\bfx, \cdot]$
	
	\State $\bfc \assign \bfc + \tilde\gamma^{-2} \tilde\bfu$
	\State $\bfPhi \assign \bfPhi + \tilde\beta \bfu \bfPhi[\bfx, \cdot]$
\EndFunction\\
\Function{NextBackward}{}
	\State \Return $\argmin_{\bfx \in \Xbatch} \frac{\bfv[\bfx]}{\sigma^2 - \bfc[\bfx]}$
\EndFunction
\end{algorithmic}
\end{algorithm}

\subsection{\FrankWolfe{}} \label{sec:appendix:frankwolfe}

\cite{pinsler_bayesian_2019} proposed to apply the Frank-Wolfe constrained optimization algorithm \citep{frank_algorithm_1956} to the problem of sparsely approximating the empirical kernel mean embedding of $\Xpool$ with non-negative weights. Like \MaxDet{}, the resulting \FrankWolfe{} method allows for a kernel-space and a feature-space implementation. \cite{pinsler_bayesian_2019} used both versions in their experiments and presented the kernel-space version as pseudocode. \Cref{alg:frankwolfe_kernel} is an optimized adaptation of their kernel-space version to our framework. A difference between our version and theirs is that in our version, \textsc{Next} does not allow choosing a previously selected point. Hence, our version prevents the possibility of generating smaller batches by selecting the same point multiple times. Moreover, our version reduces the runtime complexity from $\BigO(\Ncand^2(T_k + \Nsel))$ to $\BigO(\Ncand^2 (T_k + 1))$ by reusing previously computed quantities in \textsc{Add}. The memory complexity of \Cref{alg:frankwolfe_kernel} is $\BigO(\Ncand^2)$, which can be reduced to $\BigO(\Ncand)$ by not storing the kernel matrix $\bfK$, at the cost of having to recompute some kernel values in \textsc{Add}.

\begin{algorithm}[tb]
\caption{\FrankWolfe{} pseudocode implementation in kernel space using \Cref{alg:iterative_selection_detailed}, following \cite{pinsler_bayesian_2019}.} \label{alg:frankwolfe_kernel}
\begin{algorithmic}
\Function{Init}{}
	\State $\bfK \assign (k(\bfx, \tbfx))_{\bfx, \tbfx \in \Xcand}$ \Comment{Corresponds to $\langle \calL_n, \calL_n \rangle$}
	\State $\bfc \assign (\sqrt{\bfK[\bfx, \bfx]})_{\bfx \in \Xcand}$ \Comment{Corresponds to $\sigma_n$}
	\State $r \assign \sum_{\bfx \in \Xcand} \bfc[\bfx]$ \Comment{Corresponds to $\sigma$}
	\State $\bfu \assign (\sum_{\tbfx\in\Xcand}\bfK[\bfx, \tbfx])_{\bfx \in \Xcand}$ \Comment{Corresponds to $\langle \calL, \calL_n\rangle$}
	\State $\bfv \assign (0)_{\bfx \in \Xcand}$ \Comment{Corresponds to $\langle \calL(\bfw), \calL_n \rangle$}
	\State $s \assign 0$ \Comment{Corresponds to $\langle \calL(\bfw), \calL(\bfw)\rangle$}
	\State $t \assign 0$ \Comment{Corresponds to $\langle \calL(\bfw), \calL\rangle$}
\EndFunction\\
\Function{Add}{$\bfx$}
	\State $\gamma \assign \frac{r\bfc[\bfx]^{-1}(\bfu[\bfx] - \bfv[\bfx]) + s-t}{r^2 - 2r\bfc[\bfx]^{-1}\bfv[\bfx] + s}$
	\State $s \assign (1-\gamma)^2 s + 2(1-\gamma)\gamma r\bfc[\bfx]^{-1} \bfv[\bfx] + \gamma^2 r^2$ %
	\State $t \assign (1-\gamma) t + \gamma r\bfc[\bfx]^{-1} \bfu[\bfx]$
	\State $\bfv \assign (1-\gamma)\bfv + \gamma r\bfc[\bfx]^{-1} \bfK[\bfx, \cdot]$
\EndFunction\\
\Function{Next}{}
	\State \Return $\argmax_{\bfx \in \Xpool \setminus \Xbatch} \bfc[\bfx]^{-1} (\bfu[\bfx] - \bfv[\bfx])$
\EndFunction
\end{algorithmic}
\end{algorithm}

The quadratic complexity in $\Ncand$ for the kernel-space version of \FrankWolfe{} shown in \Cref{alg:frankwolfe_kernel} makes it infeasible for large $\Npool$, such as in our experiments. However, \FrankWolfe{} can be realized much more efficiently in moderate-dimensional feature spaces, as shown in \Cref{alg:frankwolfe_features} and implemented in our code and (less efficiently) in the code of \cite{pinsler_bayesian_2019}. In \Cref{alg:frankwolfe_features}, when ignoring the computation of $\bfPhi$, the runtime complexity of \textsc{Init} is $\BigO(\Ncand\dfeat)$, the runtime complexity of \textsc{Add} is $\BigO(\dfeat)$, and the runtime of \textsc{Next} is $\BigO(\Npool \dfeat)$. In total, we obtain a runtime complexity of $\BigO((\Ncand + \Npool\Nbatch + \Nsel)\dfeat) = \BigO((\Ncand + \Npool\Nbatch)\dfeat)$. The memory complexity of \Cref{alg:frankwolfe_features} is $\BigO(\Ncand\dfeat)$.

\begin{algorithm}[tb]
\caption{\FrankWolfe{} pseudocode implementation in feature space using \Cref{alg:iterative_selection_detailed}.} \label{alg:frankwolfe_features}
\begin{algorithmic}
\Function{Init}{}
	\State $\bfPhi \assign \phi(\Xcand) \in \bbR^{\Ncand \times \dfeat}$ \Comment{Corresponds to $\calL_n$}
	\State $\bfc \assign (\|\bfPhi[\bfx, \cdot]\|_2)_{\bfx \in \Xcand}$ \Comment{Corresponds to $\sigma_n$}
	\State $r \assign \sum_{\bfx \in \Xcand} \bfc[\bfx]$ \Comment{Corresponds to $\sigma$}
	\State $\tilde\bfPhi \assign (\bfc[\bfx]^{-1} \bfPhi[\bfx, i])_{\bfx \in \Xcand, i \in \{1, \hdots, \dfeat\}} \in \bbR^{\Ncand \times \dfeat}$ \Comment{Corresponds to $\frac{1}{\sigma_n} \calL_n$}
	\State $\bfu \assign \sum_{\bfx \in \Xcand} \bfPhi[\bfx, \cdot]$ \Comment{Corresponds to $\calL$}
	\State $\bfv \assign \bfzero \in \bbR^{\dfeat}$ \Comment{Corresponds to $\calL(\bfw)$}
\EndFunction\\
\Function{Add}{$\bfx$}
	\State $\gamma \assign \frac{\langle r\tilde\bfPhi[\bfx, \cdot] - \bfv, \bfu - \bfv\rangle}{\langle r\tilde\bfPhi[\bfx, \cdot] - \bfv, r\tilde\bfPhi[\bfx, \cdot] - \bfv\rangle}$
	\State $\bfv \assign (1-\gamma)\bfv + \gamma r \tilde\bfPhi[\bfx, \cdot]$
\EndFunction\\
\Function{Next}{}
	\State \Return $\argmax_{\bfx \in \Xpool \setminus \Xbatch} \langle \tilde\bfPhi[\bfx, \cdot], \bfu - \bfv\rangle$
\EndFunction
\end{algorithmic}
\end{algorithm}

\subsection{\MaxDist{}} \label{sec:appendix:maxdist}

The \MaxDist{} selection method has been proposed various times in the literature under many different names. Up to the selection of the first two points, it is equivalent to the Kennard-Stone algorithm \citep{kennard_computer_1969} proposed for experimental design. \cite{rosenkrantz_analysis_1977} proposed it under the name \emph{farthest insertion} to generate an insertion order for constructing an approximate TSP solution. Later, \cite{gonzalez_clustering_1985} proposed it as an approximation algorithm for a clustering problem. In this context, it is also known as \emph{farthest-point clustering} \citep{bern_approximation_1996} or \emph{k-center greedy} \citep{sener_active_2018}. Moreover, it has been proposed as an initialization method for k-means clustering \citep{katsavounidis_new_1994}. \MaxDist{} is also equivalent to the \emph{geometric greedy} algorithm for kernel interpolation \citep{de_marchi_near-optimal_2005}. When used with $\Nbatch = \Npool$ to construct an ordering of the points, it is known as \emph{farthest-first traversal} or \emph{greedy permutation} of a finite metric space \citep{eppstein_approximate_2020}. %

\Cref{alg:maxdist} shows a pseudocode implementation of \MaxDist{}. Since \textsc{Add} has a runtime of $\BigO(\Npool(T_k + 1))$, the runtime of \Cref{alg:maxdist} is $\BigO(\Npool\Nsel(T_k + 1))$. The memory complexity is $\BigO(\Npool)$.

\begin{algorithm}[tb]
\caption{\MaxDist{} pseudocode implementation using \Cref{alg:iterative_selection_detailed}.} \label{alg:maxdist}
\begin{algorithmic}
\Function{Init}{}
	\State $\bfc \assign (k(\bfx, \bfx))_{\bfx \in \Xcand}$
	\State $\bfd \assign (\infty)_{\bfx \in \Xpool}$ \Comment{Minimum squared distances}
\EndFunction\\
\Function{Add}{$\bfx$}
	\State $\widetilde\bfd \assign (\bfc[\bfx] + \bfc[\tbfx] - 2k(\bfx, \tbfx))_{\tbfx \in \Xpool}$ \Comment{Compute squared kernel distances $d_k(\bfx, \tbfx)^2$}
	\State $\bfd \assign \min(\bfd, \widetilde\bfd)$ \Comment{element-wise minimum}
\EndFunction\\
\Function{Next}{}
	\If{no point has been added yet}
		\State \Return $\argmax_{\bfx \in \Xpool \setminus \Xbatch} \bfc[\bfx]$
	\EndIf
	\State \Return $\argmax_{\bfx \in \Xpool \setminus \Xbatch} \bfd[\bfx]$
\EndFunction
\end{algorithmic}
\end{algorithm}

We will now investigate approximation guarantees for \MaxDist{} with respect to a covering objective, called \emph{minmax radius clustering} or \emph{euclidean k-center problem} \citep{bern_approximation_1996}. Approximation guarantees can also be given for a related objective called \emph{minmax diameter clustering} \citep{gonzalez_clustering_1985, bern_approximation_1996}, which will not be discussed here. The following notation will help to define the minmax radius clustering problem:

\begin{definition}
For a given pseudometric $d$ (i.e., a metric except that $d(\bfx, \bfx') = 0$ is allowed for $\bfx \neq \bfx'$), batch size $\Nbatch \in \bbN$ and batch $\Xbatch \subseteq \Xpool$, we define
\begin{IEEEeqnarray*}{+rCl+x*}
\Delta_d(\Xbatch) & \equalDef & \max_{\bfx \in \Xpool} \min_{\bfx' \in \Xmode \cup \Xbatch} d(\bfx, \bfx')~, \\
\Delta_d^{\Nbatch} & \equalDef & \min_{\Xbatch \subseteq \Xpool, |\Xbatch| = \Nbatch} \Delta_d(\Xbatch)~. & \qedhere
\end{IEEEeqnarray*}
\end{definition}

The minmax radius clustering problem is defined as finding a batch $\Xbatch \subseteq \Xpool$ such that $\Delta_d(\Xbatch)$ is close to $\Delta_d^{\Nbatch}$. The following lemma asserts that \MaxDist{} yields a 2-approximation to this problem, which is in general (close to) the best possible approximation ratio for any polynomial-time algorithm unless P = NP \citep{feder_optimal_1988}.

\begin{lemma} \label{lemma:maxdist_suboptimality}
Let $\Xbatch$ be the batch selected by \MaxDist{} applied to $k$. Then,
\begin{IEEEeqnarray*}{+rCl+x*}
\Delta_{d_k}(\Xbatch) \leq 2\Delta_{d_k}^{\Nbatch}~.
\end{IEEEeqnarray*}

\begin{proof}
For $\Xmode = \emptyset$, this has been proven for example in \cite{bern_approximation_1996}. \cite{sener_active_2018} mentioned the result for general $\Xmode$ but it is unclear where this is proven. Therefore, we give a proof sketch here.

Let $d \equalDef d_k$. Let $D$ be the distance of the last selected point in $\Xbatch$ to the remaining points in $\Xbatch \cup \Xmode$. Then, $\Delta_d(\Xbatch) \leq D$, because otherwise another point with a larger distance would have been chosen instead. At the same time, all points in $\Xbatch$ are at least a distance of $D$ apart from any other point in $\Xbatch \cup \Xmode$, since otherwise the last point would have been chosen already in an earlier step. 

Now, consider a set $\tXbatch \subseteq \Xpool$ with $|\tXbatch| = \Nbatch$ such that $\Delta_d(\tXbatch) = \Delta_d^{\Nbatch}$. To derive a contradiction, assume $\Delta_d^{\Nbatch} < \Delta_d(\Xbatch)/2$. Then, for every $\bfx \in \Xbatch$, there must be $\tilde\bfx \in \Xmode \cup \tXbatch$ such that $d(\bfx, \tilde\bfx) < \Delta_d(\Xbatch)/2$. By our previous considerations, $\tilde\bfx$ cannot be in $\Xmode$, so it must be in $\tXbatch$. Moreover, because points in $\Xbatch$ are at least $D$ apart, no two of them can be closer than $\Delta_d(\Xbatch)/2$ to the same point in $\tXbatch$, hence by the pigeonhole principle every point in $\tXbatch$ must have a point in $\Xbatch$ that is closer to it than $\Delta_d(\Xbatch)/2$. Now, let $\bfx'$ be an arbitrary point in $\Xpool$. Then, there is $\tilde\bfx \in \Xmode \cup \tXbatch$ that is closer than $\Delta_d(\Xbatch)/2$ to $\bfx'$. Moreover, there is $\bfx$ in $\Xmode \cup \Xbatch$ that is closer than $\Delta_d(\Xbatch)/2$ to $\tilde\bfx$. Hence, the triangle inequality yields $d(\bfx', \bfx) < \Delta_d(\Xbatch)$, and since $\bfx'$ was arbitrary, this is a contradiction.
\end{proof}
\end{lemma}

The following simple result will be helpful to prove an approximation guarantee when using sketching:

\begin{lemma} \label{lemma:dist_ineqs}
Let $\Xpool$ be a finite set and let $\alpha > 0$. Moreover, let $d_1, d_2$ be pseudometrics on a set $\calX$ such that $d_1(\bfx, \tilde\bfx) \leq \alpha d_2(\bfx, \tilde\bfx)$ for all $\bfx, \tilde\bfx \in \Xpool$. Then, 
\begin{IEEEeqnarray*}{+rCl+x*}
\forall \Xbatch \subseteq \Xpool: \Delta_{d_1}(\Xbatch) & \leq & \alpha \Delta_{d_2}(\Xbatch)~, \\
\forall \Nbatch \in \{1, \hdots, |\Xpool|\}: \Delta_{d_1}^{\Nbatch} & \leq & \alpha \Delta_{d_2}^{\Nbatch}~.
\end{IEEEeqnarray*}

\begin{proof}
Let $\Xbatch \subseteq \Xpool$ and let $\bfx \in \Xpool$. For the element $\bfx'' \in \Xmode \cup \Xbatch$ minimizing $d_2(\bfx, \bfx'')$, we have
\begin{equation}
    \min_{\bfx' \in \Xmode \cup \Xbatch} d_1(\bfx, \bfx') \leq d_1(\bfx, \bfx'') \leq \alpha d_2(\bfx, \bfx'') = \alpha \min_{\bfx' \in \Xtrain \cup \Xbatch} d_2(\bfx, \bfx')~.
\end{equation}
Applying an analogous argument to $\bfx$ shows that $\Delta_{d_1}(\Xbatch) \leq \alpha \Delta_{d_2}(\Xbatch)$. Another application to $\Xbatch$ then shows that $\Delta_{d_1}^{\Nbatch} \leq \alpha \Delta_{d_2}^{\Nbatch}$.
\end{proof}
\end{lemma}

Finally, we obtain the following approximation guarantee with sketching:

\begin{restatable}{theorem}{thmmaxdistrp}
Suppose that \eqref{eq:eps_isometry} holds. Then, the batch $\Xbatch$ with size $\Nbatch$ computed by \MaxDist{} applied to $k_{\to\rp{p}}$ satisfies
\begin{IEEEeqnarray*}{+rCl+x*}
\Delta_{d_k}(\Xbatch) \leq 2\frac{1+\varepsilon}{1-\varepsilon} \Delta_{d_k}^{\Nbatch}~.
\end{IEEEeqnarray*}
\end{restatable}

\begin{proof}
Let $d_1 \equalDef d_k$ and $d_2 \equalDef d_{k_{\to\rp{p}}}$. Then, 
\begin{IEEEeqnarray*}{+rCl+x*}
    \Delta_{d_1}(\Xbatch) & \stackrel{\text{\Cref{lemma:dist_ineqs}}}{\leq} & \frac{1}{1-\varepsilon} \Delta_{d_2}(\Xbatch) \stackrel{\text{\Cref{lemma:maxdist_suboptimality}}}{\leq} 2\frac{1}{1-\varepsilon} \Delta_{d_2}^{\Nbatch} \\
    & \stackrel{\text{\Cref{lemma:dist_ineqs}}}{\leq} & 2\frac{1+\varepsilon}{1-\varepsilon} \Delta_{d_1}^{\Nbatch}~. & \qedhere
\end{IEEEeqnarray*}
\end{proof}

\subsection{\KMeansPP{}} \label{sec:appendix:kmeanspp}

\Cref{alg:kmeanspp} shows pseudocode for \KMeansPP{}. Like for \MaxDist{}, we obtain a runtime complexity of $\BigO(\Npool\Nsel(T_k + 1))$ and a memory complexity of $\BigO(\Npool)$.

\begin{algorithm}[tb]
\caption{\KMeansPP{} pseudocode implementation using \Cref{alg:iterative_selection_detailed}.} \label{alg:kmeanspp}
\begin{algorithmic}
\Function{Init}{}
	\State $\bfc \assign (k(\bfx, \bfx))_{\bfx \in \Xcand}$
	\State $\bfd \assign (\infty)_{\bfx \in \Xpool}$ \Comment{Minimum squared distances}
\EndFunction\\
\Function{Add}{$\bfx$}
	\State $\widetilde\bfd \assign (\bfc[\bfx] + \bfc[\tbfx] - 2k(\bfx, \tbfx))_{\tbfx \in \Xpool}$ \Comment{Compute squared kernel distances $d_k(\bfx, \tbfx)^2$}
	\State $\bfd \assign \min(\bfd, \widetilde\bfd)$ \Comment{element-wise minimum}
\EndFunction\\
\Function{Next}{}
	\If{no point has been added yet}
		\State \Return uniform random sample from $\Xpool \setminus \Xbatch$
	\EndIf
	\State \Return sample $\bfx \in \Xpool \setminus \Xbatch$ with probability proportional to $\bfd[\bfx]$
\EndFunction
\end{algorithmic}
\end{algorithm}

\subsection{\LCMD{}} \label{sec:appendix:lcmd}

A pseudocode implementation of our newly proposed \LCMD{} method is shown in \Cref{alg:lcmd}. Here, the \textsc{Add} method has a runtime complexity of $\BigO(\Npool(T_k + 1))$ and the \textsc{Next} method has a runtime complexity of $\BigO(\Npool + \Nsel)$. The overall runtime complexity is therefore $\BigO(\Npool\Nsel(T_k + 1) + \Nbatch(\Npool + \Nsel)) = \BigO(\Npool\Nsel(T_k + 1))$. The memory complexity of \Cref{alg:lcmd} is $\BigO(\Ncand)$.

\begin{algorithm}[tb]
\caption{\LCMD{} pseudocode implementation using \Cref{alg:iterative_selection_detailed}.} \label{alg:lcmd}
\begin{algorithmic}
\Function{Init}{}
	\State $\bfc \assign (k(\bfx, \bfx))_{\bfx \in \Xcand}$
	\State $\bfd \assign (\infty)_{\bfx \in \Xpool}$ \Comment{Minimum squared distances}
	\State $\bfv \assign (0)_{\bfx \in \Xpool}$ \Comment{Associated cluster centers; dummy initialization}
\EndFunction\\
\Function{Add}{$\bfx$}
	\State $\widetilde\bfd \assign (\bfc[\bfx] + \bfc[\tbfx] - 2k(\bfx, \tbfx))_{\tbfx \in \Xpool}$ \Comment{Compute squared kernel distances $d_k(\bfx, \tbfx)^2$}
	\State $\bfv \assign (\bfv[\tbfx]$ if $\bfd[\tbfx] \leq \widetilde\bfd[\tbfx]$ else $\bfx)_{\tbfx \in \Xpool}$ \Comment{Update associated cluster centers}
	\State $\bfd \assign \min(\bfd, \widetilde\bfd)$ \Comment{element-wise minimum}
\EndFunction\\
\Function{Next}{}
	\If{no point has been added yet}
		\State \Return $\argmax_{\bfx \in \Xpool \setminus \Xbatch} \bfc[\bfx]$
	\EndIf
	\State $\bfs \assign (\sum_{\tbfx \in \Xpool \setminus \Xbatch: \bfv[\tbfx] = \bfx} \bfd[\tbfx])_{\bfx \in \Xmode \cup \Xbatch}$ \Comment{Compute cluster sizes}
	\State $s_{\mathrm{max}} = \max_{\bfx \in \Xmode \cup \Xbatch} \bfs[\bfx]$ \Comment{Maximum cluster size}
	\State \Return $\argmax_{\bfx \in \Xpool \setminus \Xbatch: \bfs[\bfx] = s_{\mathrm{max}}} \bfd[\bfx]$
\EndFunction
\end{algorithmic}
\end{algorithm}

\section{Details on Experiments} \label{sec:appendix:experiments}

In the following, we provide a more detailed description of our experimental setup and our results. All NN computations were performed with 32-bit floating-point precision, but we switched to 64-bit floating-point precision whenever posterior transformations (computations involving $\sigma^2$) were involved. All \MaxDet{} computations used the kernel-space implementation and all \FrankWolfe{} computations used the feature-space implementation. All experiments were run on a workstation with four NVIDIA RTX 3090 GPUs and an AMD Ryzen Threadripper PRO 3975WX CPU with 256 GB RAM.

\subsection{Data Sets} \label{sec:appendix:data_sets}

We selected 15 tabular regression data sets from different sources, roughly using the following criteria:
\begin{enumerate}[(a)]
\item The data set should be sufficiently large after removing rows with missing values (at least 40000 samples).
\item The data set should be in a format suitable to perform regression, e.g., not consist of a few long time series.
\item The data set should not have too many categorical or text columns with many categories.
\item The test RMSE for randomly sampled training sets should drop substantially when going from $\ntrain=256$ to $\ntrain = 17\cdot 256$. In our case, the selected 15 data sets differ from the other tested data sets in that the RMSE dropped at least by 14\%. %
With this criterion, we want to exclude data sets that would not significantly affect the benchmark results, e.g., because they are too easy to learn or because they are too noisy.
\end{enumerate}

An overview of the selected data sets can be found in \Cref{table:data_set_characteristics} and \Cref{table:data_sets_source}. Our main data sources are the UCI and OpenML repositories \citep{dua_uci_2017, vanschoren_openml_2013}. The sgemm and ct\_slices data sets have also been used by \cite{tsymbalov_dropout-based_2018}. In contrast to \cite{tsymbalov_dropout-based_2018}, we use the undirected and not the directed version of the kegg data set since it contains more samples and the RMSE drops more strongly between $\ntrain=256$ and $\ntrain = 17\cdot 256$. Concerning the other four data sets used in \cite{tsymbalov_dropout-based_2018}, we omitted the BlogFeedback, YearPredictionMSD, and Online News Popularity data sets due to criterion (d), and could not find the Rosenbrock 2000D data set online. Although the poker data set is originally a multi-class classification data set, we include it since it is noise-free and sufficiently difficult to learn. %

Our accompanying code allows automatically downloading and processing all data sets.

\subsection{Preprocessing} We preprocess the data sets in the following way: On some data sets, we remove unwanted columns such as identifier columns, see the details in our code. We then remove rows with NaN (missing) values. Next, if necessary, we randomly subsample the data set such that it contains at most 500000 samples. We use 80\% of the data set, but at most 200000 samples, for training, validation, and pool data. Of those, we initially use $\ntrain = 256$ and $\nvalid = 1024$ and reserve the rest for the pool set. Subsequently, we remove columns with only a single value. We one-hot encode small categorical columns, allowing at most 300 new continuous columns, and discard larger categorical columns. On some data sets, we transform the labels $y$, for example by applying a logarithm or by taking the median of multiple target values, we refer to our code for further details. We standardize the labels $y$ such that they have mean $0$ and variance $1$.\footnote{This only leaks a negligible amount of information from the test set and in turn allows us to better compare the errors across data sets.} We preprocess the inputs $\bfx \in \bbR^d$ as
\begin{IEEEeqnarray*}{+rCl+x*}
x_j^{\mathrm{processed}} & \equalDef & 5\tanh\left(\frac{1}{5} \cdot \frac{x_j - \hat\mu_j}{\hat\sigma_j}\right)~,
\end{IEEEeqnarray*}
where we compute
\begin{IEEEeqnarray*}{+rCl+x*}
\hat\mu_j & \equalDef & \frac{1}{\Ntrain + \Npool}\sum_{\bfx \in \Xtrain \cup \Xpool} x_j \\
\hat\sigma_j^2 & \equalDef & \frac{1}{\Ntrain + \Npool}\sum_{\bfx \in \Xtrain \cup \Xpool} (x_j - \hat\mu_j)^2~.
\end{IEEEeqnarray*}
The motivation for the $\tanh$ function is to reduce the impact of outliers by soft-clipping the coordinates to the interval $(-5, 5)$.

For the sarcos data set, we only used the training data since the test data on the GPML web page (see \Cref{table:data_sets_source}) is already contained in the training data. For the poker data set, we only used the test data since it contains roughly a million samples, while the training set contains around 25000 samples.

The resulting characteristics of the processed data sets can be found in \Cref{table:data_set_characteristics}. The full names, links, and citations are contained in \Cref{table:data_sets_source}.

\begin{table}[tb]
\centering
\begin{tabular}{ccccc}
Short name & Initial pool set size & Test set size & Number of features \\
\hline
sgemm & 192000 & 48320 & 14 \\
wec\_sydney & 56320 & 14400 & 48 \\
ct\_slices & 41520 & 10700 & 379 \\
kegg\_undir & 50407 & 12921 & 27 \\
online\_video & 53748 & 13756 & 26 \\
query & 158720 & 40000 & 4 \\
poker & 198720 & 300000 & 95 \\
road & 198720 & 234874 & 2 \\
mlr\_knn\_rng & 88123 & 22350 & 132 \\
fried & 31335 & 8153 & 10 \\
diamonds & 41872 & 10788 & 29 \\
methane & 198720 & 300000 & 33 \\
stock & 45960 & 11809 & 9 \\
protein & 35304 & 9146 & 9 \\
sarcos & 34308 & 8896 & 21
\end{tabular}
\caption{Data set characteristics.} \label{table:data_set_characteristics}
\end{table}

\begin{table}[tb]
\centering
\renewcommand{\arraystretch}{1.3}
\small
\begin{tabular}{cccC{0.25\textwidth}C{0.24\textwidth}}
Short name & Source & OpenML ID & Full name & Citation \\
\hline
sgemm & \href{https://archive.ics.uci.edu/ml/datasets/SGEMM+GPU+kernel+performance}{UCI} & & SGEMM GPU kernel performance & \cite{ballester-ripoll_sobol_2019} \\
wec\_sydney & \href{https://archive.ics.uci.edu/ml/datasets/Wave+Energy+Converters}{UCI} & & Wave Energy Converters & \cite{neshat_detailed_2018}\\
ct\_slices & \href{https://archive.ics.uci.edu/ml/datasets/Relative+location+of+CT+slices+on+axial+axis}{UCI} & & Relative location of CT slices on axial axis & \cite{graf_2d_2011} \\
kegg\_undir & \href{https://archive.ics.uci.edu/ml/datasets/KEGG+Metabolic+Reaction+Network+\%28Undirected\%29}{UCI} & & KEGG Metabolic Reaction Network (Undirected) & \cite{shannon_cytoscape_2003}\\
online\_video & \href{https://archive.ics.uci.edu/ml/datasets/Online+Video+Characteristics+and+Transcoding+Time+Dataset}UCI & & Online Video Characteristics and Transcoding Time & \cite{deneke_video_2014} \\
query & \href{https://archive.ics.uci.edu/ml/datasets/Query+Analytics+Workloads+Dataset}{UCI} & & Query Analytics Workloads & \cite{anagnostopoulos_scalable_2018}, \cite{savva_explaining_2018} \\
poker & \href{https://archive.ics.uci.edu/ml/datasets/Poker+Hand}{UCI} & & Poker Hand & --- \\
road & \href{https://archive.ics.uci.edu/ml/datasets/3D+Road+Network+\%28North+Jutland\%2C+Denmark\%29}{UCI} & & 3D Road Network (North Jutland, Denmark) & \cite{kaul_building_2013} \\
mlr\_knn\_rng & \href{https://www.openml.org/d/42454}{OpenML} & 42454 & mlr\_knn\_rng & --- \\
fried & \href{https://www.openml.org/d/564}{OpenML} & 564 & fried & \cite{friedman_multivariate_1991} \\
diamonds & \href{https://www.openml.org/d/42225}{OpenML} & 42225 & diamonds & --- \\
methane & \href{https://www.openml.org/d/42701}{OpenML} & 42701 & Methane & \cite{slezak_framework_2018} \\
stock & \href{https://www.openml.org/d/1200}{OpenML} & 1200 & BNG(stock) & --- \\
protein & \href{https://www.openml.org/d/42903}{OpenML} & 42903 & physicochemical-protein & --- \\
sarcos & \href{http://www.gaussianprocess.org/gpml/data/}{GPML} & & SARCOS data & \cite{vijayakumar_locally_2000}
\end{tabular}
\caption{Overview of used data sets. The second column entries are hyperlinks to the respective web pages.} \label{table:data_sets_source}
\end{table}

\subsection{Neural Network Configuration} \label{sec:appendix:nn_config}

We use a fully-connected NN with two hidden layers with 512 neurons each ($L=3$, $d_1 = d_2 = 512$). We employ the neural tangent parametrization as discussed in \Cref{sec:reg_fcnn} with the ReLU activation function. We initialize biases to zero and weights i.i.d.\ from $\calN(0, 1)$. 
For optimization, we use the Adam \citep{kingma_adam_2015} optimizer with its default parameters $\beta_1 = 0.9, \beta_2 = 0.999$, and let the learning rate (see below) decay linearly to zero over training. We use a mini-batch size of $256$ and train for $256$ epochs. 
After each epoch, we measure the validation RMSE on a validation set with 1024 samples. After training, we set the trained model parameters $\bftheta_T$ to the parameters from the end of the epoch where the lowest validation RMSE was attained. While the use of a large validation set might not be realistic for many data-scarce BMAL scenarios, we see this as a simple proxy for more complicated cross-validation or refitting strategies. 

While the reasoning of forward variance preservation as in the well-known Kaiming initialization \citep{he_delving_2015} suggests to set $\sigma_w = \sqrt{2}$, we find that smaller values of $\sigma_w$ can substantially improve the RMSE of the trained models. Possible explanations for this phenomenon might be that large $\sigma_w$ increases the scale of the disturbance by the random initial function of the network \citep{nonnenmacher_which_2021} or brings the NN more towards a \quot{lazy training} regime \citep{chizat_lazy_2019}. Therefore, we manually tuned $\sigma_w, \sigma_b$ and the initial learning rate to optimize the mean log RMSE across all data sets for \Random{} selection. We arrived at $\sigma_w = \sigma_b = 0.2$ and an initial learning rate of $0.375$. 

To assess whether our insights apply to other NN configurations, we also run experiments for a fully-connected NN with the SiLU (a.k.a.\ Swish) activation function \citep{elfwing_sigmoid-weighted_2018}. Again, we use optimized hyperparameters for SiLU, specifically an initial learning rate of $0.15$ as well as $\sigma_w = 0.5, \sigma_b = 1.0$. 

\subsection{Results} \label{sec:appendix:results}

\Cref{table:all_algs} shows averaged logarithmic error metrics and runtimes for a wide variety of configurations. Some conclusions from these results are discussed in \Cref{sec:bmdal:experiments}. \Cref{table:all_algs_silu} shows analogous results for our NN configuration with the SiLU instead of the ReLU activation function. Note that the results for $k_{\mathrm{nngp}}$ still use the NNGP for ReLU, though, since we do not know of an analytic expression of the NNGP for SiLU. Results on individual data sets for selected methods are shown in \Cref{table:table_data_sets} and \Cref{table:table_data_sets_lasterror}.

In this section, we also provide more plots complementing the figures from the main part of the paper. \Cref{fig:batch_sizes_individual_rmse} shows batch size plots on individual data sets for RMSE and \Cref{fig:learning_curves_individual_q99} shows learning curve plots on individual data sets for the 99\% quantile. Moreover, \Cref{fig:correlation_plot_rmse} and \Cref{fig:correlation_plot_maxe} allow comparing two methods across data sets on RMSE and MAXE, respectively.

The estimated standard deviations of the mean estimators in \Cref{fig:existing_algs}, \Cref{fig:learning_curves}, \Cref{fig:learning_curves_individual_rmse}, \Cref{fig:batch_sizes}, \Cref{fig:batch_sizes_individual_rmse} and \Cref{fig:learning_curves_individual_q99} are computed as follows: Consider random variables $X_{ij}$ representing the log metric values on repetition $i$ and data set $j$. Then, it is well-known that for the mean estimator $\hat\mu_j \equalDef \frac{1}{20} \sum_{i=1}^{20} X_{ij}$, an unbiased estimator of its variance is given by
\begin{IEEEeqnarray*}{+rCl+x*}
\hat \sigma^2_j \equalDef \frac{1}{20-1} \sum_{i=1}^{20} (X_{ij} - \hat \mu_j)^2
\end{IEEEeqnarray*} 
Since all mean estimators $\hat\mu_j$ are independent, the variance of the total mean estimator $\hat\mu \equalDef \frac{1}{15} \sum_{j=1}^{15} \hat\mu_j$ can be estimated as
\begin{IEEEeqnarray*}{+rCl+x*}
\hat \sigma^2 \equalDef \frac{1}{15^2} \sum_{j=1}^{15} \hat\sigma^2_j~.
\end{IEEEeqnarray*}
Our plots hence show $\hat\sigma$ as the estimated standard deviation of the mean estimator $\hat\mu$.

\begin{table}[p]
\centering
\ssmall
\setlength{\tabcolsep}{2pt}
\begin{tabular}{ccccccccc}
Data set & \textsc{Random} & \textsc{MaxDiag} & \textsc{MaxDet}-P & \textsc{Bait-F}-P & \textsc{FrankWolfe}-P & \textsc{MaxDist}-P & \textsc{KMeansPP}-P & \textsc{LCMD}-TP (ours)\\
\hline
ct\_slices & 0.141 & 0.123 & 0.085 & 0.076 & 0.088 & 0.085 & 0.081 & \textbf{0.072} \\
diamonds & 0.173 & 0.169 & 0.166 & \textbf{0.161} & 0.162 & 0.166 & 0.162 & \textbf{0.161} \\
fried & 0.230 & 0.231 & 0.228 & \textbf{0.227} & 0.228 & 0.229 & 0.229 & 0.228 \\
kegg\_undir & 0.380 & 0.346 & 0.245 & 0.222 & 0.248 & 0.243 & 0.227 & \textbf{0.220} \\
methane & 0.733 & 0.770 & 0.736 & 0.714 & 0.714 & 0.740 & 0.713 & \textbf{0.708} \\
mlr\_knn\_rng & 0.294 & 0.326 & 0.211 & 0.183 & 0.199 & 0.209 & 0.190 & \textbf{0.176} \\
online\_video & 0.263 & 0.190 & 0.159 & \textbf{0.149} & 0.159 & 0.158 & 0.153 & 0.152 \\
poker & 0.806 & 0.803 & \textbf{0.742} & 0.751 & 0.797 & 0.754 & 0.794 & 0.797 \\
protein & 0.763 & 0.793 & 0.782 & \textbf{0.757} & 0.761 & 0.782 & \textbf{0.757} & 0.759 \\
query & 0.058 & 0.082 & 0.066 & 0.057 & 0.060 & 0.066 & 0.057 & \textbf{0.053} \\
road & \textbf{0.586} & 0.702 & 0.625 & 0.592 & 0.606 & 0.624 & 0.598 & 0.591 \\
sarcos & 0.181 & 0.190 & 0.176 & 0.164 & 0.168 & 0.176 & 0.167 & \textbf{0.163} \\
sgemm & 0.152 & 0.185 & 0.155 & 0.142 & 0.144 & 0.153 & 0.144 & \textbf{0.140} \\
stock & 0.531 & 0.541 & 0.540 & 0.529 & 0.527 & 0.540 & \textbf{0.526} & 0.528 \\
wec\_sydney & 0.027 & 0.034 & 0.030 & \textbf{0.025} & 0.027 & 0.030 & 0.026 & 0.028
\end{tabular}
\caption{This table shows the averaged (non-logarithmic) RMSEs per data set, averaged over all repetitions, and BMAL steps, for each of the selection methods with kernels as in \Cref{table:selected_kernels}.} \label{table:table_data_sets}
\end{table}

\begin{table}[p]
\centering
\ssmall
\setlength{\tabcolsep}{2pt}
\begin{tabular}{ccccccccc}
Data set & \textsc{Random} & \textsc{MaxDiag} & \textsc{MaxDet}-P & \textsc{Bait-F}-P & \textsc{FrankWolfe}-P & \textsc{MaxDist}-P & \textsc{KMeansPP}-P & \textsc{LCMD}-TP (ours)\\
\hline
ct\_slices & 0.093 & 0.069 & 0.053 & 0.044 & 0.050 & 0.052 & 0.046 & \textbf{0.038} \\
diamonds & 0.166 & 0.155 & 0.153 & \textbf{0.152} & \textbf{0.152} & 0.153 & 0.153 & \textbf{0.152} \\
fried & 0.221 & 0.219 & \textbf{0.218} & 0.219 & 0.219 & 0.219 & 0.220 & 0.219 \\
kegg\_undir & 0.291 & 0.225 & 0.180 & 0.165 & 0.173 & 0.179 & 0.166 & \textbf{0.154} \\
methane & 0.692 & 0.731 & 0.704 & 0.675 & 0.670 & 0.709 & 0.669 & \textbf{0.661} \\
mlr\_knn\_rng & 0.208 & 0.168 & 0.126 & 0.112 & 0.119 & 0.131 & 0.108 & \textbf{0.105} \\
online\_video & 0.206 & 0.137 & 0.116 & 0.106 & 0.111 & 0.115 & 0.109 & \textbf{0.105} \\
poker & 0.600 & 0.598 & \textbf{0.533} & 0.537 & 0.578 & 0.547 & 0.558 & 0.589 \\
protein & 0.727 & 0.758 & 0.751 & 0.721 & 0.729 & 0.750 & \textbf{0.717} & 0.721 \\
query & 0.040 & 0.061 & 0.054 & 0.042 & 0.046 & 0.052 & 0.040 & \textbf{0.037} \\
road & \textbf{0.538} & 0.671 & 0.582 & 0.544 & 0.571 & 0.570 & 0.549 & 0.551 \\
sarcos & 0.155 & 0.162 & 0.155 & 0.140 & 0.143 & 0.154 & 0.142 & \textbf{0.138} \\
sgemm & 0.100 & 0.126 & 0.107 & 0.098 & 0.097 & 0.106 & 0.096 & \textbf{0.092} \\
stock & 0.511 & 0.517 & 0.518 & 0.508 & \textbf{0.506} & 0.519 & \textbf{0.506} & \textbf{0.506} \\
wec\_sydney & 0.021 & 0.023 & 0.021 & \textbf{0.020} & 0.021 & 0.021 & \textbf{0.020} & \textbf{0.020}
\end{tabular}
\caption{This table shows the (non-logarithmic) RMSEs after the last BMAL step per data set, averaged over all repetitions, for each of the selection methods with kernels as in \Cref{table:selected_kernels}.} \label{table:table_data_sets_lasterror}
\end{table}

\begin{table}[tbp]
\centering
\fontsize{4}{5}\selectfont
\begin{tabular}{cccccccc}
Selection method & Kernel & MAE & RMSE & 95\% & 99\% & MAXE & avg.\ time [$s$]\\
\hline
\textsc{Random} & --- & -1.934 & -1.401 & -0.766 & -0.163 & 1.107 & 0.001 \\
\hline
\textsc{MaxDiag} & $k_{\mathrm{grad} \to \operatorname{sketch}(512) \to \operatorname{acs-rf}(512)}$ & -1.777 & -1.370 & -0.690 & -0.189 & 0.978 & 0.650\\
\textsc{MaxDiag} & $k_{\mathrm{grad} \to \operatorname{sketch}(512) \to \mathcal{X}_{\operatorname{train}}}$ & -1.777 & -1.369 & -0.690 & -0.186 & 0.986 & 0.551\\
\textsc{MaxDiag} & $k_{\mathrm{grad} \to \operatorname{ens}(3) \to \operatorname{sketch}(512) \to \mathcal{X}_{\operatorname{train}}}$ & -1.768 & -1.366 & -0.685 & -0.188 & 0.970 & 1.392\\
\textsc{MaxDiag} & $k_{\mathrm{grad} \to \mathcal{X}_{\operatorname{train}}}$ & -1.766 & -1.355 & -0.675 & -0.168 & 0.996 & 4.108\\
\textsc{MaxDiag} & $k_{\mathrm{grad} \to \operatorname{sketch}(512) \to \operatorname{acs-grad}}$ & -1.751 & -1.345 & -0.664 & -0.163 & 1.007 & 0.553\\
\textsc{MaxDiag} & $k_{\mathrm{grad} \to \operatorname{sketch}(512) \to \operatorname{acs-rf-hyper}(512)}$ & -1.743 & -1.334 & -0.656 & -0.148 & 1.021 & 0.650\\
\textsc{MaxDiag} & $k_{\mathrm{ll} \to \operatorname{acs-rf}(512)}$ & -1.713 & -1.286 & -0.606 & -0.084 & 1.052 & 0.030\\
\textsc{MaxDiag} & $k_{\mathrm{ll} \to \mathcal{X}_{\operatorname{train}}}$ & -1.722 & -1.285 & -0.614 & -0.077 & 1.080 & 0.142\\
\textsc{MaxDiag} & $k_{\mathrm{nngp} \to \mathcal{X}_{\operatorname{train}}}$ & -1.754 & -1.272 & -0.606 & -0.044 & 1.088 & 2.592\\
\textsc{MaxDiag} & $k_{\mathrm{lin} \to \mathcal{X}_{\operatorname{train}}}$ & -1.585 & -1.103 & -0.426 & 0.143 & 1.222 & 0.007 \\
\hline
\textsc{MaxDet}-TP & $k_{\mathrm{grad} \to \operatorname{scale}(\mathcal{X}_{\operatorname{train}})}$ & -1.930 & -1.522 & -0.855 & -0.356 & 0.856 & 8.883\\
\textsc{MaxDet}-P & $k_{\mathrm{grad} \to \operatorname{sketch}(512) \to \mathcal{X}_{\operatorname{train}}}$ & -1.915 & -1.512 & -0.844 & -0.350 & 0.867 & 0.770\\
\textsc{MaxDet}-P & $k_{\mathrm{ll} \to \operatorname{ens}(3) \to \operatorname{sketch}(512) \to \mathcal{X}_{\operatorname{train}}}$ & -1.931 & -1.504 & -0.849 & -0.337 & 0.873 & 0.673\\
\textsc{MaxDet}-P & $k_{\mathrm{grad} \to \operatorname{ens}(3) \to \operatorname{sketch}(512) \to \mathcal{X}_{\operatorname{train}}}$ & -1.895 & -1.500 & -0.830 & -0.344 & 0.868 & 1.608\\
\textsc{MaxDet}-P & $k_{\mathrm{grad} \to \operatorname{sketch}(512) \to \operatorname{acs-rf}(512)}$ & -1.893 & -1.492 & -0.820 & -0.323 & 0.886 & 0.869\\
\textsc{MaxDet}-P & $k_{\mathrm{grad} \to \operatorname{sketch}(512) \to \operatorname{acs-grad}}$ & -1.872 & -1.475 & -0.802 & -0.311 & 0.878 & 0.893\\
\textsc{MaxDet}-P & $k_{\mathrm{grad} \to \operatorname{sketch}(512) \to \operatorname{acs-rf-hyper}(512)}$ & -1.872 & -1.475 & -0.803 & -0.310 & 0.888 & 0.872\\
\textsc{MaxDet}-P & $k_{\mathrm{ll} \to \mathcal{X}_{\operatorname{train}}}$ & -1.895 & -1.463 & -0.808 & -0.288 & 0.916 & 0.370\\
\textsc{MaxDet}-P & $k_{\mathrm{ll} \to \operatorname{acs-rf}(512)}$ & -1.876 & -1.443 & -0.792 & -0.264 & 0.961 & 0.518\\
\textsc{MaxDet}-TP & $k_{\mathrm{nngp} \to \operatorname{scale}(\mathcal{X}_{\operatorname{train}})}$ & -1.848 & -1.358 & -0.702 & -0.133 & 1.028 & 5.449\\
\textsc{MaxDet}-P & $k_{\mathrm{lin} \to \mathcal{X}_{\operatorname{train}}}$ & -1.682 & -1.187 & -0.524 & 0.055 & 1.191 & 0.170 \\
\hline
\textsc{Bait-F}-P & $k_{\mathrm{grad} \to \operatorname{ens}(3) \to \operatorname{sketch}(512) \to \mathcal{X}_{\operatorname{train}}}$ & -2.011 & -1.587 & -0.927 & \textbf{-0.419} & 0.859 & 2.346\\
\textsc{Bait-F}-P & $k_{\mathrm{grad} \to \operatorname{sketch}(512) \to \mathcal{X}_{\operatorname{train}}}$ & -2.013 & -1.585 & -0.926 & -0.412 & 0.862 & 1.508\\
\textsc{Bait-FB}-P & $k_{\mathrm{grad} \to \operatorname{sketch}(512) \to \mathcal{X}_{\operatorname{train}}}$ & -2.007 & -1.584 & -0.921 & -0.411 & \textbf{0.852} & 3.050\\
\textsc{Bait-F}-P & $k_{\mathrm{ll} \to \operatorname{ens}(3) \to \operatorname{sketch}(512) \to \mathcal{X}_{\operatorname{train}}}$ & \textbf{-2.041} & -1.583 & -0.940 & -0.400 & 0.880 & 1.408\\
\textsc{Bait-F}-P & $k_{\mathrm{ll} \to \mathcal{X}_{\operatorname{train}}}$ & -2.003 & -1.545 & -0.900 & -0.362 & 0.891 & 1.149\\
\textsc{Bait-FB}-P & $k_{\mathrm{ll} \to \mathcal{X}_{\operatorname{train}}}$ & -1.998 & -1.541 & -0.895 & -0.357 & 0.888 & 2.731\\
\textsc{Bait-F}-P & $k_{\mathrm{lin} \to \mathcal{X}_{\operatorname{train}}}$ & -1.721 & -1.220 & -0.562 & 0.022 & 1.162 & 0.232 \\
\hline
\textsc{FrankWolfe}-P & $k_{\mathrm{grad} \to \operatorname{sketch}(512) \to \operatorname{acs-rf-hyper}(512)}$ & -1.977 & -1.542 & -0.892 & -0.362 & 0.918 & 0.823\\
\textsc{FrankWolfe}-P & $k_{\mathrm{grad} \to \operatorname{sketch}(512) \to \operatorname{acs-grad} \to \operatorname{sketch}(512)}$ & -1.999 & -1.520 & -0.879 & -0.317 & 1.015 & 0.914\\
\textsc{FrankWolfe}-P & $k_{\mathrm{ll} \to \operatorname{acs-rf}(512)}$ & -1.992 & -1.519 & -0.883 & -0.321 & 1.022 & 0.421\\
\textsc{FrankWolfe}-P & $k_{\mathrm{grad} \to \operatorname{sketch}(512) \to \operatorname{acs-rf}(512)}$ & -1.995 & -1.499 & -0.864 & -0.287 & 1.055 & 0.825\\
\textsc{FrankWolfe}-P & $k_{\mathrm{ll} \to \operatorname{acs-grad} \to \operatorname{sketch}(512)}$ & -1.943 & -1.446 & -0.807 & -0.226 & 1.085 & 0.517\\
\textsc{FrankWolfe}-P & $k_{\mathrm{ll} \to \operatorname{acs-rf-hyper}(512)}$ & -1.937 & -1.439 & -0.793 & -0.225 & 1.016 & 0.421\\
\textsc{FrankWolfe}-P & $k_{\mathrm{grad} \to \operatorname{sketch}(512) \to \mathcal{X}_{\operatorname{train}}}$ & -1.924 & -1.410 & -0.765 & -0.178 & 1.134 & 0.723\\
\textsc{FrankWolfe}-P & $k_{\mathrm{ll} \to \mathcal{X}_{\operatorname{train}}}$ & -1.896 & -1.391 & -0.752 & -0.158 & 1.139 & 0.326 \\
\hline
\textsc{MaxDist}-TP & $k_{\mathrm{ll} \to \operatorname{ens}(3) \to \operatorname{sketch}(512)}$ & -1.948 & -1.518 & -0.860 & -0.348 & 0.905 & 0.655\\
\textsc{MaxDist}-P & $k_{\mathrm{grad} \to \operatorname{sketch}(512) \to \mathcal{X}_{\operatorname{train}}}$ & -1.916 & -1.514 & -0.845 & -0.351 & 0.866 & 0.713\\
\textsc{MaxDist}-TP & $k_{\mathrm{grad} \to \operatorname{sketch}(512)}$ & -1.899 & -1.506 & -0.838 & -0.347 & 0.868 & 0.653\\
\textsc{MaxDist}-TP & $k_{\mathrm{grad}}$ & -1.894 & -1.503 & -0.834 & -0.342 & 0.866 & 2.347\\
\textsc{MaxDist}-TP & $k_{\mathrm{grad} \to \operatorname{ens}(3) \to \operatorname{sketch}(512)}$ & -1.889 & -1.498 & -0.831 & -0.342 & 0.871 & 0.743\\
\textsc{MaxDist}-TP & $k_{\mathrm{ll}}$ & -1.924 & -1.491 & -0.832 & -0.307 & 0.927 & 0.621\\
\textsc{MaxDist}-P & $k_{\mathrm{grad} \to \operatorname{sketch}(512) \to \operatorname{acs-rf}(512)}$ & -1.893 & -1.491 & -0.819 & -0.322 & 0.891 & 0.810\\
\textsc{MaxDist}-P & $k_{\mathrm{grad} \to \operatorname{sketch}(512) \to \operatorname{acs-grad}}$ & -1.873 & -1.477 & -0.802 & -0.312 & 0.888 & 0.832\\
\textsc{MaxDist}-P & $k_{\mathrm{grad} \to \operatorname{sketch}(512) \to \operatorname{acs-rf-hyper}(512)}$ & -1.867 & -1.472 & -0.798 & -0.306 & 0.895 & 0.811\\
\textsc{MaxDist}-P & $k_{\mathrm{ll} \to \mathcal{X}_{\operatorname{train}}}$ & -1.889 & -1.459 & -0.807 & -0.283 & 0.947 & 0.309\\
\textsc{MaxDist}-P & $k_{\mathrm{ll} \to \operatorname{acs-rf}(512)}$ & -1.863 & -1.430 & -0.777 & -0.247 & 0.980 & 0.410\\
\textsc{MaxDist}-TP & $k_{\mathrm{lin}}$ & -1.888 & -1.398 & -0.749 & -0.172 & 1.034 & 0.242\\
\textsc{MaxDist}-TP & $k_{\mathrm{nngp}}$ & -1.876 & -1.386 & -0.735 & -0.159 & 1.038 & 1.315 \\
\hline
\textsc{KMeansPP}-TP & $k_{\mathrm{grad}}$ & -2.025 & -1.569 & -0.927 & -0.378 & 0.966 & 2.357\\
\textsc{KMeansPP}-P & $k_{\mathrm{grad} \to \operatorname{sketch}(512) \to \operatorname{acs-rf}(512)}$ & -2.006 & -1.569 & -0.912 & -0.385 & 0.929 & 0.836\\
\textsc{KMeansPP}-TP & $k_{\mathrm{grad} \to \operatorname{ens}(3) \to \operatorname{sketch}(512)}$ & -2.025 & -1.569 & -0.926 & -0.377 & 0.967 & 0.754\\
\textsc{KMeansPP}-TP & $k_{\mathrm{grad} \to \operatorname{sketch}(512)}$ & -2.023 & -1.567 & -0.925 & -0.376 & 0.967 & 0.663\\
\textsc{KMeansPP}-P & $k_{\mathrm{grad} \to \operatorname{sketch}(512) \to \operatorname{acs-grad}}$ & -2.008 & -1.558 & -0.905 & -0.366 & 0.979 & 0.859\\
\textsc{KMeansPP}-P & $k_{\mathrm{grad} \to \operatorname{sketch}(512) \to \operatorname{acs-rf-hyper}(512)}$ & -1.994 & -1.554 & -0.899 & -0.370 & 0.957 & 0.836\\
\textsc{KMeansPP}-P & $k_{\mathrm{grad} \to \operatorname{sketch}(512) \to \mathcal{X}_{\operatorname{train}}}$ & -2.020 & -1.549 & -0.905 & -0.348 & 0.997 & 0.738\\
\textsc{KMeansPP}-P & $k_{\mathrm{ll} \to \mathcal{X}_{\operatorname{train}}}$ & -2.007 & -1.530 & -0.895 & -0.329 & 1.008 & 0.335\\
\textsc{KMeansPP}-P & $k_{\mathrm{ll} \to \operatorname{acs-rf}(512)}$ & -1.986 & -1.529 & -0.889 & -0.339 & 0.977 & 0.435\\
\textsc{KMeansPP}-TP & $k_{\mathrm{ll} \to \operatorname{ens}(3) \to \operatorname{sketch}(512)}$ & -2.020 & -1.522 & -0.890 & -0.317 & 1.014 & 0.666\\
\textsc{KMeansPP}-TP & $k_{\mathrm{ll}}$ & -2.015 & -1.521 & -0.887 & -0.316 & 1.014 & 0.632\\
\textsc{KMeansPP}-TP & $k_{\mathrm{nngp}}$ & -1.969 & -1.446 & -0.817 & -0.215 & 1.072 & 1.319\\
\textsc{KMeansPP}-TP & $k_{\mathrm{lin}}$ & -1.968 & -1.441 & -0.816 & -0.212 & 1.077 & 0.252 \\
\hline
\textsc{LCMD}-TP (ours) & $k_{\mathrm{grad} \to \operatorname{ens}(3) \to \operatorname{sketch}(512)}$ & -2.040 & \textbf{-1.594} & \textbf{-0.947} & -0.408 & 0.920 & 1.073\\
\textsc{LCMD}-TP (ours) & $k_{\mathrm{grad}}$ & -2.038 & \textbf{-1.594} & -0.946 & -0.408 & 0.908 & 2.714\\
\textsc{LCMD}-TP (ours) & $k_{\mathrm{grad} \to \operatorname{sketch}(512)}$ & -2.033 & -1.590 & -0.941 & -0.404 & 0.917 & 0.981\\
\textsc{LCMD}-TP (ours) & $k_{\mathrm{ll} \to \operatorname{ens}(3) \to \operatorname{sketch}(512)}$ & -2.026 & -1.555 & -0.912 & -0.358 & 0.952 & 0.984\\
\textsc{LCMD}-P (ours) & $k_{\mathrm{grad} \to \operatorname{sketch}(512) \to \operatorname{acs-rf}(512)}$ & -1.992 & -1.547 & -0.905 & -0.369 & 0.974 & 0.875\\
\textsc{LCMD}-TP (ours) & $k_{\mathrm{ll}}$ & -2.017 & -1.544 & -0.902 & -0.346 & 0.961 & 0.948\\
\textsc{LCMD}-P (ours) & $k_{\mathrm{grad} \to \operatorname{sketch}(512) \to \mathcal{X}_{\operatorname{train}}}$ & -1.940 & -1.534 & -0.869 & -0.371 & 0.864 & 0.774\\
\textsc{LCMD}-P (ours) & $k_{\mathrm{ll} \to \operatorname{acs-rf}(512)}$ & -1.969 & -1.513 & -0.885 & -0.331 & 1.010 & 0.473\\
\textsc{LCMD}-P (ours) & $k_{\mathrm{grad} \to \operatorname{sketch}(512) \to \operatorname{acs-rf-hyper}(512)}$ & -1.898 & -1.501 & -0.827 & -0.336 & 0.878 & 0.876\\
\textsc{LCMD}-P (ours) & $k_{\mathrm{grad} \to \operatorname{sketch}(512) \to \operatorname{acs-grad}}$ & -1.896 & -1.496 & -0.827 & -0.334 & 0.880 & 0.899\\
\textsc{LCMD}-P (ours) & $k_{\mathrm{ll} \to \mathcal{X}_{\operatorname{train}}}$ & -1.920 & -1.485 & -0.840 & -0.311 & 0.935 & 0.372\\
\textsc{LCMD}-TP (ours) & $k_{\mathrm{nngp}}$ & -1.960 & -1.447 & -0.811 & -0.216 & 1.043 & 1.587\\
\textsc{LCMD}-TP (ours) & $k_{\mathrm{lin}}$ & -1.958 & -1.445 & -0.810 & -0.215 & 1.036 & 0.525
\end{tabular}
\captionsetup{font=footnotesize}
\caption{This table shows the performance and runtime of different combinations of selection methods and kernels. The columns labeled \quot{MAE} to \quot{MAXE} contain averaged logarithmic values of the corresponding metrics, averaged over all data sets, repetitions, and BMAL steps. For ensembled kernels, the metrics of the individual ensemble members were averaged to isolate the effect of ensembling on the batch selection. Runtimes were measured at one of the 20 repetitions where only one process was started per GPU, and are averaged over all BMAL steps and data sets. The employed hardware is described in \Cref{sec:appendix:experiments}.} \label{table:all_algs}
\end{table}

\begin{table}[tbp]
\centering
\fontsize{4}{5}\selectfont
\captionsetup{font=footnotesize}
\begin{tabular}{cccccccc}
Selection method & Kernel & MAE & RMSE & 95\% & 99\% & MAXE & avg.\ time [$s$]\\
\hline
\textsc{Random} & --- & -1.923 & -1.406 & -0.774 & -0.178 & 1.119 & 0.001 \\
\hline
\textsc{MaxDiag} & $k_{\mathrm{grad} \to \operatorname{ens}(3) \to \operatorname{sketch}(512) \to \mathcal{X}_{\operatorname{train}}}$ & -1.753 & -1.358 & -0.690 & -0.205 & 0.956 & 1.393\\
\textsc{MaxDiag} & $k_{\mathrm{grad} \to \operatorname{sketch}(512) \to \mathcal{X}_{\operatorname{train}}}$ & -1.751 & -1.351 & -0.682 & -0.192 & 0.961 & 0.551\\
\textsc{MaxDiag} & $k_{\mathrm{grad} \to \operatorname{sketch}(512) \to \operatorname{acs-rf}(512)}$ & -1.749 & -1.350 & -0.680 & -0.189 & 0.962 & 0.651\\
\textsc{MaxDiag} & $k_{\mathrm{grad} \to \mathcal{X}_{\operatorname{train}}}$ & -1.749 & -1.346 & -0.677 & -0.182 & 0.967 & 4.107\\
\textsc{MaxDiag} & $k_{\mathrm{grad} \to \operatorname{sketch}(512) \to \operatorname{acs-grad}}$ & -1.735 & -1.342 & -0.671 & -0.184 & 0.961 & 0.553\\
\textsc{MaxDiag} & $k_{\mathrm{nngp} \to \mathcal{X}_{\operatorname{train}}}$ & -1.768 & -1.301 & -0.634 & -0.095 & 1.054 & 2.584\\
\textsc{MaxDiag} & $k_{\mathrm{ll} \to \mathcal{X}_{\operatorname{train}}}$ & -1.744 & -1.300 & -0.631 & -0.100 & 1.037 & 0.142\\
\textsc{MaxDiag} & $k_{\mathrm{ll} \to \operatorname{acs-rf}(512)}$ & -1.713 & -1.286 & -0.606 & -0.084 & 1.052 & 0.031\\
\textsc{MaxDiag} & $k_{\mathrm{grad} \to \operatorname{sketch}(512) \to \operatorname{acs-rf-hyper}(512)}$ & -1.675 & -1.279 & -0.604 & -0.111 & 1.007 & 0.649\\
\textsc{MaxDiag} & $k_{\mathrm{lin} \to \mathcal{X}_{\operatorname{train}}}$ & -1.622 & -1.158 & -0.483 & 0.068 & 1.156 & 0.007 \\
\hline
\textsc{MaxDet}-TP & $k_{\mathrm{grad} \to \operatorname{scale}(\mathcal{X}_{\operatorname{train}})}$ & -1.955 & -1.546 & -0.895 & -0.400 & 0.834 & 8.914\\
\textsc{MaxDet}-P & $k_{\mathrm{grad} \to \operatorname{ens}(3) \to \operatorname{sketch}(512) \to \mathcal{X}_{\operatorname{train}}}$ & -1.913 & -1.523 & -0.867 & -0.390 & 0.843 & 1.607\\
\textsc{MaxDet}-P & $k_{\mathrm{grad} \to \operatorname{sketch}(512) \to \mathcal{X}_{\operatorname{train}}}$ & -1.915 & -1.520 & -0.864 & -0.381 & 0.846 & 0.770\\
\textsc{MaxDet}-P & $k_{\mathrm{grad} \to \operatorname{sketch}(512) \to \operatorname{acs-rf}(512)}$ & -1.910 & -1.514 & -0.857 & -0.372 & 0.855 & 0.868\\
\textsc{MaxDet}-P & $k_{\mathrm{grad} \to \operatorname{sketch}(512) \to \operatorname{acs-grad}}$ & -1.886 & -1.503 & -0.840 & -0.370 & 0.846 & 0.890\\
\textsc{MaxDet}-P & $k_{\mathrm{ll} \to \operatorname{ens}(3) \to \operatorname{sketch}(512) \to \mathcal{X}_{\operatorname{train}}}$ & -1.921 & -1.485 & -0.823 & -0.304 & 0.894 & 0.675\\
\textsc{MaxDet}-P & $k_{\mathrm{grad} \to \operatorname{sketch}(512) \to \operatorname{acs-rf-hyper}(512)}$ & -1.847 & -1.468 & -0.804 & -0.335 & 0.863 & 0.871\\
\textsc{MaxDet}-P & $k_{\mathrm{ll} \to \mathcal{X}_{\operatorname{train}}}$ & -1.910 & -1.467 & -0.807 & -0.281 & 0.913 & 0.371\\
\textsc{MaxDet}-P & $k_{\mathrm{ll} \to \operatorname{acs-rf}(512)}$ & -1.905 & -1.459 & -0.801 & -0.272 & 0.924 & 0.470\\
\textsc{MaxDet}-TP & $k_{\mathrm{nngp} \to \operatorname{scale}(\mathcal{X}_{\operatorname{train}})}$ & -1.869 & -1.401 & -0.737 & -0.198 & 0.980 & 5.456\\
\textsc{MaxDet}-P & $k_{\mathrm{lin} \to \mathcal{X}_{\operatorname{train}}}$ & -1.714 & -1.239 & -0.574 & -0.020 & 1.118 & 0.170 \\
\hline
\textsc{Bait-F}-P & $k_{\mathrm{grad} \to \operatorname{ens}(3) \to \operatorname{sketch}(512) \to \mathcal{X}_{\operatorname{train}}}$ & -2.025 & -1.594 & -0.948 & \textbf{-0.436} & 0.855 & 2.342\\
\textsc{Bait-F}-P & $k_{\mathrm{grad} \to \operatorname{sketch}(512) \to \mathcal{X}_{\operatorname{train}}}$ & -2.023 & -1.588 & -0.943 & -0.429 & 0.859 & 1.505\\
\textsc{Bait-FB}-P & $k_{\mathrm{grad} \to \operatorname{sketch}(512) \to \mathcal{X}_{\operatorname{train}}}$ & -2.019 & -1.585 & -0.938 & -0.426 & 0.853 & 3.051\\
\textsc{Bait-F}-P & $k_{\mathrm{ll} \to \operatorname{ens}(3) \to \operatorname{sketch}(512) \to \mathcal{X}_{\operatorname{train}}}$ & -2.002 & -1.541 & -0.885 & -0.345 & 0.894 & 1.407\\
\textsc{Bait-F}-P & $k_{\mathrm{ll} \to \mathcal{X}_{\operatorname{train}}}$ & -1.989 & -1.524 & -0.869 & -0.324 & 0.913 & 1.147\\
\textsc{Bait-FB}-P & $k_{\mathrm{ll} \to \mathcal{X}_{\operatorname{train}}}$ & -1.989 & -1.522 & -0.868 & -0.323 & 0.917 & 2.733\\
\textsc{Bait-F}-P & $k_{\mathrm{lin} \to \mathcal{X}_{\operatorname{train}}}$ & -1.742 & -1.264 & -0.601 & -0.045 & 1.110 & 0.233 \\
\hline
\textsc{FrankWolfe}-P & $k_{\mathrm{grad} \to \operatorname{sketch}(512) \to \operatorname{acs-rf-hyper}(512)}$ & -1.961 & -1.529 & -0.882 & -0.363 & 0.922 & 0.822\\
\textsc{FrankWolfe}-P & $k_{\mathrm{grad} \to \operatorname{sketch}(512) \to \operatorname{acs-rf}(512)}$ & -1.977 & -1.502 & -0.864 & -0.299 & 1.015 & 0.824\\
\textsc{FrankWolfe}-P & $k_{\mathrm{grad} \to \operatorname{sketch}(512) \to \operatorname{acs-grad} \to \operatorname{sketch}(512)}$ & -1.960 & -1.481 & -0.840 & -0.279 & 1.049 & 0.915\\
\textsc{FrankWolfe}-P & $k_{\mathrm{ll} \to \operatorname{acs-rf}(512)}$ & -1.959 & -1.467 & -0.830 & -0.255 & 1.031 & 0.422\\
\textsc{FrankWolfe}-P & $k_{\mathrm{ll} \to \operatorname{acs-rf-hyper}(512)}$ & -1.923 & -1.437 & -0.797 & -0.227 & 1.016 & 0.420\\
\textsc{FrankWolfe}-P & $k_{\mathrm{ll} \to \operatorname{acs-grad} \to \operatorname{sketch}(512)}$ & -1.911 & -1.413 & -0.771 & -0.191 & 1.092 & 0.519\\
\textsc{FrankWolfe}-P & $k_{\mathrm{grad} \to \operatorname{sketch}(512) \to \mathcal{X}_{\operatorname{train}}}$ & -1.881 & -1.384 & -0.736 & -0.157 & 1.129 & 0.725\\
\textsc{FrankWolfe}-P & $k_{\mathrm{ll} \to \mathcal{X}_{\operatorname{train}}}$ & -1.859 & -1.358 & -0.711 & -0.128 & 1.143 & 0.324 \\
\hline
\textsc{MaxDist}-TP & $k_{\mathrm{grad} \to \operatorname{ens}(3) \to \operatorname{sketch}(512)}$ & -1.919 & -1.540 & -0.889 & -0.416 & \textbf{0.818} & 0.744\\
\textsc{MaxDist}-TP & $k_{\mathrm{grad}}$ & -1.919 & -1.539 & -0.888 & -0.413 & 0.823 & 2.351\\
\textsc{MaxDist}-TP & $k_{\mathrm{grad} \to \operatorname{sketch}(512)}$ & -1.918 & -1.536 & -0.886 & -0.409 & 0.820 & 0.652\\
\textsc{MaxDist}-TP & $k_{\mathrm{ll} \to \operatorname{ens}(3) \to \operatorname{sketch}(512)}$ & -1.958 & -1.526 & -0.879 & -0.360 & 0.876 & 0.655\\
\textsc{MaxDist}-P & $k_{\mathrm{grad} \to \operatorname{sketch}(512) \to \mathcal{X}_{\operatorname{train}}}$ & -1.913 & -1.518 & -0.861 & -0.378 & 0.849 & 0.712\\
\textsc{MaxDist}-TP & $k_{\mathrm{ll}}$ & -1.951 & -1.515 & -0.870 & -0.346 & 0.888 & 0.621\\
\textsc{MaxDist}-P & $k_{\mathrm{grad} \to \operatorname{sketch}(512) \to \operatorname{acs-rf}(512)}$ & -1.906 & -1.508 & -0.851 & -0.364 & 0.865 & 0.811\\
\textsc{MaxDist}-P & $k_{\mathrm{grad} \to \operatorname{sketch}(512) \to \operatorname{acs-grad}}$ & -1.890 & -1.507 & -0.844 & -0.373 & 0.843 & 0.833\\
\textsc{MaxDist}-P & $k_{\mathrm{ll} \to \mathcal{X}_{\operatorname{train}}}$ & -1.913 & -1.468 & -0.810 & -0.282 & 0.917 & 0.311\\
\textsc{MaxDist}-P & $k_{\mathrm{grad} \to \operatorname{sketch}(512) \to \operatorname{acs-rf-hyper}(512)}$ & -1.846 & -1.462 & -0.802 & -0.327 & 0.868 & 0.810\\
\textsc{MaxDist}-P & $k_{\mathrm{ll} \to \operatorname{acs-rf}(512)}$ & -1.899 & -1.452 & -0.794 & -0.264 & 0.937 & 0.409\\
\textsc{MaxDist}-TP & $k_{\mathrm{lin}}$ & -1.905 & -1.435 & -0.777 & -0.233 & 0.982 & 0.238\\
\textsc{MaxDist}-TP & $k_{\mathrm{nngp}}$ & -1.897 & -1.426 & -0.768 & -0.224 & 0.989 & 1.286 \\
\hline
\textsc{KMeansPP}-TP & $k_{\mathrm{grad} \to \operatorname{ens}(3) \to \operatorname{sketch}(512)}$ & -2.022 & -1.566 & -0.931 & -0.382 & 0.971 & 0.755\\
\textsc{KMeansPP}-TP & $k_{\mathrm{grad}}$ & -2.023 & -1.566 & -0.932 & -0.381 & 0.969 & 2.360\\
\textsc{KMeansPP}-TP & $k_{\mathrm{grad} \to \operatorname{sketch}(512)}$ & -2.022 & -1.566 & -0.932 & -0.381 & 0.969 & 0.662\\
\textsc{KMeansPP}-P & $k_{\mathrm{grad} \to \operatorname{sketch}(512) \to \operatorname{acs-rf}(512)}$ & -2.001 & -1.558 & -0.914 & -0.381 & 0.920 & 0.837\\
\textsc{KMeansPP}-P & $k_{\mathrm{grad} \to \operatorname{sketch}(512) \to \operatorname{acs-grad}}$ & -2.006 & -1.542 & -0.902 & -0.351 & 0.972 & 0.858\\
\textsc{KMeansPP}-P & $k_{\mathrm{grad} \to \operatorname{sketch}(512) \to \operatorname{acs-rf-hyper}(512)}$ & -1.987 & -1.541 & -0.898 & -0.363 & 0.972 & 0.836\\
\textsc{KMeansPP}-P & $k_{\mathrm{grad} \to \operatorname{sketch}(512) \to \mathcal{X}_{\operatorname{train}}}$ & -2.006 & -1.531 & -0.895 & -0.334 & 0.996 & 0.736\\
\textsc{KMeansPP}-TP & $k_{\mathrm{ll} \to \operatorname{ens}(3) \to \operatorname{sketch}(512)}$ & -2.011 & -1.525 & -0.894 & -0.325 & 1.000 & 0.665\\
\textsc{KMeansPP}-TP & $k_{\mathrm{ll}}$ & -2.007 & -1.523 & -0.889 & -0.320 & 1.010 & 0.632\\
\textsc{KMeansPP}-P & $k_{\mathrm{ll} \to \operatorname{acs-rf}(512)}$ & -1.982 & -1.501 & -0.861 & -0.299 & 0.985 & 0.434\\
\textsc{KMeansPP}-P & $k_{\mathrm{ll} \to \mathcal{X}_{\operatorname{train}}}$ & -1.980 & -1.484 & -0.849 & -0.275 & 1.020 & 0.335\\
\textsc{KMeansPP}-TP & $k_{\mathrm{nngp}}$ & -1.968 & -1.465 & -0.829 & -0.246 & 1.051 & 1.298\\
\textsc{KMeansPP}-TP & $k_{\mathrm{lin}}$ & -1.966 & -1.464 & -0.827 & -0.248 & 1.057 & 0.249 \\
\hline
\textsc{LCMD}-TP (ours) & $k_{\mathrm{grad} \to \operatorname{ens}(3) \to \operatorname{sketch}(512)}$ & \textbf{-2.041} & \textbf{-1.600} & \textbf{-0.961} & -0.425 & 0.901 & 1.071\\
\textsc{LCMD}-TP (ours) & $k_{\mathrm{grad}}$ & -2.038 & -1.598 & -0.958 & -0.423 & 0.899 & 2.712\\
\textsc{LCMD}-TP (ours) & $k_{\mathrm{grad} \to \operatorname{sketch}(512)}$ & -2.038 & -1.597 & -0.957 & -0.422 & 0.898 & 0.977\\
\textsc{LCMD}-TP (ours) & $k_{\mathrm{ll} \to \operatorname{ens}(3) \to \operatorname{sketch}(512)}$ & -2.027 & -1.554 & -0.916 & -0.359 & 0.941 & 0.979\\
\textsc{LCMD}-TP (ours) & $k_{\mathrm{ll}}$ & -2.022 & -1.550 & -0.911 & -0.357 & 0.944 & 0.946\\
\textsc{LCMD}-P (ours) & $k_{\mathrm{grad} \to \operatorname{sketch}(512) \to \operatorname{acs-rf}(512)}$ & -1.982 & -1.532 & -0.889 & -0.348 & 0.945 & 0.874\\
\textsc{LCMD}-P (ours) & $k_{\mathrm{grad} \to \operatorname{sketch}(512) \to \mathcal{X}_{\operatorname{train}}}$ & -1.930 & -1.531 & -0.880 & -0.392 & 0.847 & 0.774\\
\textsc{LCMD}-P (ours) & $k_{\mathrm{grad} \to \operatorname{sketch}(512) \to \operatorname{acs-grad}}$ & -1.907 & -1.520 & -0.862 & -0.389 & 0.849 & 0.898\\
\textsc{LCMD}-P (ours) & $k_{\mathrm{grad} \to \operatorname{sketch}(512) \to \operatorname{acs-rf-hyper}(512)}$ & -1.905 & -1.508 & -0.856 & -0.371 & 0.870 & 0.877\\
\textsc{LCMD}-TP (ours) & $k_{\mathrm{lin}}$ & -1.971 & -1.484 & -0.836 & -0.274 & 0.991 & 0.517\\
\textsc{LCMD}-TP (ours) & $k_{\mathrm{nngp}}$ & -1.970 & -1.482 & -0.834 & -0.271 & 0.996 & 1.560\\
\textsc{LCMD}-P (ours) & $k_{\mathrm{ll} \to \mathcal{X}_{\operatorname{train}}}$ & -1.928 & -1.481 & -0.825 & -0.296 & 0.919 & 0.370\\
\textsc{LCMD}-P (ours) & $k_{\mathrm{ll} \to \operatorname{acs-rf}(512)}$ & -1.963 & -1.476 & -0.840 & -0.272 & 0.986 & 0.475
\end{tabular}
\caption{This table shows the performance and runtime of different combinations of selection methods and kernels for the SiLU activation function. The columns labeled \quot{MAE} to \quot{MAXE} contain averaged logarithmic values of the corresponding metrics, averaged over all data sets, repetitions and BMAL steps. For ensembled kernels, the metrics of the individual ensemble members were averaged to isolate the effect of ensembling on the batch selection. Runtimes were measured at one of the 20 repetitions where only one process was started per GPU, and are averaged over all BMAL steps and data sets. The employed hardware is described in \Cref{sec:appendix:experiments}.} \label{table:all_algs_silu}
\end{table}

\begin{figure}[tbp]
\centering
\includegraphics[height=0.75\textheight]{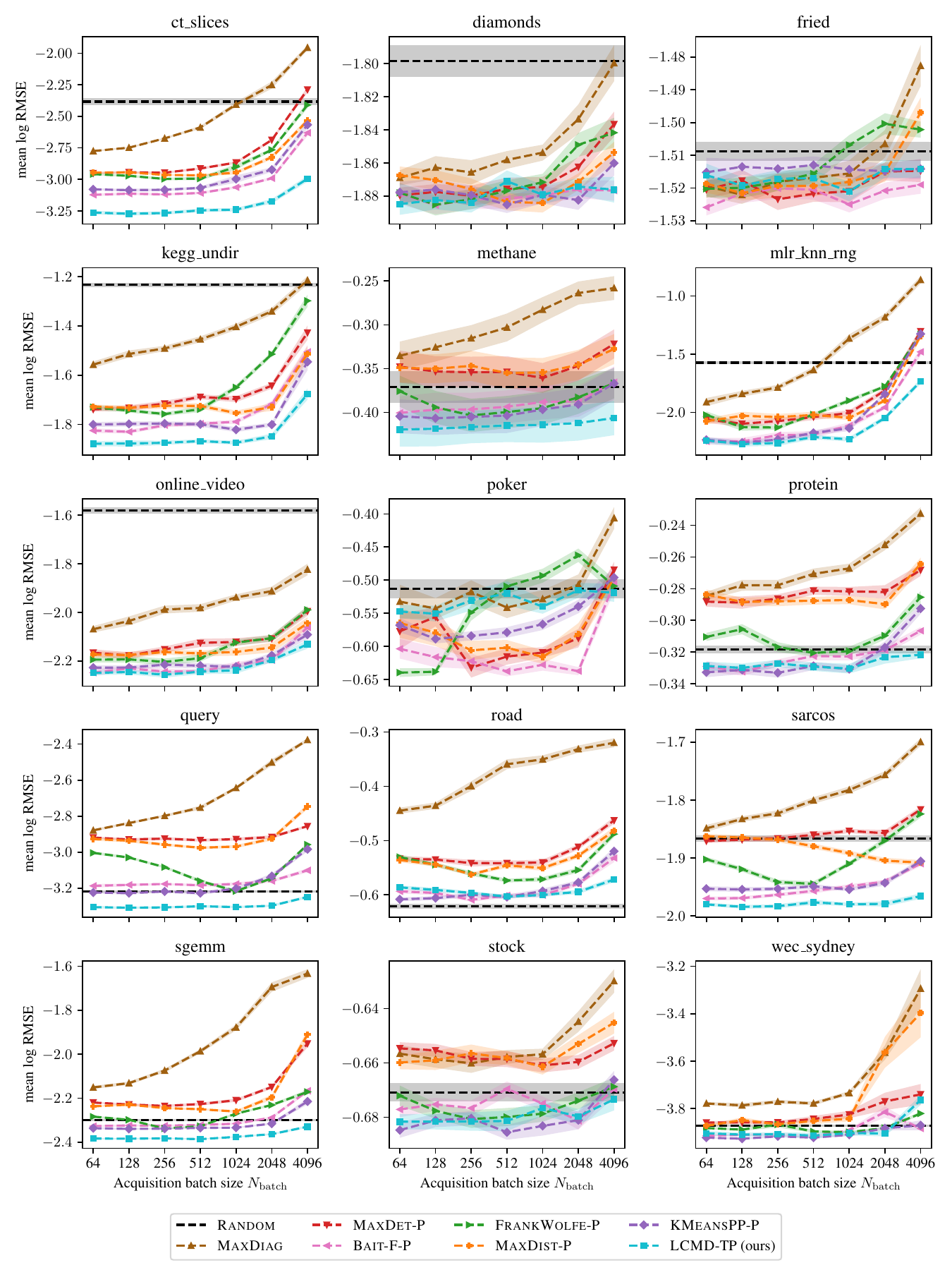}
\caption{%
This figure shows how much the final accuracy of different BMDAL methods deteriorates on individual data sets when fewer BMAL steps with larger batch sizes are used. Specifically, we use different selection methods with the corresponding kernels from \Cref{table:selected_kernels}, starting with $\Ntrain = 256$ and then performing $2^m$ BMAL steps with batch size $\Nbatch = 2^{12-m}$ for $m \in \{0, \hdots, 6\}$, such that the final training set size is $4352$ in each case. The plots show the final logarithmic error metric, averaged over all repetitions. Note that the performance of \Random{} selection does not depend on $\Nbatch$ but only on the final training set size, hence it is shown as a constant line here. The shaded area corresponds to one estimated standard deviation of the mean estimator, cf.\ \Cref{sec:appendix:results}.} \label{fig:batch_sizes_individual_rmse}
\end{figure}

\begin{figure}[tbp]
\centering
\includegraphics[height=0.82\textheight]{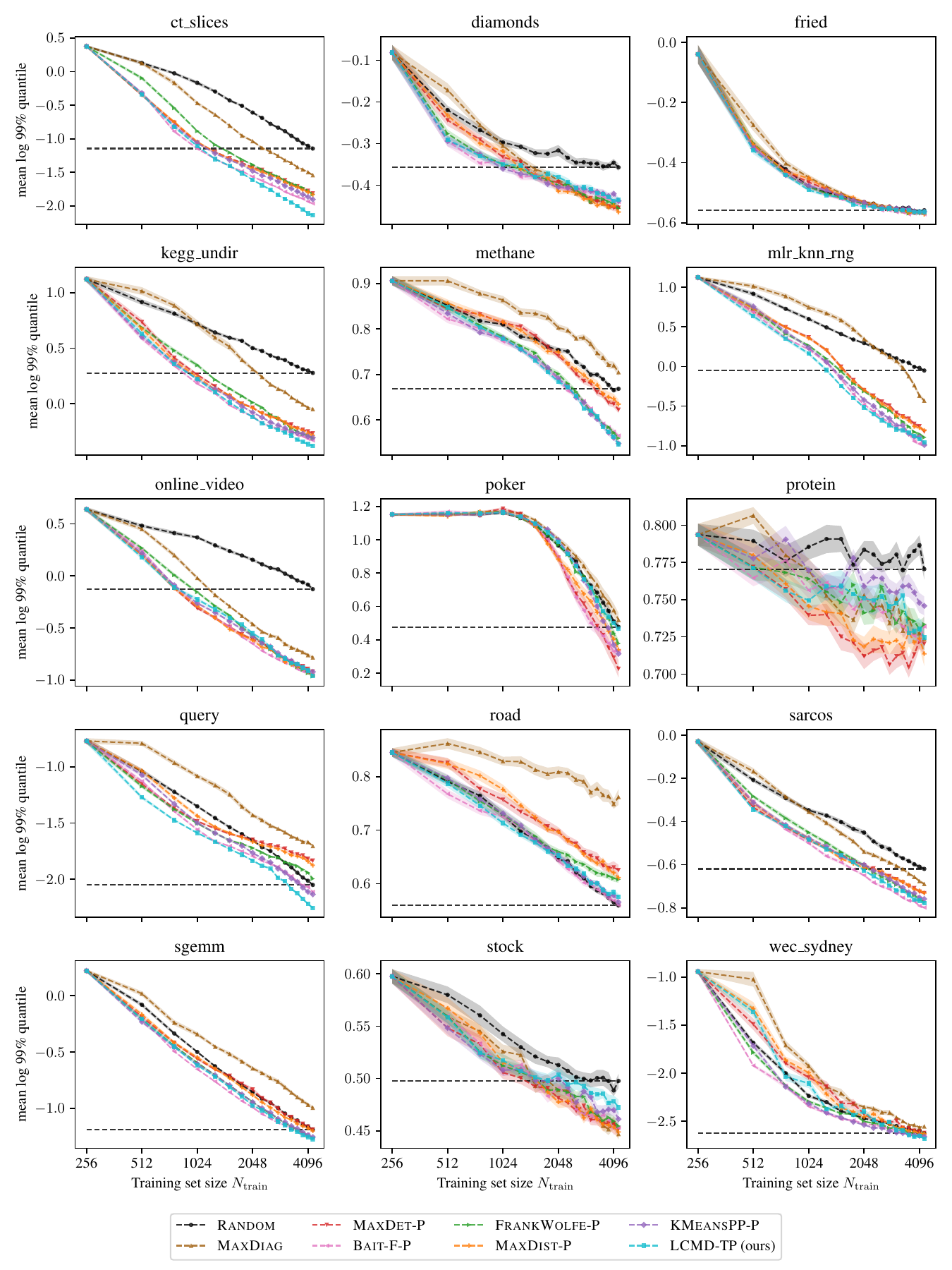}
\caption{%
This figure shows how fast the 99\% quantile decreases during BMAL on the individual benchmark data sets for different selection methods and their corresponding kernels from \Cref{table:selected_kernels}.
Specifically, the plots above show the logarithmic 99\% quantile between each BMAL step for $\Nbatch = 256$, averaged over all repetitions.
The black horizontal dashed line corresponds to the final performance of \Random{} at $\Ntrain = 4352$. The shaded area corresponds to one estimated standard deviation of the mean estimator, cf.\ \Cref{sec:appendix:results}.} \label{fig:learning_curves_individual_q99}
\end{figure}

\begin{figure}[tbp]
\centering
\includegraphics[width=\textwidth]{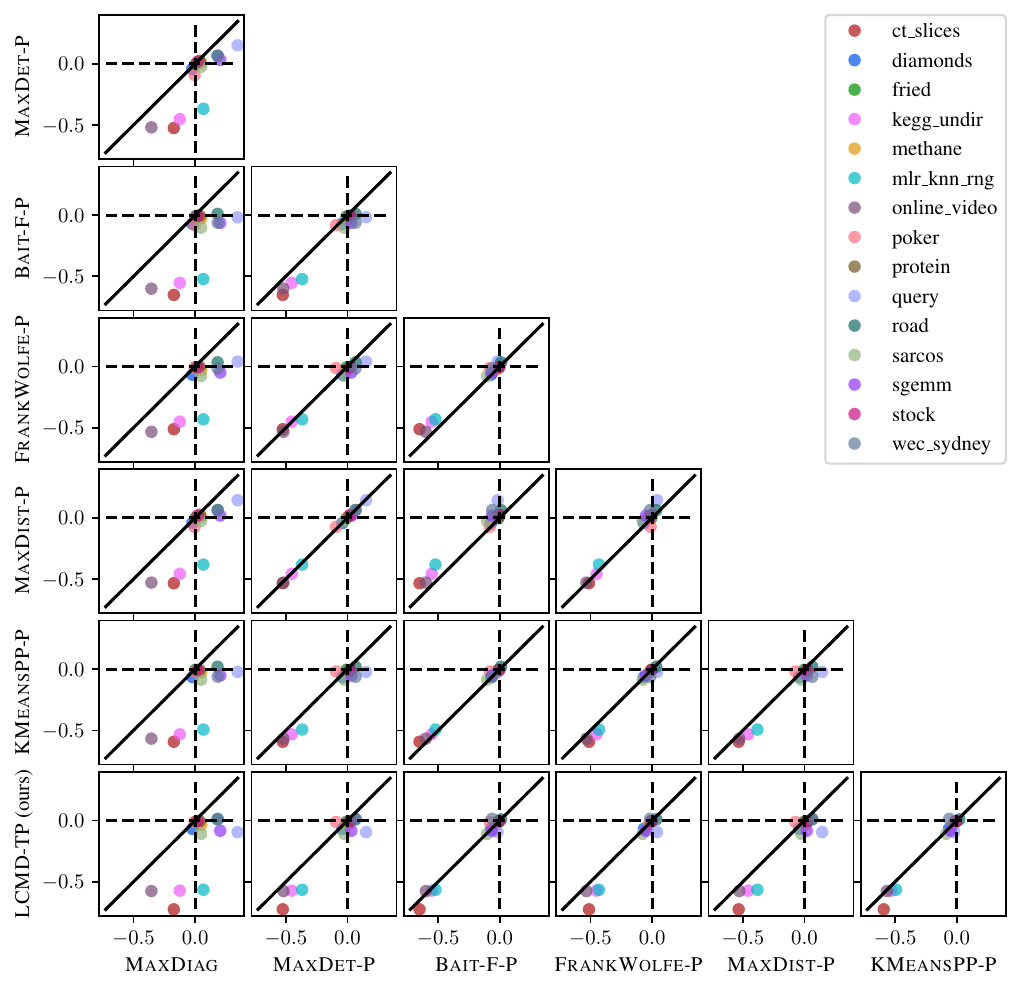}
\caption{Each subplot shows the errors of two selection methods and their corresponding selected kernels from \Cref{table:selected_kernels}. Specifically, the coordinates correspond to the mean log RMSE of the method on the data set minus the mean log RMSE of \textsc{Random} selection on the same data set. Hence, the method on the $x$ axis has a lower mean log RMSE than \Random{} on a data set if the corresponding point is left of the vertical dashed line, and it has a lower mean log RMSE than the method on the $y$ axis if the corresponding point is left of the diagonal line. Similarly, the method on the $y$ axis has a lower mean log RMSE than \Random{} on a data set if the corresponding point is below the horizontal dashed line, and it has a lower mean log RMSE than the method on the $x$ axis if the corresponding point is below the diagonal line.} \label{fig:correlation_plot_rmse}
\end{figure}

\begin{figure}[tbp]
\centering
\includegraphics[width=\textwidth]{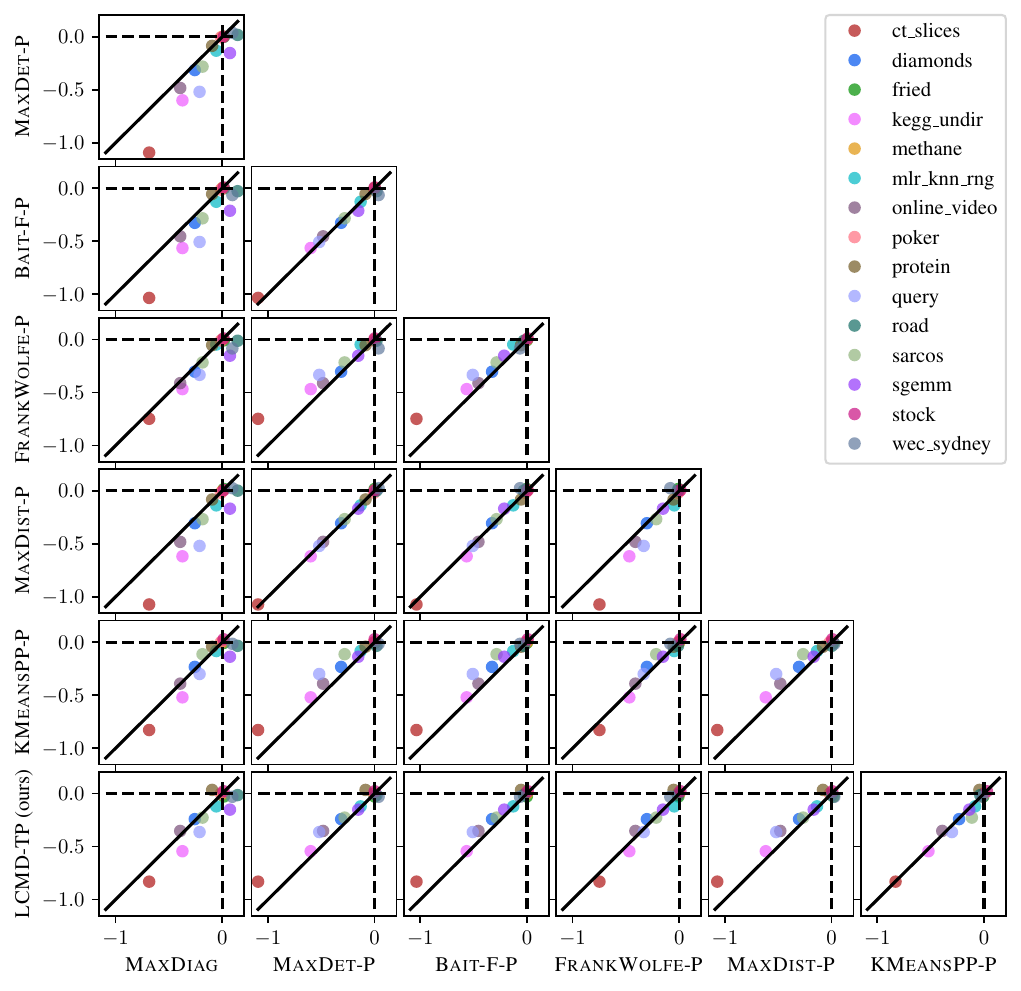}
\caption{Each subplot shows the errors of two selection methods and their corresponding selected kernels from \Cref{table:selected_kernels}. Specifically, the coordinates correspond to the mean log MAXE of the method on the data set minus the mean log MAXE of \textsc{Random} selection on the same data set. Hence, the method on the $x$ axis has a lower mean log MAXE than \textsc{Random} on a data set if the corresponding point is left of the vertical dashed line, and it has a lower mean log MAXE than the method on the $y$ axis if the corresponding point is left of the diagonal line.} \label{fig:correlation_plot_maxe}
\end{figure}

\ifnotinthesis{%
\clearpage
\bibliography{2021_batch_al}}

\begin{thebibliography}{132}
\providecommand{\natexlab}[1]{#1}
\providecommand{\url}[1]{\texttt{#1}}
\expandafter\ifx\csname urlstyle\endcsname\relax
  \providecommand{\doi}[1]{doi: #1}\else
  \providecommand{\doi}{doi: \begingroup \urlstyle{rm}\Url}\fi

\bibitem[Ahle et~al.(2020)Ahle, Kapralov, Knudsen, Pagh, Velingker, Woodruff,
  and Zandieh]{ahle_oblivious_2020}
Thomas~D. Ahle, Michael Kapralov, Jakob~BT Knudsen, Rasmus Pagh, Ameya
  Velingker, David~P. Woodruff, and Amir Zandieh.
\newblock Oblivious sketching of high-degree polynomial kernels.
\newblock In \emph{{ACM}-{SIAM} {Symposium} on {Discrete} {Algorithms}}, 2020.

\bibitem[Aljundi et~al.(2022)Aljundi, Chumerin, and
  Reino]{aljundi_identifying_2022}
Rahaf Aljundi, Nikolay Chumerin, and Daniel~Olmeda Reino.
\newblock Identifying wrongly predicted samples: {A} method for active
  learning.
\newblock In \emph{Winter {Conference} on {Applications} of {Computer}
  {Vision}}, 2022.

\bibitem[Anagnostopoulos et~al.(2018)Anagnostopoulos, Savva, and
  Triantafillou]{anagnostopoulos_scalable_2018}
Christos Anagnostopoulos, Fotis Savva, and Peter Triantafillou.
\newblock Scalable aggregation predictive analytics.
\newblock \emph{Applied Intelligence}, 48\penalty0 (9):\penalty0 2546--2567,
  2018.

\bibitem[Arora et~al.(2019)Arora, Du, Hu, Li, Salakhutdinov, and
  Wang]{arora_exact_2019}
Sanjeev Arora, Simon~S. Du, Wei Hu, Zhiyuan Li, Russ~R. Salakhutdinov, and
  Ruosong Wang.
\newblock On exact computation with an infinitely wide neural net.
\newblock In \emph{Neural {Information} {Processing} {Systems}}, 2019.

\bibitem[Arriaga and Vempala(1999)]{arriaga_algorithmic_1999}
Rosa~I. Arriaga and Santosh Vempala.
\newblock Algorithmic theories of learning.
\newblock In \emph{Foundations of {Computer} {Science}}, 1999.

\bibitem[Arthur and Vassilvitskii(2007)]{arthur_k-means_2007}
David Arthur and Sergei Vassilvitskii.
\newblock k-means++: {The} advantages of careful seeding.
\newblock In \emph{{ACM}-{SIAM} {Symposium} on {Discrete} {Algorithms}}, 2007.

\bibitem[Ash et~al.(2021)Ash, Goel, Krishnamurthy, and Kakade]{ash_gone_2021}
Jordan Ash, Surbhi Goel, Akshay Krishnamurthy, and Sham Kakade.
\newblock Gone fishing: {Neural} active learning with {Fisher} embeddings.
\newblock In \emph{Neural {Information} {Processing} {Systems}}, 2021.

\bibitem[Ash et~al.(2019)Ash, Zhang, Krishnamurthy, Langford, and
  Agarwal]{ash_deep_2019}
Jordan~T. Ash, Chicheng Zhang, Akshay Krishnamurthy, John Langford, and Alekh
  Agarwal.
\newblock Deep batch active learning by diverse, uncertain gradient lower
  bounds.
\newblock In \emph{International {Conference} on {Learning} {Representations}},
  2019.

\bibitem[Atanasov et~al.(2021)Atanasov, Bordelon, and
  Pehlevan]{atanasov_neural_2021}
Alexander Atanasov, Blake Bordelon, and Cengiz Pehlevan.
\newblock Neural networks as kernel learners: {The} silent alignment effect.
\newblock In \emph{International {Conference} on {Learning} {Representations}},
  2021.

\bibitem[Ballester-Ripoll et~al.(2019)Ballester-Ripoll, Paredes, and
  Pajarola]{ballester-ripoll_sobol_2019}
Rafael Ballester-Ripoll, Enrique~G. Paredes, and Renato Pajarola.
\newblock Sobol tensor trains for global sensitivity analysis.
\newblock \emph{Reliability Engineering \& System Safety}, 183:\penalty0
  311--322, 2019.

\bibitem[Behler(2016)]{behler_perspective_2016}
Jörg Behler.
\newblock Perspective: {Machine} learning potentials for atomistic simulations.
\newblock \emph{Journal of Chemical Physics}, 145\penalty0 (17):\penalty0
  170901, 2016.

\bibitem[Beluch et~al.(2018)Beluch, Genewein, Nürnberger, and
  Köhler]{beluch_power_2018}
William~H. Beluch, Tim Genewein, Andreas Nürnberger, and Jan~M. Köhler.
\newblock The power of ensembles for active learning in image classification.
\newblock In \emph{Conference on {Computer} {Vision} and {Pattern}
  {Recognition}}, 2018.

\bibitem[Bern and Eppstein(1996)]{bern_approximation_1996}
Marshall Bern and David Eppstein.
\newblock Approximation algorithms for geometric problems.
\newblock In \emph{Approximation {Algorithms} for {NP}-hard {Problems}}, pages
  296--345. PWS Publishing Company, 1996.

\bibitem[Bishop(2006)]{bishop_pattern_2006}
Christopher~M. Bishop.
\newblock \emph{Pattern {Recognition} and {Machine} {Learning}}.
\newblock Springer, 2006.

\bibitem[Borsos et~al.(2020)Borsos, Mutny, and Krause]{borsos_coresets_2020}
Zalán Borsos, Mojmir Mutny, and Andreas Krause.
\newblock Coresets via bilevel optimization for continual learning and
  streaming.
\newblock In \emph{Neural {Information} {Processing} {Systems}}, 2020.

\bibitem[Borsos et~al.(2021)Borsos, Tagliasacchi, and
  Krause]{borsos_semi-supervised_2021}
Zalán Borsos, Marco Tagliasacchi, and Andreas Krause.
\newblock Semi-supervised batch active learning via bilevel optimization.
\newblock In \emph{Conference on {Acoustics}, {Speech} and {Signal}
  {Processing}}, 2021.

\bibitem[Bıyık et~al.(2019)Bıyık, Wang, Anari, and
  Sadigh]{biyik_batch_2019}
Erdem Bıyık, Kenneth Wang, Nima Anari, and Dorsa Sadigh.
\newblock Batch active learning using determinantal point processes.
\newblock \emph{arXiv:1906.07975}, 2019.

\bibitem[Caselton and Zidek(1984)]{caselton_optimal_1984}
William~F. Caselton and James~V. Zidek.
\newblock Optimal monitoring network designs.
\newblock \emph{Statistics \& Probability Letters}, 2\penalty0 (4):\penalty0
  223--227, 1984.

\bibitem[Celebi et~al.(2013)Celebi, Kingravi, and
  Vela]{celebi_comparative_2013}
M.~Emre Celebi, Hassan~A. Kingravi, and Patricio~A. Vela.
\newblock A comparative study of efficient initialization methods for the
  k-means clustering algorithm.
\newblock \emph{Expert Systems with Applications}, 40\penalty0 (1):\penalty0
  200--210, 2013.

\bibitem[Chaloner and Verdinelli(1995)]{chaloner_bayesian_1995}
Kathryn Chaloner and Isabella Verdinelli.
\newblock Bayesian experimental design: {A} review.
\newblock \emph{Statistical Science}, pages 273--304, 1995.

\bibitem[Chen et~al.(2018)Chen, Zhang, and Zhou]{chen_fast_2018}
Laming Chen, Guoxin Zhang, and Eric Zhou.
\newblock Fast greedy map inference for determinantal point process to improve
  recommendation diversity.
\newblock In \emph{Neural {Information} {Processing} {Systems}}, 2018.

\bibitem[Chizat et~al.(2019)Chizat, Oyallon, and Bach]{chizat_lazy_2019}
Lenaic Chizat, Edouard Oyallon, and Francis Bach.
\newblock On lazy training in differentiable programming.
\newblock In \emph{Neural {Information} {Processing} {Systems}}, 2019.

\bibitem[Civril and Magdon-Ismail(2013)]{civril_exponential_2013}
Ali Civril and Malik Magdon-Ismail.
\newblock Exponential inapproximability of selecting a maximum volume
  sub-matrix.
\newblock \emph{Algorithmica}, 65\penalty0 (1):\penalty0 159--176, 2013.

\bibitem[Cohn(1996)]{cohn_neural_1996}
David~A. Cohn.
\newblock Neural network exploration using optimal experiment design.
\newblock \emph{Neural Networks}, 9\penalty0 (6):\penalty0 1071--1083, 1996.

\bibitem[Coleman et~al.(2019)Coleman, Yeh, Mussmann, Mirzasoleiman, Bailis,
  Liang, Leskovec, and Zaharia]{coleman_selection_2019}
Cody Coleman, Christopher Yeh, Stephen Mussmann, Baharan Mirzasoleiman, Peter
  Bailis, Percy Liang, Jure Leskovec, and Matei Zaharia.
\newblock Selection via proxy: {Efficient} data selection for deep learning.
\newblock In \emph{International {Conference} on {Learning} {Representations}},
  2019.

\bibitem[Daxberger et~al.(2021)Daxberger, Kristiadi, Immer, Eschenhagen, Bauer,
  and Hennig]{daxberger_laplace_2021}
Erik Daxberger, Agustinus Kristiadi, Alexander Immer, Runa Eschenhagen,
  Matthias Bauer, and Philipp Hennig.
\newblock Laplace {Redux}-{Effortless} {Bayesian} {Deep} {Learning}.
\newblock In \emph{Neural {Information} {Processing} {Systems}}, 2021.

\bibitem[De~Marchi et~al.(2005)De~Marchi, Schaback, and
  Wendland]{de_marchi_near-optimal_2005}
Stefano De~Marchi, Robert Schaback, and Holger Wendland.
\newblock Near-optimal data-independent point locations for radial basis
  function interpolation.
\newblock \emph{Advances in Computational Mathematics}, 23\penalty0
  (3):\penalty0 317--330, 2005.

\bibitem[Deneke et~al.(2014)Deneke, Haile, Lafond, and
  Lilius]{deneke_video_2014}
Tewodors Deneke, Habtegebreil Haile, Sébastien Lafond, and Johan Lilius.
\newblock Video transcoding time prediction for proactive load balancing.
\newblock In \emph{International {Conference} on {Multimedia} and {Expo}},
  2014.

\bibitem[Dua and Graff(2017)]{dua_uci_2017}
Dheeru Dua and Casey Graff.
\newblock {UCI} {Machine} {Learning} {Repository}, 2017.
\newblock URL \url{http://archive.ics.uci.edu/ml}.

\bibitem[Elfwing et~al.(2018)Elfwing, Uchibe, and
  Doya]{elfwing_sigmoid-weighted_2018}
Stefan Elfwing, Eiji Uchibe, and Kenji Doya.
\newblock Sigmoid-weighted linear units for neural network function
  approximation in reinforcement learning.
\newblock \emph{Neural Networks}, 107:\penalty0 3--11, 2018.

\bibitem[Eppstein et~al.(2020)Eppstein, Har-Peled, and
  Sidiropoulos]{eppstein_approximate_2020}
David Eppstein, Sariel Har-Peled, and Anastasios Sidiropoulos.
\newblock Approximate greedy clustering and distance selection for graph
  metrics.
\newblock \emph{Journal of Computational Geometry}, 11\penalty0 (1):\penalty0
  629--652, 2020.

\bibitem[Eschenhagen et~al.(2021)Eschenhagen, Daxberger, Hennig, and
  Kristiadi]{eschenhagen_mixtures_2021}
Runa Eschenhagen, Erik Daxberger, Philipp Hennig, and Agustinus Kristiadi.
\newblock Mixtures of {Laplace} approximations for improved post-hoc
  uncertainty in deep learning.
\newblock In \emph{{NeurIPS} 2021 {Workshop} on {Bayesian} {Deep} {Learning}},
  2021.

\bibitem[Farquhar et~al.(2021)Farquhar, Gal, and
  Rainforth]{farquhar_statistical_2021}
Sebastian Farquhar, Yarin Gal, and Tom Rainforth.
\newblock On statistical bias in active learning: {How} and when to fix it.
\newblock In \emph{International {Conference} on {Learning} {Representations}},
  2021.

\bibitem[Feder and Greene(1988)]{feder_optimal_1988}
Tomás Feder and Daniel Greene.
\newblock Optimal algorithms for approximate clustering.
\newblock In \emph{{ACM} {Symposium} on {Theory} of {Computing}}, 1988.

\bibitem[Fedorov(1972)]{fedorov_theory_1972}
Valerii~V. Fedorov.
\newblock \emph{Theory of {Optimal} {Experiments}}.
\newblock Academic Press, New York, 1972.

\bibitem[Fort et~al.(2020)Fort, Dziugaite, Paul, Kharaghani, Roy, and
  Ganguli]{fort_deep_2020}
Stanislav Fort, Gintare~Karolina Dziugaite, Mansheej Paul, Sepideh Kharaghani,
  Daniel~M. Roy, and Surya Ganguli.
\newblock Deep learning versus kernel learning: an empirical study of loss
  landscape geometry and the time evolution of the neural tangent kernel.
\newblock In \emph{Neural {Information} {Processing} {Systems}}, 2020.

\bibitem[Frank and Wolfe(1956)]{frank_algorithm_1956}
Marguerite Frank and Philip Wolfe.
\newblock An algorithm for quadratic programming.
\newblock \emph{Naval Research Logistics Quarterly}, 3\penalty0 (1-2):\penalty0
  95--110, 1956.

\bibitem[Friedman(1991)]{friedman_multivariate_1991}
Jerome~H. Friedman.
\newblock Multivariate adaptive regression splines.
\newblock \emph{The Annals of Statistics}, pages 1--67, 1991.

\bibitem[Gal and Ghahramani(2016)]{gal_dropout_2016}
Yarin Gal and Zoubin Ghahramani.
\newblock Dropout as a bayesian approximation: {Representing} model uncertainty
  in deep learning.
\newblock In \emph{International {Conference} on {Machine} {Learning}}, 2016.

\bibitem[Gal et~al.(2017)Gal, Islam, and Ghahramani]{gal_deep_2017}
Yarin Gal, Riashat Islam, and Zoubin Ghahramani.
\newblock Deep bayesian active learning with image data.
\newblock In \emph{International {Conference} on {Machine} {Learning}}, 2017.

\bibitem[Geifman and El-Yaniv(2017)]{geifman_deep_2017}
Yonatan Geifman and Ran El-Yaniv.
\newblock Deep active learning over the long tail.
\newblock \emph{arXiv:1711.00941}, 2017.

\bibitem[Ghorbani et~al.(2022)Ghorbani, Zou, and Esteva]{ghorbani_data_2022}
Amirata Ghorbani, James Zou, and Andre Esteva.
\newblock Data shapley valuation for efficient batch active learning.
\newblock In \emph{Asilomar {Conference} on {Signals}, {Systems}, and
  {Computers}}. IEEE, 2022.

\bibitem[Gonzalez(1985)]{gonzalez_clustering_1985}
Teofilo~F. Gonzalez.
\newblock Clustering to minimize the maximum intercluster distance.
\newblock \emph{Theoretical Computer Science}, 38:\penalty0 293--306, 1985.

\bibitem[Gorishniy et~al.(2021)Gorishniy, Rubachev, Khrulkov, and
  Babenko]{gorishniy_revisiting_2021}
Yury Gorishniy, Ivan Rubachev, Valentin Khrulkov, and Artem Babenko.
\newblock Revisiting deep learning models for tabular data.
\newblock In \emph{Neural {Information} {Processing} {Systems}}, 2021.

\bibitem[Graf et~al.(2011)Graf, Kriegel, Schubert, Pölsterl, and
  Cavallaro]{graf_2d_2011}
Franz Graf, Hans-Peter Kriegel, Matthias Schubert, Sebastian Pölsterl, and
  Alexander Cavallaro.
\newblock {2D} image registration in ct images using radial image descriptors.
\newblock In \emph{International {Conference} on {Medical} {Image} {Computing}
  and {Computer}-{Assisted} {Intervention}}. Springer, 2011.

\bibitem[Han et~al.(2021)Han, Avron, Shoham, Kim, and Shin]{han_random_2021}
Insu Han, Haim Avron, Neta Shoham, Chaewon Kim, and Jinwoo Shin.
\newblock Random features for the neural tangent kernel.
\newblock \emph{arXiv:2104.01351}, 2021.

\bibitem[Hansen and Salamon(1990)]{hansen_neural_1990}
Lars~Kai Hansen and Peter Salamon.
\newblock Neural network ensembles.
\newblock \emph{IEEE Transactions on Pattern Analysis and Machine
  Intelligence}, 12\penalty0 (10):\penalty0 993--1001, 1990.

\bibitem[He et~al.(2015)He, Zhang, Ren, and Sun]{he_delving_2015}
Kaiming He, Xiangyu Zhang, Shaoqing Ren, and Jian Sun.
\newblock Delving deep into rectifiers: {Surpassing} human-level performance on
  {ImageNet} classification.
\newblock In \emph{{IEEE} {Conference} on {Computer} {Vision}}, 2015.

\bibitem[Houlsby et~al.(2011)Houlsby, Huszár, Ghahramani, and
  Lengyel]{houlsby_bayesian_2011}
Neil Houlsby, Ferenc Huszár, Zoubin Ghahramani, and Máté Lengyel.
\newblock Bayesian active learning for classification and preference learning.
\newblock \emph{arXiv:1112.5745}, 2011.

\bibitem[Immer et~al.(2021)Immer, Korzepa, and Bauer]{immer_improving_2021}
Alexander Immer, Maciej Korzepa, and Matthias Bauer.
\newblock Improving predictions of {Bayesian} neural nets via local
  linearization.
\newblock In \emph{International {Conference} on {Artificial} {Intelligence}
  and {Statistics}}, 2021.

\bibitem[Jacot et~al.(2018)Jacot, Gabriel, and Hongler]{jacot_neural_2018}
Arthur Jacot, Franck Gabriel, and Clément Hongler.
\newblock Neural {Tangent} {Kernel}: {Convergence} and generalization in neural
  networks.
\newblock In \emph{Neural {Information} {Processing} {Systems}}, 2018.

\bibitem[Johnson and Lindenstrauss(1984)]{johnson_extensions_1984}
William~B. Johnson and Joram Lindenstrauss.
\newblock Extensions of {Lipschitz} mappings into a {Hilbert} space.
\newblock \emph{Contemporary Mathematics}, 26, 1984.

\bibitem[Kadra et~al.(2021)Kadra, Lindauer, Hutter, and
  Grabocka]{kadra_well-tuned_2021}
Arlind Kadra, Marius Lindauer, Frank Hutter, and Josif Grabocka.
\newblock Well-tuned simple nets excel on tabular datasets.
\newblock In \emph{Neural {Information} {Processing} {Systems}}, 2021.

\bibitem[Kar and Karnick(2012)]{kar_random_2012}
Purushottam Kar and Harish Karnick.
\newblock Random feature maps for dot product kernels.
\newblock In \emph{Artificial {Intelligence} and {Statistics}}, 2012.

\bibitem[Katsavounidis et~al.(1994)Katsavounidis, Kuo, and
  Zhang]{katsavounidis_new_1994}
Ioannis Katsavounidis, C.-C.~Jay Kuo, and Zhen Zhang.
\newblock A new initialization technique for generalized {Lloyd} iteration.
\newblock \emph{IEEE Signal Processing Letters}, 1\penalty0 (10):\penalty0
  144--146, 1994.

\bibitem[Kaufman and Rousseeuw(1990)]{kaufman_finding_1990}
Leonard Kaufman and Peter~J. Rousseeuw.
\newblock Finding groups in data: an introduction to cluster analysis.
\newblock \emph{Wiley Series in Probability and Mathematical Statistics.
  Applied Probability and Statistics}, 1990.

\bibitem[Kaul et~al.(2013)Kaul, Yang, and Jensen]{kaul_building_2013}
Manohar Kaul, Bin Yang, and Christian~S. Jensen.
\newblock Building accurate 3d spatial networks to enable next generation
  intelligent transportation systems.
\newblock In \emph{International {Conference} on {Mobile} {Data} {Management}},
  2013.

\bibitem[Kennard and Stone(1969)]{kennard_computer_1969}
Ronald~W. Kennard and Larry~A. Stone.
\newblock Computer aided design of experiments.
\newblock \emph{Technometrics}, 11\penalty0 (1):\penalty0 137--148, 1969.

\bibitem[Khan et~al.(2019)Khan, Immer, Abedi, and
  Korzepa]{khan_approximate_2019}
Mohammad Emtiyaz~E. Khan, Alexander Immer, Ehsan Abedi, and Maciej Korzepa.
\newblock Approximate inference turns deep networks into gaussian processes.
\newblock In \emph{Neural {Information} {Processing} {Systems}}, 2019.

\bibitem[Kingma and Ba(2015)]{kingma_adam_2015}
Diederik~P. Kingma and Jimmy Ba.
\newblock Adam: {A} method for stochastic optimization.
\newblock In \emph{International {Conference} on {Learning} {Representations}},
  2015.

\bibitem[Kirsch et~al.(2019)Kirsch, Van~Amersfoort, and
  Gal]{kirsch_batchbald_2019}
Andreas Kirsch, Joost Van~Amersfoort, and Yarin Gal.
\newblock {BatchBALD}: {Efficient} and diverse batch acquisition for deep
  {Bayesian} active learning.
\newblock In \emph{Neural {Information} {Processing} {Systems}}, 2019.

\bibitem[Krause et~al.(2008)Krause, Singh, and
  Guestrin]{krause_near-optimal_2008}
Andreas Krause, Ajit Singh, and Carlos Guestrin.
\newblock Near-optimal sensor placements in {Gaussian} processes: {Theory},
  efficient algorithms and empirical studies.
\newblock \emph{Journal of Machine Learning Research}, 9\penalty0 (2), 2008.

\bibitem[Kristiadi et~al.(2020)Kristiadi, Hein, and
  Hennig]{kristiadi_being_2020}
Agustinus Kristiadi, Matthias Hein, and Philipp Hennig.
\newblock Being {Bayesian}, even just a bit, fixes overconfidence in relu
  networks.
\newblock In \emph{International {Conference} on {Machine} {Learning}}, 2020.

\bibitem[Krizhevsky(2009)]{krizhevsky_learning_2009}
Alex Krizhevsky.
\newblock Learning multiple layers of features from tiny images.
\newblock Technical report, University of Toronto, 2009.

\bibitem[Krogh and Vedelsby(1994)]{krogh_neural_1994}
Anders Krogh and Jesper Vedelsby.
\newblock Neural network ensembles, cross validation, and active learning.
\newblock In \emph{Neural {Information} {Processing} {Systems}}, 1994.

\bibitem[Kumar and Gupta(2020)]{kumar_active_2020}
Punit Kumar and Atul Gupta.
\newblock Active learning query strategies for classification, regression, and
  clustering: a survey.
\newblock \emph{Journal of Computer Science and Technology}, 35\penalty0
  (4):\penalty0 913--945, 2020.

\bibitem[Kutz(2017)]{kutz_deep_2017}
J.~Nathan Kutz.
\newblock Deep learning in fluid dynamics.
\newblock \emph{Journal of Fluid Mechanics}, 814:\penalty0 1--4, 2017.

\bibitem[Lakshminarayanan et~al.(2017)Lakshminarayanan, Pritzel, and
  Blundell]{lakshminarayanan_simple_2017}
Balaji Lakshminarayanan, Alexander Pritzel, and Charles Blundell.
\newblock Simple and scalable predictive uncertainty estimation using deep
  ensembles.
\newblock In \emph{Neural {Information} {Processing} {Systems}}, volume~30,
  2017.

\bibitem[Laplace(1774)]{laplace_memoire_1774}
Pierre~Simon Laplace.
\newblock Mémoire sur la probabilité de causes par les évènements.
\newblock \emph{Mémoires de Mathématique et de Physique, Presentés à
  l'Académie Royale des Sciences, par divers Savants \& lus dans ses
  Assemblées. Tome Sixième}, pages 621--656, 1774.

\bibitem[Lavin et~al.(2021)Lavin, Zenil, Paige, Krakauer, Gottschlich, Mattson,
  Anandkumar, Choudry, Rocki, Baydin, and {others}]{lavin_simulation_2021}
Alexander Lavin, Hector Zenil, Brooks Paige, David Krakauer, Justin
  Gottschlich, Tim Mattson, Anima Anandkumar, Sanjay Choudry, Kamil Rocki,
  Atılım~Güneş Baydin, and {others}.
\newblock Simulation intelligence: {Towards} a new generation of scientific
  methods.
\newblock \emph{arXiv:2112.03235}, 2021.

\bibitem[Lee et~al.(2018)Lee, Bahri, Novak, Schoenholz, Pennington, and
  Sohl-Dickstein]{lee_deep_2018}
Jaehoon Lee, Yasaman Bahri, Roman Novak, Samuel~S. Schoenholz, Jeffrey
  Pennington, and Jascha Sohl-Dickstein.
\newblock Deep neural networks as {Gaussian} processes.
\newblock In \emph{International {Conference} on {Learning} {Representations}},
  2018.

\bibitem[Lee et~al.(2019)Lee, Xiao, Schoenholz, Bahri, Novak, Sohl-Dickstein,
  and Pennington]{lee_wide_2019}
Jaehoon Lee, Lechao Xiao, Samuel Schoenholz, Yasaman Bahri, Roman Novak, Jascha
  Sohl-Dickstein, and Jeffrey Pennington.
\newblock Wide neural networks of any depth evolve as linear models under
  gradient descent.
\newblock In \emph{Neural {Information} {Processing} {Systems}}, 2019.

\bibitem[Lloyd(1982)]{lloyd_least_1982}
Stuart Lloyd.
\newblock Least squares quantization in {PCM}.
\newblock \emph{IEEE Transactions on Information Theory}, 28\penalty0
  (2):\penalty0 129--137, 1982.

\bibitem[Long(2021)]{long_properties_2021}
Philip~M. Long.
\newblock Properties of the after kernel.
\newblock \emph{arXiv:2105.10585}, 2021.

\bibitem[Lázaro-Gredilla and
  Figueiras-Vidal(2010)]{lazaro-gredilla_marginalized_2010}
Miguel Lázaro-Gredilla and Aníbal~R. Figueiras-Vidal.
\newblock Marginalized neural network mixtures for large-scale regression.
\newblock \emph{IEEE Transactions on Neural Networks}, 21\penalty0
  (8):\penalty0 1345--1351, 2010.

\bibitem[MacKay(1992{\natexlab{a}})]{mackay_bayesian_1992}
David~JC MacKay.
\newblock Bayesian interpolation.
\newblock \emph{Neural Computation}, 4\penalty0 (3):\penalty0 415--447,
  1992{\natexlab{a}}.

\bibitem[MacKay(1992{\natexlab{b}})]{mackay_information-based_1992}
David~JC MacKay.
\newblock Information-based objective functions for active data selection.
\newblock \emph{Neural Computation}, 4\penalty0 (4):\penalty0 590--604,
  1992{\natexlab{b}}.

\bibitem[Madan et~al.(2019)Madan, Singh, Tantipongpipat, and
  Xie]{madan_combinatorial_2019}
Vivek Madan, Mohit Singh, Uthaipon Tantipongpipat, and Weijun Xie.
\newblock Combinatorial algorithms for optimal design.
\newblock In \emph{Conference on {Learning} {Theory}}, 2019.

\bibitem[Maddox et~al.(2019)Maddox, Izmailov, Garipov, Vetrov, and
  Wilson]{maddox_simple_2019}
Wesley~J. Maddox, Pavel Izmailov, Timur Garipov, Dmitry~P. Vetrov, and
  Andrew~Gordon Wilson.
\newblock A simple baseline for bayesian uncertainty in deep learning.
\newblock In \emph{Neural {Information} {Processing} {Systems}}, 2019.

\bibitem[Matthews et~al.(2018)Matthews, Hron, Rowland, Turner, and
  Ghahramani]{matthews_gaussian_2018}
Alexander G. de~G. Matthews, Jiri Hron, Mark Rowland, Richard~E. Turner, and
  Zoubin Ghahramani.
\newblock Gaussian process behaviour in wide deep neural networks.
\newblock In \emph{International {Conference} on {Learning} {Representations}},
  2018.

\bibitem[Mehrjou et~al.(2021)Mehrjou, Soleymani, Jesson, Notin, Gal, Bauer, and
  Schwab]{mehrjou_genedisco_2021}
Arash Mehrjou, Ashkan Soleymani, Andrew Jesson, Pascal Notin, Yarin Gal, Stefan
  Bauer, and Patrick Schwab.
\newblock {GeneDisco}: {A} benchmark for experimental design in drug discovery.
\newblock In \emph{International {Conference} on {Learning} {Representations}},
  2021.

\bibitem[Mohamadi et~al.(2022)Mohamadi, Bae, and
  Sutherland]{mohamadi_making_2022}
Mohamad~Amin Mohamadi, Wonho Bae, and Danica~J. Sutherland.
\newblock Making look-ahead active learning strategies feasible with {Neural}
  {Tangent} {Kernels}.
\newblock In \emph{Advances in {Neural} {Information} {Processing} {Systems}},
  2022.

\bibitem[Montgomery(2017)]{montgomery_design_2017}
Douglas~C. Montgomery.
\newblock \emph{Design and {Analysis} of {Experiments}}.
\newblock John Wiley \& Sons, 2017.

\bibitem[Neal(1994)]{neal_priors_1994}
Radford~M. Neal.
\newblock Priors for {Infinite} {Networks}.
\newblock Technical Report CRG-TR-94-1, Dept. of Computer Science, University
  of Toronto, 1994.

\bibitem[Neshat et~al.(2018)Neshat, Alexander, Wagner, and
  Xia]{neshat_detailed_2018}
Mehdi Neshat, Bradley Alexander, Markus Wagner, and Yuanzhong Xia.
\newblock A detailed comparison of meta-heuristic methods for optimising wave
  energy converter placements.
\newblock In \emph{Proceedings of the {Genetic} and {Evolutionary}
  {Computation} {Conference}}, 2018.

\bibitem[Nonnenmacher et~al.(2021)Nonnenmacher, Reeb, and
  Steinwart]{nonnenmacher_which_2021}
Manuel Nonnenmacher, David Reeb, and Ingo Steinwart.
\newblock Which minimizer does my neural network converge to?
\newblock In \emph{Joint {European} {Conference} on {Machine} {Learning} and
  {Knowledge} {Discovery} in {Databases}}, 2021.

\bibitem[Novak et~al.(2022)Novak, Sohl-Dickstein, and
  Schoenholz]{novak_fast_2022}
Roman Novak, Jascha Sohl-Dickstein, and Samuel~S. Schoenholz.
\newblock Fast finite width neural tangent kernel.
\newblock In \emph{International {Conference} on {Machine} {Learning}}, 2022.

\bibitem[Ober and Rasmussen(2019)]{ober_benchmarking_2019}
Sebastian~W. Ober and Carl~Edward Rasmussen.
\newblock Benchmarking the neural linear model for regression.
\newblock In \emph{Symposium on {Advances} in {Approximate} {Bayesian}
  {Inference}}, 2019.

\bibitem[Ostrovsky et~al.(2006)Ostrovsky, Rabani, Schulman, and
  Swamy]{ostrovsky_effectiveness_2006}
Rafail Ostrovsky, Yuval Rabani, Leonard~J. Schulman, and Chaitanya Swamy.
\newblock The effectiveness of {Lloyd}-type methods for the k-means problem.
\newblock In \emph{{IEEE} {Symposium} on {Foundations} of {Computer}
  {Science}}, 2006.

\bibitem[Paszke et~al.(2019)Paszke, Gross, Massa, Lerer, Bradbury, Chanan,
  Killeen, Lin, Gimelshein, and Antiga]{paszke_pytorch_2019}
Adam Paszke, Sam Gross, Francisco Massa, Adam Lerer, James Bradbury, Gregory
  Chanan, Trevor Killeen, Zeming Lin, Natalia Gimelshein, and Luca Antiga.
\newblock Pytorch: {An} imperative style, high-performance deep learning
  library.
\newblock In \emph{Neural {Information} {Processing} {Systems}}, 2019.

\bibitem[Pazouki and Schaback(2011)]{pazouki_bases_2011}
Maryam Pazouki and Robert Schaback.
\newblock Bases for kernel-based spaces.
\newblock \emph{Journal of Computational and Applied Mathematics}, 236\penalty0
  (4):\penalty0 575--588, 2011.

\bibitem[Pinsler et~al.(2019)Pinsler, Gordon, Nalisnick, and
  Hernández-Lobato]{pinsler_bayesian_2019}
Robert Pinsler, Jonathan Gordon, Eric Nalisnick, and José~Miguel
  Hernández-Lobato.
\newblock Bayesian batch active learning as sparse subset approximation.
\newblock In \emph{Neural {Information} {Processing} {Systems}}, 2019.

\bibitem[Pop and Fulop(2018)]{pop_deep_2018}
Remus Pop and Patric Fulop.
\newblock Deep ensemble bayesian active learning: {Addressing} the mode
  collapse issue in monte carlo dropout via ensembles.
\newblock \emph{arXiv:1811.03897}, 2018.

\bibitem[Raissi et~al.(2019)Raissi, Perdikaris, and
  Karniadakis]{raissi_physics-informed_2019}
Maziar Raissi, Paris Perdikaris, and George~E. Karniadakis.
\newblock Physics-informed neural networks: {A} deep learning framework for
  solving forward and inverse problems involving nonlinear partial differential
  equations.
\newblock \emph{Journal of Computational Physics}, 378:\penalty0 686--707,
  2019.

\bibitem[Ranganathan et~al.(2020)Ranganathan, Venkateswara, Chakraborty, and
  Panchanathan]{ranganathan_deep_2020}
Hiranmayi Ranganathan, Hemanth Venkateswara, Shayok Chakraborty, and Sethuraman
  Panchanathan.
\newblock Deep active learning for image regression.
\newblock In \emph{Deep {Learning} {Applications}}, pages 113--135. Springer,
  Singapore, 2020.

\bibitem[Ren et~al.(2021)Ren, Xiao, Chang, Huang, Li, Gupta, Chen, and
  Wang]{ren_survey_2021}
Pengzhen Ren, Yun Xiao, Xiaojun Chang, Po-Yao Huang, Zhihui Li, Brij~B. Gupta,
  Xiaojiang Chen, and Xin Wang.
\newblock A survey of deep active learning.
\newblock \emph{ACM Computing Surveys}, 54\penalty0 (9):\penalty0 1--40, 2021.

\bibitem[Rosenkrantz et~al.(1977)Rosenkrantz, Stearns, and
  Lewis]{rosenkrantz_analysis_1977}
Daniel~J. Rosenkrantz, Richard~E. Stearns, and Philip~M. Lewis.
\newblock An analysis of several heuristics for the traveling salesman problem.
\newblock \emph{SIAM Journal on Computing}, 6\penalty0 (3):\penalty0 563--581,
  1977.

\bibitem[Santin et~al.(2021)Santin, Karvonen, and
  Haasdonk]{santin_sampling_2021}
Gabriele Santin, Toni Karvonen, and Bernard Haasdonk.
\newblock Sampling based approximation of linear functionals in reproducing
  kernel {Hilbert} spaces.
\newblock \emph{BIT Numerical Mathematics}, pages 1--32, 2021.
\newblock Publisher: Springer.

\bibitem[Savva et~al.(2018)Savva, Anagnostopoulos, and
  Triantafillou]{savva_explaining_2018}
Fotis Savva, Christos Anagnostopoulos, and Peter Triantafillou.
\newblock Explaining aggregates for exploratory analytics.
\newblock In \emph{International {Conference} on {Big} {Data}}, 2018.

\bibitem[Schraudolph(2002)]{schraudolph_fast_2002}
Nicol~N. Schraudolph.
\newblock Fast curvature matrix-vector products for second-order gradient
  descent.
\newblock \emph{Neural Computation}, 14\penalty0 (7):\penalty0 1723--1738,
  2002.

\bibitem[Sener and Savarese(2018)]{sener_active_2018}
Ozan Sener and Silvio Savarese.
\newblock Active learning for convolutional neural networks: {A} core-set
  approach.
\newblock In \emph{International {Conference} on {Learning} {Representations}},
  2018.

\bibitem[Seo et~al.(2000)Seo, Wallat, Graepel, and
  Obermayer]{seo_gaussian_2000}
Sambu Seo, Marko Wallat, Thore Graepel, and Klaus Obermayer.
\newblock Gaussian process regression: {Active} data selection and test point
  rejection.
\newblock In \emph{Mustererkennung 2000}, pages 27--34. Springer, 2000.

\bibitem[Settles(2009)]{settles_active_2009}
Burr Settles.
\newblock Active learning literature survey.
\newblock Computer {Sciences} {Technical} {Report} 1648, University of
  Wisconsin–Madison, 2009.

\bibitem[Seung et~al.(1992)Seung, Opper, and Sompolinsky]{seung_query_1992}
H.~Sebastian Seung, Manfred Opper, and Haim Sompolinsky.
\newblock Query by committee.
\newblock In \emph{Workshop on {Computational} {Learning} {Theory}}, 1992.

\bibitem[Shan and Bordelon(2021)]{shan_theory_2021}
Haozhe Shan and Blake Bordelon.
\newblock A theory of neural tangent kernel alignment and its influence on
  training.
\newblock \emph{arXiv:2105.14301}, 2021.

\bibitem[Shannon(1948)]{shannon_mathematical_1948}
Claude~Elwood Shannon.
\newblock A mathematical theory of communication.
\newblock \emph{The Bell System Technical Journal}, 27\penalty0 (3):\penalty0
  379--423, 1948.

\bibitem[Shannon et~al.(2003)Shannon, Markiel, Ozier, Baliga, Wang, Ramage,
  Amin, Schwikowski, and Ideker]{shannon_cytoscape_2003}
Paul Shannon, Andrew Markiel, Owen Ozier, Nitin~S. Baliga, Jonathan~T. Wang,
  Daniel Ramage, Nada Amin, Benno Schwikowski, and Trey Ideker.
\newblock Cytoscape: a software environment for integrated models of
  biomolecular interaction networks.
\newblock \emph{Genome research}, 13\penalty0 (11):\penalty0 2498--2504, 2003.

\bibitem[Sharma et~al.(2021)Sharma, Azizan, and Pavone]{sharma_sketching_2021}
Apoorva Sharma, Navid Azizan, and Marco Pavone.
\newblock Sketching curvature for efficient out-of-distribution detection for
  deep neural networks.
\newblock In \emph{Uncertainty in {Artificial} {Intelligence}}, 2021.

\bibitem[Shoham and Avron(2023)]{shoham_experimental_2023}
Neta Shoham and Haim Avron.
\newblock Experimental design for overparameterized learning with application
  to single shot deep active learning.
\newblock \emph{IEEE Transactions on Pattern Analysis and Machine
  Intelligence}, 2023.
\newblock Publisher: IEEE.

\bibitem[Smith(1918)]{smith_standard_1918}
Kirstine Smith.
\newblock On the standard deviations of adjusted and interpolated values of an
  observed polynomial function and its constants and the guidance they give
  towards a proper choice of the distribution of observations.
\newblock \emph{Biometrika}, 12\penalty0 (1/2):\penalty0 1--85, 1918.

\bibitem[Snoek et~al.(2015)Snoek, Rippel, Swersky, Kiros, Satish, Sundaram,
  Patwary, Prabhat, and Adams]{snoek_scalable_2015}
Jasper Snoek, Oren Rippel, Kevin Swersky, Ryan Kiros, Nadathur Satish,
  Narayanan Sundaram, Mostofa Patwary, Mr~Prabhat, and Ryan Adams.
\newblock Scalable bayesian optimization using deep neural networks.
\newblock In \emph{International {Conference} on {Machine} {Learning}}, 2015.

\bibitem[Somepalli et~al.(2022)Somepalli, Schwarzschild, Goldblum, Bruss, and
  Goldstein]{somepalli_saint_2022}
Gowthami Somepalli, Avi Schwarzschild, Micah Goldblum, C.~Bayan Bruss, and Tom
  Goldstein.
\newblock {SAINT}: {Improved} neural networks for tabular data via row
  attention and contrastive pre-training.
\newblock In \emph{{NeurIPS} 2022 {Table} {Representation} {Workshop}}, 2022.

\bibitem[Srivastava et~al.(2014)Srivastava, Hinton, Krizhevsky, Sutskever, and
  Salakhutdinov]{srivastava_dropout_2014}
Nitish Srivastava, Geoffrey Hinton, Alex Krizhevsky, Ilya Sutskever, and Ruslan
  Salakhutdinov.
\newblock Dropout: a simple way to prevent neural networks from overfitting.
\newblock \emph{The Journal of Machine Learning Research}, 15\penalty0
  (1):\penalty0 1929--1958, 2014.

\bibitem[Steinwart and Christmann(2008)]{steinwart_support_2008}
Ingo Steinwart and Andreas Christmann.
\newblock \emph{Support vector machines}.
\newblock Springer Science \& Business Media, 2008.

\bibitem[Tsymbalov et~al.(2018)Tsymbalov, Panov, and
  Shapeev]{tsymbalov_dropout-based_2018}
Evgenii Tsymbalov, Maxim Panov, and Alexander Shapeev.
\newblock Dropout-based active learning for regression.
\newblock In \emph{International {Conference} on {Analysis} of {Images},
  {Social} {Networks} and {Texts}}, 2018.

\bibitem[Vanschoren et~al.(2013)Vanschoren, van Rijn, Bischl, and
  Torgo]{vanschoren_openml_2013}
Joaquin Vanschoren, Jan~N. van Rijn, Bernd Bischl, and Luis Torgo.
\newblock {OpenML}: {Networked} science in machine learning.
\newblock \emph{SIGKDD Explorations}, 15\penalty0 (2):\penalty0 49--60, 2013.

\bibitem[Vijayakumar and Schaal(2000)]{vijayakumar_locally_2000}
Sethu Vijayakumar and Stefan Schaal.
\newblock Locally weighted projection regression: {An} {O}(n) algorithm for
  incremental real time learning in high dimensional space.
\newblock In \emph{International {Conference} on {Machine} {Learning}}, 2000.

\bibitem[Wald(1943)]{wald_efficient_1943}
Abraham Wald.
\newblock On the efficient design of statistical investigations.
\newblock \emph{The Annals of Mathematical Statistics}, 14\penalty0
  (2):\penalty0 134--140, 1943.

\bibitem[Wang et~al.(2022)Wang, Huang, Wu, Tong, Margenot, and
  He]{wang_deep_2022}
Haonan Wang, Wei Huang, Ziwei Wu, Hanghang Tong, Andrew~J. Margenot, and
  Jingrui He.
\newblock Deep active learning by leveraging training dynamics.
\newblock In \emph{Advances in {Neural} {Information} {Processing} {Systems}},
  2022.

\bibitem[Wang et~al.(2021)Wang, Awasthi, Dann, Sekhari, and
  Gentile]{wang_neural_2021}
Zhilei Wang, Pranjal Awasthi, Christoph Dann, Ayush Sekhari, and Claudio
  Gentile.
\newblock Neural active learning with performance guarantees.
\newblock In \emph{Neural {Information} {Processing} {Systems}}, 2021.

\bibitem[Weng(2022)]{weng_learning_2022}
Lilian Weng.
\newblock Learning with not enough data part 2: {Active} learning.
\newblock \emph{lilianweng.github.io/lil-log}, 2022.
\newblock URL
  \url{https://lilianweng.github.io/lil-log/2022/02/20/active-learning.html}.

\bibitem[Wenzel et~al.(2021)Wenzel, Santin, and Haasdonk]{wenzel_novel_2021}
Tizian Wenzel, Gabriele Santin, and Bernard Haasdonk.
\newblock A novel class of stabilized greedy kernel approximation algorithms:
  {Convergence}, stability and uniform point distribution.
\newblock \emph{Journal of Approximation Theory}, 262:\penalty0 105508, 2021.

\bibitem[Wilson and Izmailov(2020)]{wilson_bayesian_2020}
Andrew~G. Wilson and Pavel Izmailov.
\newblock Bayesian deep learning and a probabilistic perspective of
  generalization.
\newblock In \emph{Neural {Information} {Processing} {Systems}}, 2020.

\bibitem[Woodruff(2014)]{woodruff_sketching_2014}
David~P. Woodruff.
\newblock Sketching as a tool for numerical linear algebra.
\newblock \emph{Foundations and Trends in Theoretical Computer Science},
  10\penalty0 (1–2):\penalty0 1--157, 2014.

\bibitem[Wu(2018)]{wu_pool-based_2018}
Dongrui Wu.
\newblock Pool-based sequential active learning for regression.
\newblock \emph{IEEE Transactions on Neural Networks and Learning Systems},
  30\penalty0 (5):\penalty0 1348--1359, 2018.

\bibitem[Wynn(1970)]{wynn_sequential_1970}
Henry~P. Wynn.
\newblock The sequential generation of {D}-optimum experimental designs.
\newblock \emph{The Annals of Mathematical Statistics}, 41\penalty0
  (5):\penalty0 1655--1664, 1970.

\bibitem[Yu and Kim(2010)]{yu_passive_2010}
Hwanjo Yu and Sungchul Kim.
\newblock Passive sampling for regression.
\newblock In \emph{International {Conference} on {Data} {Mining}}, 2010.

\bibitem[Zandieh et~al.(2021)Zandieh, Han, Avron, Shoham, Kim, and
  Shin]{zandieh_scaling_2021}
Amir Zandieh, Insu Han, Haim Avron, Neta Shoham, Chaewon Kim, and Jinwoo Shin.
\newblock Scaling neural tangent kernels via sketching and random features.
\newblock In \emph{Neural {Information} {Processing} {Systems}}, 2021.

\bibitem[Zaverkin and Kästner(2021)]{zaverkin_exploration_2021}
Viktor Zaverkin and Johannes Kästner.
\newblock Exploration of transferable and uniformly accurate neural network
  interatomic potentials using optimal experimental design.
\newblock \emph{Machine Learning: Science and Technology}, 2\penalty0
  (3):\penalty0 035009, 2021.

\bibitem[Zhdanov(2019)]{zhdanov_diverse_2019}
Fedor Zhdanov.
\newblock Diverse mini-batch active learning.
\newblock \emph{arXiv:1901.05954}, 2019.

\bibitem[Zhou et~al.(2021)Zhou, Renduchintala, Li, Wang, Mehdad, and
  Ghoshal]{zhou_towards_2021}
Yilun Zhou, Adithya Renduchintala, Xian Li, Sida Wang, Yashar Mehdad, and Asish
  Ghoshal.
\newblock Towards understanding the behaviors of optimal deep active learning
  algorithms.
\newblock In \emph{International {Conference} on {Artificial} {Intelligence}
  and {Statistics}}, 2021.

\bibitem[Ślęzak et~al.(2018)Ślęzak, Grzegorowski, Janusz, Kozielski,
  Nguyen, Sikora, Stawicki, and Wróbel]{slezak_framework_2018}
Dominik Ślęzak, Marek Grzegorowski, Andrzej Janusz, Michał Kozielski,
  Sinh~Hoa Nguyen, Marek Sikora, Sebastian Stawicki, and Łukasz Wróbel.
\newblock A framework for learning and embedding multi-sensor forecasting
  models into a decision support system: {A} case study of methane
  concentration in coal mines.
\newblock \emph{Information Sciences}, 451:\penalty0 112--133, 2018.

\end{thebibliography}
\end{appendixenv}

\end{document}